\documentclass[unnumsec,webpdf,contemporary,large]{oup-authoring-template}
\usepackage{graphicx} 

\usepackage{array}
\usepackage{amsthm}
\usepackage{amsfonts}
\usepackage{commath}
\usepackage{mathtools}
\usepackage{rotating}
\usepackage{threeparttable}
\usepackage{blkarray}
\usepackage{aligned-overset}
\usepackage{booktabs}
\usepackage{colortbl}
\usepackage{algorithm}
\usepackage{algpseudocode}
\usepackage{hyperref}
\usepackage{longtable}

\newtheorem{theorem}{Theorem}
\newtheorem*{theorem*}{Theorem}

\newtheorem{corollary}{Corollary}
\newtheorem*{prop*}{Proposition}
\newtheorem{lemma}{Lemma}
\newtheorem*{lemma*}{Lemma}
\newtheorem{assumption}{Assumption}
\newtheorem{amsthmexample}{Example}

\algrenewcommand\algorithmicrequire{\textbf{Input:}}
\algrenewcommand\algorithmicensure{\textbf{Output:}}

\makeatletter
\newcommand\appendix{%
  \par
  \setcounter{section}{0}%
  \setcounter{subsection}{0}%
  \renewcommand\thesection{\Alph{section}}%
  \renewcommand\thesubsection{\thesection.\arabic{subsection}}%
}

\def\@seccntformat#1{%
  \csname the#1\endcsname\quad
}
\makeatother

\journaltitle{Preprint}
\DOI{none}
\copyrightyear{2026}
\pubyear{2026}
\access{Preprint}
\appnotes{Preprint}

\title[A Scalable Multi-Axis Gaussian Graphical Model]{Making Multi-Axis Gaussian Graphical Models Scalable to Millions of Cells}

\author[1,$\dagger$,$\ast$]{Bailey Andrew}
\author[2,$\dagger$]{Erica L. Harris}
\author[2,$\dagger$]{James A. Poulter}
\author[3,$\dagger$]{David R. Westhead}
\author[4,$\dagger$]{Luisa Cutillo}

\address[1]{\orgdiv{School of Computing}}
\address[2]{\orgdiv{School of Medicine}}
\address[3]{\orgdiv{School of Molecular and Cellular Biology}}
\address[4]{\orgdiv{School of Mathematics}}
\corresp[$\dagger$]{\orgname{University of Leeds}, \orgaddress{\street{Woodhouse Lane}, \postcode{LS2 9JT}, \state{West Yorkshire}, \country{England}}}

\corresp[\hspace{5pt}$\ast$]{Corresponding author: \href{mailto:sceba@leeds.ac.uk}{sceba@leeds.ac.uk}}

\authormark{Andrew et al.}

\begin{document}

\makeatletter
\global\def\@abstract{
    \textbf{Motivation:}
    Networks underlie the generation and interpretation of many biological datasets: gene networks shed light on the regulatory structure of the genome, and cell networks can capture structure of the tumor micro-environment.  However, most methods that learn such networks make the faulty `independence assumption'; to learn the gene network, they assume that no cell network exists.  `Multi-axis' methods, which do not make this assumption, fail to scale beyond a few thousand cells or genes.  This limits their applicability to only the smallest datasets.
    
    \textbf{Results:} We develop a multi-axis method capable of processing million-cell datasets within minutes.  This was previously impossible, and unlocks the use of such methods on modern scRNA-seq datasets, as well as more complex datasets.  We show that our method yields novel biological insights from real single-cell data{, and compares favorably to the existing hdWGCNA methodology}.  In particular, it identifies long non-coding RNA genes that potentially have a regulatory or functional role in neuronal development.
    
    \textbf{Availability and implementation:} Our methodology is available as a Python package GmGM on PyPI (\href{https://pypi.org/project/GmGM/0.5.3/}{https://pypi.org/project/GmGM/0.5.3/}).  The code for all experiments performed in this paper is available on GitHub (\href{https://github.com/BaileyAndrew/GmGM-Bioinformatics}{https://github.com/BaileyAndrew/GmGM-Bioinformatics}).
    
    \textbf{Contact:} \href{mailto:sceba@leeds.ac.uk}{sceba@leeds.ac.uk}
    
    \textbf{Supplementary information:} Our proofs, and some additional experiments, are available in the supplementary material.
\par
\if@modern
{\sffamilyfontcnbold\fontsize{8bp}{11}\keywordsname\ \sffamilyfontcn gaussian graphical models, multi-axis models, transcriptomics, multi-omics, scalability}%
\else
\if@traditional
{\fontsize{8bp}{10}\textbf{\uppercase{Keywords}:}\ gaussian graphical models, multi-axis models, transcriptomics, multi-omics, scalability}%
\else
{\fontsize{8bp}{10}\textbf{{Keywords}:}\ gaussian graphical models, multi-axis models, transcriptomics, multi-omics, scalability}%
\fi\fi
}
\makeatother

\maketitle

\section{Introduction}
\label{sec:gmgm-bio-introduction}

Learning a network (`graph') that describes a dateset can improve understanding of the systems that generated it.  A common example is learning a gene association network from a single-cell RNA-sequencing (scRNA-seq) dataset.  These datasets typically come in the form of a $\text{cells}\times\text{genes}$ matrix, with the individual matrix entries of gene expression values; learning a gene network from this type of data gives insight into the regulatory pathways governing cell behavior.  There are many methods to learn networks, such as the graphical lasso \citep{friedman_sparse_2008}, WGCNA and its single-celled variant hdWGCNA \citep{langfelder_wgcna_2008, morabito_hdwgcna_2023}, and neighborhood selection \citep{meinshausen_high-dimensional_2006}.

However these methods all make the `independence assumption'; to learn the gene network, they must assume that the cells do not interact.  This is clearly false; cells come from a highly interactive micro-environment, and come in a variety of cell types - all of this information is useful for the construction of gene networks, and should not be thrown away.  In this paper, we propose a scalable method to learn gene networks while still capturing cellular heterogeneity; as a byproduct, our method also produces a cell network that can be used for cell type inference.  By learning both networks simultaneously, one can expect to get better results than learning them sequentially - not only does our method not make the independence assumption, but it can share information between the networks when learning them, leading to better estimates of both.

Our method is a type of `multi-axis' method (which are often called `Kronecker-separable' methods), named so because they consider the dependencies within all \textit{axes} of a dataset; in contrast, methods that make the independence assumption are `single-axis'.  In general, for any matrix-variate dataset there are two potential \textit{kinds of dependencies} one must consider: sample (`row', `cell') and feature (`column', `gene') dependencies.  This concept, and our method, can be extended to more complex (`tensor-variate') datasets, which occasionally show up in bioinformatics (Example \ref{ex:gmgm-bio-tensor-variate}).

\begin{amsthmexample}
    scRNA-seq datasets have two axes (cells, genes).
\end{amsthmexample}
\begin{amsthmexample}
    \mbox{scATAC-seq datasets have two axes (cells, peaks).}
\end{amsthmexample}
\begin{amsthmexample}
    \label{ex:gmgm-bio-tensor-variate}
    A longitudinal bulk RNA study of multiple patients has three axes (timepoints, patients, genes).
\end{amsthmexample}

Also in scope are datasets with `shared axes'.  A shared axis is an axis that appears in two or more matrices at a time.  The goal is to learn a network describing each axis, without making an independence assumption on any.

\begin{amsthmexample}
    Paired scRNA-seq/scATAC-seq datasets have three axes (genes, peaks, and a cells axis shared between the two modalities).
\end{amsthmexample}
\begin{amsthmexample}
    A 5 patient scRNA-seq dataset has six axes (one for the cells of each patient, and a gene axis shared between them all).
\end{amsthmexample}

The standard independence assumption is that there are no row dependencies (i.e. there are no dependencies between cells in the scRNA-seq case), which is tractable but inaccurate.  One could instead allow `full dependence', i.e. every matrix element $(i_1, j_1)$ could potentially be dependent on any other matrix element $(i_2, j_2)$.  This would allow very granular statements to be made about the data, such as `gene A in cell X affects gene B in cell Y'.  Sadly, inference of such a network is statistically intractable: a $d\times d$ matrix would have $O(d^4)$ potential dependencies, and only one `sample' (the matrix itself) from which to learn such dependencies.

To strike a balance between model accuracy and model tractability, the independence assumption is replaced with a weaker, more realistic assumption about how the axes may interact.  In this paper, we make the \textbf{Cartesian product assumption}, which is easiest to describe in terms of a concrete scRNA-seq example: `a gene can only be (conditionally) dependent on the expression of \textit{other genes within the same cell}, or the \textit{same gene in other cells}'.  This assumption is weaker than the independence assumption, yet like the independence assumption only a more tractable $O(d^2)$ possible dependencies need be considered.

The networks considered in this paper are \textit{conditional dependence} networks; two genes are conditionally dependent, given the rest of the genes in the dataset, if they are statistically dependent on one another after removing the effect of the other genes.  Conditional dependence captures the concept of `direct correlation'; if $A$ and $B$ are only correlated indirectly due to their mutual correlation with another gene $C$, then they would be conditionally independent.  Thus, conditional dependency networks tend to be sparser and more meaningful than co-expression networks.  For Gaussian data, the inverse covariance matrix $\mathbf{\Psi}$ encodes the conditional dependency graph as follows:

\begin{align*}
    \mathbf{\Psi}_{ij} = 0 &\iff \text{gene }i \text{ is conditionally independent from gene } j
\end{align*}

This convenient relationship allows our assumptions about the data to be expressed as a probabilistic model, where $\mathbf{\Psi}^\textit{(cells)}, \mathbf{\Psi}^\textit{(genes)}$ are the cell and gene conditional dependency matrix and $d_1, d_2$ are the numbers of cells and genes, respectively.  In doing so, our model is amenable to maximum likelihood estimation (MLE).  Letting $\mathbf{X}$ be our dataset, $\mathbf{I}$ and $\mathbf{0}$ be the identity matrix and zero matrix, $\mathrm{vec}$ be the operation that flattens matrices into vectors, and $\otimes$ be the `Kronecker product' operation, traditional single-axis models and our model can be expressed as follows (with $\mathcal{N}\left(\boldsymbol{\mu},\mathbf{\Sigma}\right)$ being the normal distribution with mean $\boldsymbol{\mu}$ and covariance $\mathbf{\Sigma}$):

\begin{align*}
    \mathrm{vec}\left[\mathbf{X}\right] &\sim \mathcal{N}\left(\mathbf{0}_{d_1d_2},  \left( \mathbf{\Psi}^{(\textit{genes})}_{d_2\times d_2} \otimes \mathbf{I}_{d_1\times d_1} + \mathbf{I}_{d_2 \times d_2} \otimes \mathbf{\Psi}^{(\textit{cells})}_{d_1\times d_1}\right)^{-1}\right) \tag{Cartesian product [our method]}
\end{align*}
\begin{align*}
    \mathrm{vec}\left[\mathbf{X}\right] &\sim \mathcal{N}\left(\mathbf{0}_{d_1d_2}, \left(\mathbf{\Psi}^{(\textit{genes})}_{d_2\times d_2} \otimes \mathbf{I}_{d_1 \times d_1}\right)^{-1}\right) \tag{Independence of cells [glasso, hdWGCNA, etc.]}
\end{align*}

The operation $\mathbf{F}\otimes \mathbf{I} + \mathbf{I}\otimes \mathbf{G}$ is common enough that it is known as the `Kronecker sum' $\mathbf{F} \oplus \mathbf{G}$.  For two graphs $\mathfrak{F}, \mathfrak{G}$ encoded as matrices $\mathbf{F}, \mathbf{G}$, the Kronecker sum $\mathbf{F} \oplus \mathbf{G}$ encodes the Cartesian product of $\mathfrak{F}$ and $\mathfrak{G}$.

A Gaussian assumption would not be correct for most biological datasets, particularly those comprising of binary or count data.  However, this assumption can be replaced by the much weaker `Gaussian copula' assumption, which allows the data to follow any distribution as long as the genes `interact Gaussianly'; in particular, if genes were to follow a negative binomial or Poisson distribution, they can still fit the Gaussian copula assumption.  This is discussed further in Section \ref{sec:gmgm-bio-assumptions}.

This paper is not the first to consider the Cartesian product assumption; there are several methods to learn such graphs.  The most scalable include TeraLasso \citep{greenewald_tensor_2019}, which was the first to generalize to tensor-variate data, and GmGM \citep{andrew_gmgm_2024}, which further generalized the scenario to the `shared axis' case\footnote{Other methods that work with shared axes in multi-modal datasets do exist, but they all make the independence assumption on one of the axes; for example, the group graphical lasso and fused graphical lasso \citep{danaher_joint_2014, cai_joint_2016},  and Gaussian chain graphical models \citep{mccarter_sparse_2014}.}.  However, TeraLasso is only scalable to hundreds of cells, and GmGM to low thousands.  Modern scRNA-seq datasets have hundreds of thousands of cells.  In fact, there are 11 Human Cell Atlas datasets with more than a million cells \citep{regev_human_2017}.  Extrapolating their runtimes to a million-cell dataset, TeraLasso would take 15,000 years to run and GmGM would take 200.  Both would require at least 16 terabytes of memory.  Even with access to a high-performance-computing (HPC) cluster, the usability of these algorithms on such datasets is dubious.

Due to these scalability constraints, multi-axis models cannot be applied to modern scRNA-seq datasets.  In this paper, we make significant alterations to the GmGM algorithm to improve its scalability, resulting in an algorithm that can run on million-cell datasets in a matter of minutes.  Our contributions unlock the use of multi-axis models on modern datasets.

\section{Materials and methods}
\label{sec:gmgm-bio-methods}

\subsection{Notation}
\label{sec:gmgm-bio-notation}

Because our model is built to work on a wide variety of datasets (including those with shared axes and more than two axes), the notation can be quite dense; in particular, many variables of the form $d_y^x$ are defined below.  For most users, only the matrix-variate case (which includes scRNA-seq) will be of interest, for which the notification can be drastically simplified.  See Section \ref{sec:gmgm-bio-simplified notation} for such a simplification.  Furthermore, Figure \ref{fig:gmgm-bio-notation} gives a graphical overview of the most complicated bit of the notation.

\begin{figure*}[ht!]
    \centering
    \includegraphics[width=0.6\linewidth]{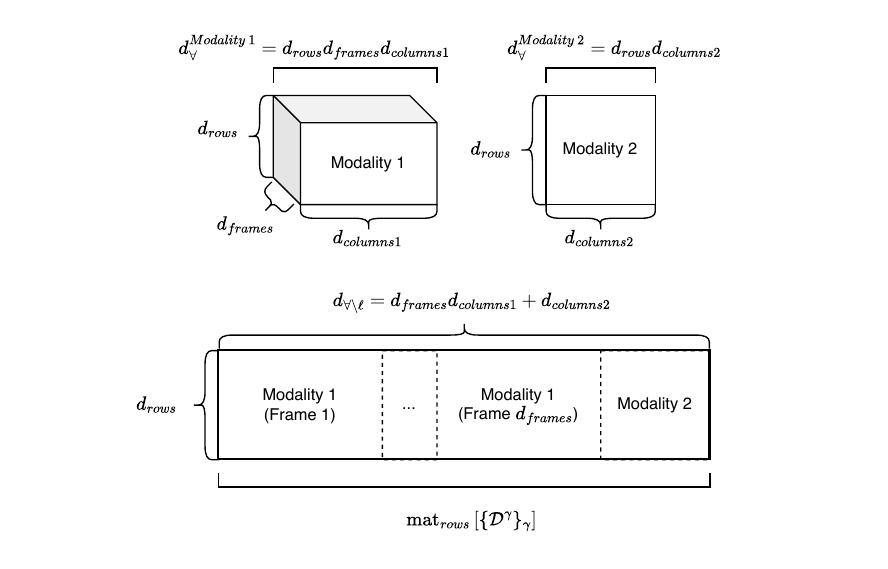}
    \caption{An example two-modality dataset with a few relevant properties highlighted.  Also contains a graphical definition of $\mathrm{mat}_\ell\left[\{\mathcal{D}^\gamma\}_\gamma\right]$ (bottom).}
    \label{fig:gmgm-bio-notation}
\end{figure*}

Axes are denoted using subscripts (often with the symbol $\ell$); for example, $\mathbf{\Psi}_\mathrm{cells}$ would be the precision matrix of the cells axis.  Modalities will be represented with superscripts (often with the symbol $\gamma$); in a multi-patient scRNA-seq study, $\mathbf{D}^\mathrm{patient 1}$ would represent the cells $\times$ genes matrix of the first patient.  $\ell\in\gamma$ represents the statement that an axis is recorded in a modality; in a paired scRNA/ATAC-seq dataset, ``$\mathrm{genes}\in\mathrm{scRNA\text{-}seq\text{ }modality}$'' but ``$\mathrm{genes}\notin\mathrm{scATAC\text{-}seq\text{ }modality}$''.

Let $d_\ell$ be the size of axis $\ell$, $d_\forall^\gamma = \prod_{\ell\in\gamma}d_\ell$ be the size of modality $\gamma$, $d_{\backslash\ell}^\gamma = \frac{d_\forall^\gamma}{d_\ell}$ be the number of samples from $\ell$'s perspective in $\gamma$, and $d_{\backslash\ell}^\forall = \sum_{\gamma | \ell\in\gamma} d_{\backslash\ell}^\gamma$ be the number of samples from $\ell$'s perspective overall.  $r_\ell$ and analogous variables are defined to represent the number of principal components chosen for axis $\ell$.

$\mathrm{det}^\dagger$ represents the \textit{pseudo-determinant} and $\mathrm{vec}$ represents vectorization.  $\mathrm{mat}_\ell\left[\mathcal{X}^\gamma\right]$ represents the `matricization' of a tensor $\mathcal{X}^\gamma$, which reshapes the data into a $d_\ell \times d^\gamma_{\backslash\ell}$ matrix.  Matricization of a dataset, $\mathrm{mat}_\ell\left[\{\mathcal{X}^\gamma\}_\gamma\right]$, reshapes the data into a $d_\ell \times d^\forall_{\backslash\ell}$ matrix.  See Figure \ref{fig:gmgm-bio-notation} for a graphical example of these definitions.  $\mathbf{S}_\ell = \mathrm{mat}_\ell\left[\{\mathcal{X}^\gamma\}_\gamma\right]\mathrm{mat}_\ell\left[\{\mathcal{X}^\gamma\}_\gamma\right]^T$ is the covariance matrix for axis $\ell$.  Its top $r_\ell$ eigenvectors and eigenvalues are $\mathbf{V}_\ell^{(r_\ell)}$ and $\mathbf{E}_\ell^{(r_\ell)}$, respectively.

The final definition needed is that of the stridewise-blockwise trace $\mathrm{tr}^a_b\left[\mathbf{X}\right]$; it arises in the gradient of terms involving Kronecker products.  Letting $\mathbf{J}^{(ij)}$ is a matrix of zeros except for a single 1 in position $(i, j)$, we have:

\[\mathrm{tr}^a_b\left[\mathbf{X}\right]_{ij} = \mathrm{tr}\left[\mathbf{X}\left(\mathbf{I}_{a\times a} \otimes \mathbf{J}^{(ij)}\otimes \mathbf{I}_{b\times b}\right)\right]\]

\subsubsection{Simplified notation}
\label{sec:gmgm-bio-simplified notation}

If one is only interested in matrix-variate datasets, the notation can be simplified.  For any such dataset, there is only one modality, so all superscripts $\gamma$ can be ignored.  Furthermore, there are only two axes (of sizes $d_1, d_2$), and $d_1 = d_{\backslash 2}=d^\forall_{\backslash 2}$ (with analogous relationships for $d_2$).  $d_{\forall} = d_1 d_2$, and for a dataset $\mathbf{X}$ it holds that $\mathrm{mat}_1\left[\mathbf{X}\right]=\mathbf{X}^T$ and $\mathrm{mat}_2\left[\mathbf{X}\right]=\mathbf{X}$.  Any kind of product over $\ell \in \gamma$ is just a product over the two axes; $\ell\in\{1,2\}$.

\subsection{Assumptions}
\label{sec:gmgm-bio-assumptions}

The five assumptions of this work are listed below, citing where they have been used for multi-axis modelling before.

\begin{assumption}
    \label{ass:ks-structure}
    The data's dependencies are Kronecker-sum structured.  First used by \citet{kalaitzis_bigraphical_2013}.
\end{assumption}
\begin{assumption}
    \label{ass:nonparanormality}
    The data has a Gaussian copula, also called the `nonparanormality assumption'.  Used by \citet{li_scalable_2022}.  This assumption allows the data to have any marginal distribution (such as negative binomial), as long as the genes interact `Gaussian-ly'.
\end{assumption}
\begin{assumption}
    \label{ass:linear-sparsity}
    Conditional dependency graphs with $d$ vertices have $O(d)$ edges ,i.e. are `linearly sparse'.  Follows from Assumption A4 of \citet{greenewald_tensor_2019}.
\end{assumption}
\begin{assumption}
    \label{ass:minimum-eigenvalue}
    The largest nonzero eigenvalue of $\mathbf{\Psi}_\ell$ is at least $\epsilon_\ell$, for some small $\epsilon_\ell$.  Follows from Assumption A2 of \citet{greenewald_tensor_2019}.
\end{assumption}
\begin{assumption}[New]
    \label{ass:low-rank-assumption}
    The input data and true graphs are well-approximated by low-rank matrices.  This assumption is common throughout data science (see Appendix \ref{sec:gmgm-bio-low-rank-approx-evidence}) but has not been applied to improve performance for multi-axis models.
\end{assumption}

Our method inherits the nonparanormality assumption from prior work.  Intuitively, this corresponds to a multivariate distribution with arbitrary marginals but whose variables interact with each other `like Gaussian variables'.  This is done by only considering rank statistics of gene expression, rather than the exact expression values.  See \citet{liu_nonparanormal_2012} for more details.  To our knowledge, there has been no Kronecker-structured work that does not make this or a stronger assumption.  If there were, it would likely lack the crucial properties of the Gaussian copula that make it efficient to scale (namely the relationship between $\mathbf{\Psi}$ and the graph of conditional dependencies) - although it is certainly a worthwhile avenue for future work to explore.

The main bottleneck to scalability in the prior version of GmGM was that its memory complexity was $O(\sum_\ell d_\ell^2 )$.  Without making any additional assumptions, this memory complexity is optimal; a set of $K$ fully dense $d_\ell \times d_\ell$ precision matrices require exactly $\sum_\ell\frac{d_\ell^2+d_\ell}{2}$ values to specify.  However, when are working with conditional dependencies, the output graphs will typically be sparse.  The `linearly sparse' assumption makes it at least theoretically possible to achieve $O(\sum d_\ell)$ memory complexity.  Real world networks are often linearly sparse.  Consider the case of a social network; no matter how many people $p$ there are in the world, there is (sadly) a constant upper bound on the possible number of friends someone can make during their life.  This yields a graph with $O(p)$ edges.

While linear sparsity makes a linear memory cost possible, all prior work relies on eigendecompositions.  This entails the use of the matrix of eigenvectors, one for each axis, requiring $\sum_\ell d_\ell^2$ elements.  To achieve scalable memory usage, this must be avoided.  Thankfully, if the dataset is well-approximated by a low-rank matrix, then only the first few eigenvectors are actually meaningful!  The full eigendecomposition is not needed.  This assumption is easily verifiable on real data; PCA is a standard part of many data analysis pipelines, with there being many techniques to estimate the number of principal components (i.e. eigenvectors) required to capture most of the variance in the data.

\subsection{The model}
\label{sec:gmgm-bio-model}

Our model is similar in structure to prior work; however, to take advantage of Assumption \ref{ass:low-rank-assumption} a `\textit{singular}' probability distribution is used; see Appendix \ref{sec:gmgm-bio-prob-deriv} for an explanation and derivation.  Letting $r_\ell$ be the rank chosen for each axis, and $r_\forall^\gamma$ be the product of ranks of all axes occurring in modality $\gamma$, our model can be expressed as follows:

\begin{align}
    \mathrm{pdf}_{\text{normal}}(\mathbf{\mathcal{X}}^\gamma) &= \frac{\sqrt{\mathrm{det}^\dagger\left(\bigoplus_{\ell \in \gamma}\mathbf{\Psi}_\ell\right)}}{\left(2\pi\right)^{\frac{r_\forall^\gamma}{2}}} e^{-\frac{1}{2}\mathrm{vec}\left[\mathcal{X}^\gamma\right]^T\left(\bigoplus_{\ell \in \gamma}\mathbf{\Psi}_\ell\right)\mathrm{vec}\left[\mathcal{X}^\gamma\right]}\notag \\
    &= \frac{\sqrt{\mathrm{det}^\dagger\left(\bigoplus_{\ell \in \gamma}\mathbf{\Psi}_\ell\right)}}{\left(2\pi\right)^{\frac{r_\forall^\gamma}{2}}} e^{-\frac{1}{2}\sum_{\ell\in\gamma} \mathrm{tr}\left[\mathbf{S}_\ell^\gamma\mathbf{\Psi}_\ell\right]}\label{eq:gmgm-bio-model}
\end{align}

When given a multi-modal dataset, the model is:

\begin{align*}
    \mathrm{pdf}_{\text{normal}}\left(\left\{\mathbf{\mathcal{X}}^\gamma\right\}\right) &=  \prod_\gamma \frac{\sqrt{\mathrm{det}^\dagger\left(\bigoplus_{\ell \in \gamma}\mathbf{\Psi}_\ell\right)}}{\left(2\pi\right)^{\frac{r_\forall^\gamma}{2}}} e^{-\frac{1}{2}\sum_{\ell\in\gamma} \mathrm{tr}\left[\mathbf{S}_\ell^\gamma\mathbf{\Psi}_\ell\right]}
\end{align*}

Above, the data is assumed to be Gaussian ($\mathbf{x} \sim \mathcal{N}(\cdot, \cdot)$).  To achieve the weaker `Gaussian copula' assumption (Assumption \ref{ass:nonparanormality}), which states that the dataset is a \textit{monotonic transformation} $f$ of a Gaussian dataset ($f^{-1}(\mathbf{x}) \sim \mathcal{N}(\cdot, \cdot)$), we follow \citet{li_scalable_2022} in using the Nonparanormal Skeptic \citep{liu_nonparanormal_2012} to first compute the sufficient statistics $\{\mathbf{S}_\ell\}$, before applying our model.

\subsection{Theory}
\label{sec:gmgm-bio-theory}

The aim of this paper is to find the maximum likelihood estimator (MLE) of Equation \ref{eq:gmgm-bio-model} efficiently.  A major bottleneck in prior work was the need to re-compute eigenvectors every iteration.  Appendix \ref{sec:gmgm-bio-max-likelihood-estimator} contains the proof for the following theorem:

\begin{theorem}
    \label{thm:gmgm-bio-partial-eigendecomposition}
    Let the rank-$r_\ell$ partial eigendecomposition of $\mathbf{S}_\ell$ be $\mathbf{V}_\ell^{(r_\ell)}\mathbf{E}_\ell^{(r_\ell)}\mathbf{V}_\ell^{(r_\ell), T}$.
    Then $\mathbf{V}_\ell^{(r_\ell)}$ are also eigenvectors of the maximum likelihood estimate $\hat{\mathbf{\Psi}}_\ell$.
\end{theorem}

This reduces the problem to finding the eigenvalues, which lack an exact solution but can be solved for iteratively due to the convexity of the problem.  The next theorem shows that this is a well-founded problem:

\begin{theorem}
    \label{thm:gmgm-bio-well-foundedness}
    There exists a unique MLE that estimates the precision matrix of the singular Kronecker-sum-structured normal distribution, provided $r_\ell \leq \mathrm{rank}\left[\mathbf{S}_\ell\right]$.

    Note that $r_\ell \leq d_{\backslash\ell}^\forall$ suffices.
\end{theorem}

Kronecker-structured models often suffer identifiability issues; in Appendix \ref{sec:gmgm-bio-identifiability} it is shown that there exists an identifiable parameterization.  In Appendix \ref{sec:gmgm-bio-fisher-information}, a hypothesis testing procedure is derived to allow one to deduce the statistically significant edges.

\begin{theorem}
    \label{thm:gmgm-bio-hypothesis-testing}
    Under the null hypothesis $\bigoplus_{\ell\in\gamma} \mathbf{\Psi}_\ell = \frac{1}{\left(\sigma^\gamma\right)^2}\mathbf{I}_{d^\gamma_\forall}$, the following distribution holds for each $\psi_{\ell_{ij}}$ independently:

    \begin{align*}
        \sqrt{\frac{\sum_{\gamma|\ell\in\gamma} d^\gamma_\forall\left(\sigma^\gamma\right)^4}{d_\ell}}\frac{\psi_{\ell_{ij}}}{2} &\mathrel{\dot{\sim}} \mathcal{N}\left(0, 1\right)
    \end{align*}
\end{theorem}

In practice, we have found this test to be quite weak, so thresholding may be preferred, especially if the graph is used as a preprocessing step (such as for clustering) rather than an end in and of itself.

\section{Algorithm}
\label{sec:gmgm-bio-algorithm}

\begin{figure*}[]
    \centering
    \includegraphics[width=0.65\textwidth]{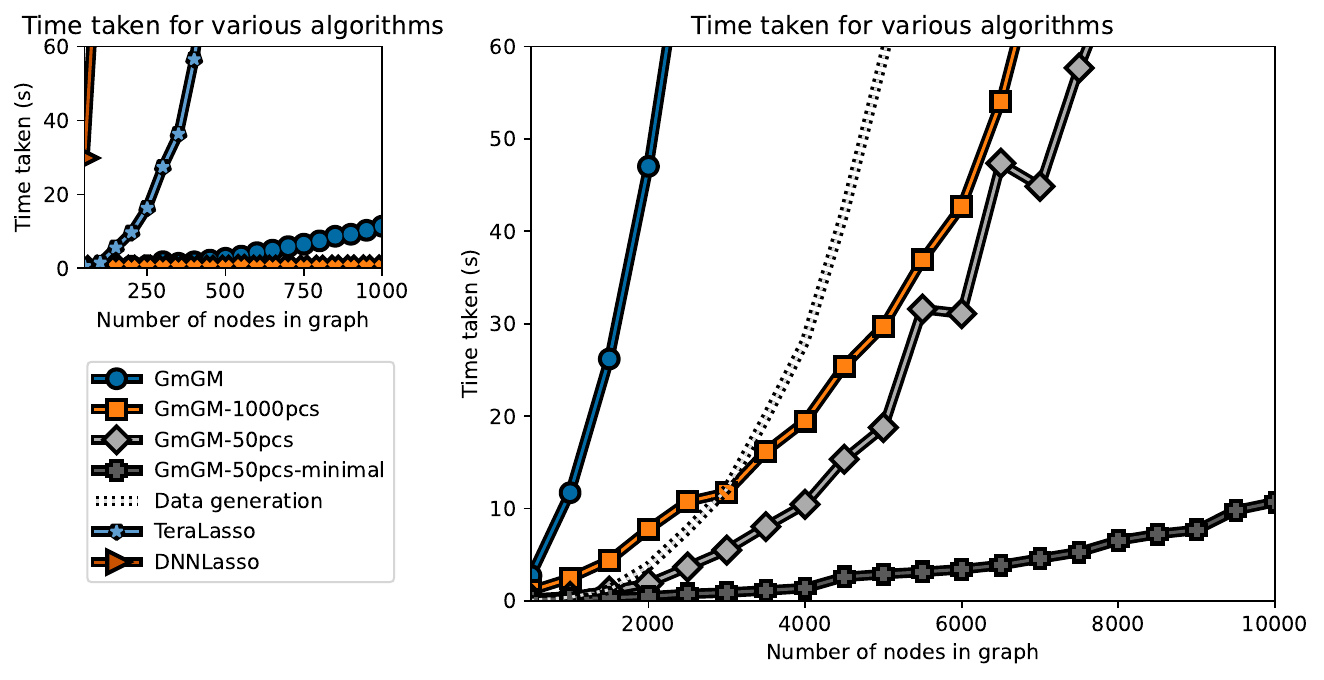}
    \caption{The running time of various multi-axis algorithms on synthetic data as the number of nodes in the graph increases.  The synthetic datasets were generated from a Kronecker-sum-structured normal distribution  The top-left focuses in on the graphs with less than 2000 nodes to be able to show TeraLasso \citep{greenewald_tensor_2019} and DNNLasso \citep{lin_dnnlasso_2024}.  `GmGM' refers to the old GmGM \citep{andrew_gmgm_2024}, whereas GmGM-Xpcs refers to our modifications, using only $X$ principal components.  GmGM-50pcs-minimal corresponds to the case in which the number of edges in the graph is proportional to the number of nodes (Assumption \ref{ass:linear-sparsity}); for the other models, the full graph was kept.}
    \label{fig:gmgm-bio-runtime}
\end{figure*}

Our maximum likelihood estimation algorithm (Appendix \ref{sec:gmgm-bio-algorithm-statement}, Algorithm \ref{alg:GmGM}) can be described as: calculate each $\mathbf{S}_\ell$, partially eigendecompose them, perform gradient descent on the eigenvalues\footnote{The formula for the gradient of the eigenvalues is derived in Appendix \ref{sec:gmgm-bio-max-likelihood-estimator}.}, eigen-recompose to get $\mathbf{\Psi}_\ell$, and then threshold to get the final solution.  However, creation of $\mathbf{S}_\ell$ requires $O(d_\ell^2)$ space - the goal is to only require $O(r_\ell d_\ell)$ space.  This can be avoided by noting that:

\begin{align*}
    \mathbf{S}_\ell &= \sum_\gamma \mathbf{S}_\ell^\gamma \\
    &= \sum_\gamma \mathrm{mat}_\ell\left[\mathcal{D}^\gamma\right] \mathrm{mat}_\ell\left[\mathcal{D}^\gamma\right]^T \\
    &= \begin{bmatrix}
        \mathrm{mat}_\ell\left[\mathcal{D}^1\right] & ... & \mathrm{mat}_\ell\left[\mathcal{D}^\Gamma\right]
    \end{bmatrix}\begin{bmatrix}
        \mathrm{mat}_\ell\left[\mathcal{D}^1\right]^T \\ ... \\ \mathrm{mat}_\ell\left[\mathcal{D}^\Gamma\right]^T
    \end{bmatrix} \\
    &\triangleq \mathrm{mat}_\ell\left[\{\mathcal{D}^\gamma\}_\gamma\right]\mathrm{mat}_\ell\left[\{\mathcal{D}^\gamma\}_\gamma\right]^T
\end{align*}

$\mathbf{S}_\ell$ can be thought of as the covariance matrix for a $d_\ell \times \sum_\gamma d_{\backslash\ell}^\gamma$ matrix ``$\mathrm{mat}_\ell\left[\{\mathcal{D}^\gamma\}_\gamma\right]$'', and hence its eigenvectors are precisely the singular vectors of $\mathrm{mat}_\ell\left[\{\mathcal{D}^\gamma\}_\gamma\right]$.  This allows avoiding ever constructing $\mathbf{S}_\ell$.

Such an approach has one major drawback; the Nonparanormal Skeptic works by supplying a rank-based $\mathbf{S}_\ell$ that, intuitively, represents the correlation matrix `as if the data had come from a Gaussian'.  To both avoid quadratic memory usage \textit{and} weaken the normality assumption, it requires a modification.  COCA \citep{han_high_2014} is used to calculate nonparanormal eigenvectors; it works by replacing $\mathrm{mat}_\ell\left[\{\mathcal{D}^\gamma\}_\gamma\right]$ with a matrix of ranks, mapping these ranks to a normal distribution, and then finding the singular vectors of this rank matrix.

The final practical issue to surmount is that of preserving input sparsity.  $\mathrm{mat}_\ell\left[\{\mathcal{D}^\gamma\}_\gamma\right]$ has $d_\forall$ elements, but in many contexts, such as transcriptomics, the dataset will be highly sparse; the majority of the input is full of zeros.  Ideally, the matrix would never be `densified'.  Using COCA, however, requires computing the matrix of ranks, which will be fully dense.  If unaddressed, this would enforce an $\Omega(d_\forall)$ minimum memory requirement.  \textbf{Solving this is necessary for scalability}; as will be seen in Section \ref{sec:gmgm-bio-scalability}, our algorithm can handle very large sparse datasets; if densified, they would be in excess of 30 GB of memory.  In Appendix \ref{sec:gmgm-bio-input-sparsity}, we show how to solve this using a rank-one update technique.

\section{Implementation}
\label{sec:gmgm-bio-implementation}

The algorithm is implemented in Python, using Numpy, Numba, SciPy, and Dask \citep{harris_array_2020, lam_numba_2015, virtanen_scipy_2020, rocklin_dask_2015}.  The experiments were run on a 13-inch MacBook Pro with an M1 chip and 8GB RAM.  23GB of hard disk storage were available when running the experiments; the million cell experiment (Section \ref{sec:gmgm-bio-scalability}) could not fit in 8GB of RAM, and thus much of the free hard disk space was used as swap memory.

\section{Results}
\label{sec:gmgm-bio-results}

The scalability of our method is demonstrated in Section \ref{sec:gmgm-bio-scalability}.  In Section \ref{sec:gmgm-bio-lncrna}, an exploration of the role of long non-coding RNA (lncRNA) in neural progenitor cells is performed.  In biology, it can be hard to know the ground truth; in Appendix \ref{sec:gmgm-bio-additional-experiments}, there is an additional comparison on a non-biological real-world toy dataset with known ground truth to validate our methodology.

\subsection{Scalability}
\label{sec:gmgm-bio-scalability}

In this section, the scalability of our method is demonstrated, comparing it to two other multi-axis methods on in-distribution synthetic data\footnote{See Appendix \ref{sec:gmgm-bio-additional-experiments} for precision/recall comparisons on said data.}.  The synthetic datasets were generated from a Kronecker-sum-structured normal distribution, and the ground truth networks to be learned were generated from a Barabasi-Albert distribution.  The Barabasi-Albert distribution was chosen as it is a power law distribution; many real world networks approximately follow a power law.  Figure \ref{fig:gmgm-bio-runtime} shows how much more scalable our modification makes GmGM, with it being the only method that can feasibly run on datasets of several thousand cells.

To emphasize this point, our method was run on a 1,248,980 cell PBMC dataset by \citet{yazar_single-cell_2022}, keeping the top 50 principal components ($\sim$50\% of the variance) and only the statistically significant edges (5890 of them).  Our method ran in less than two minutes on a personal computer (see Appendix \ref{sec:gmgm-bio-additional-experiments} for more details and a sanity-checking of the results).  It should be noted that prior multi-axis work is fundamentally not scalable to problems of this size, \textit{even with access to high performance computing clusters}.

On a theoretical level, the scalability of our method is achievable in large part due to its lower computational complexity.  For a dataset of $O(d)$ genes and cells, prior work required an $O(d^3)$ runtime and $O(d^2)$ space use; in contrast, ours requires $O(d^2)$ time and $O(d)$ space.  A derivation of the computational complexity, as well as its full statement for tensor-variate and shared-axis datasets, is given in Appendix \ref{sec:gmgm-bio-computational-complexity}.

\subsection{Application to single cell neuronal differentiation data}
\label{sec:gmgm-bio-lncrna}

\begin{figure*}[ht!]
    \centering
    \includegraphics[width=0.8\textwidth]{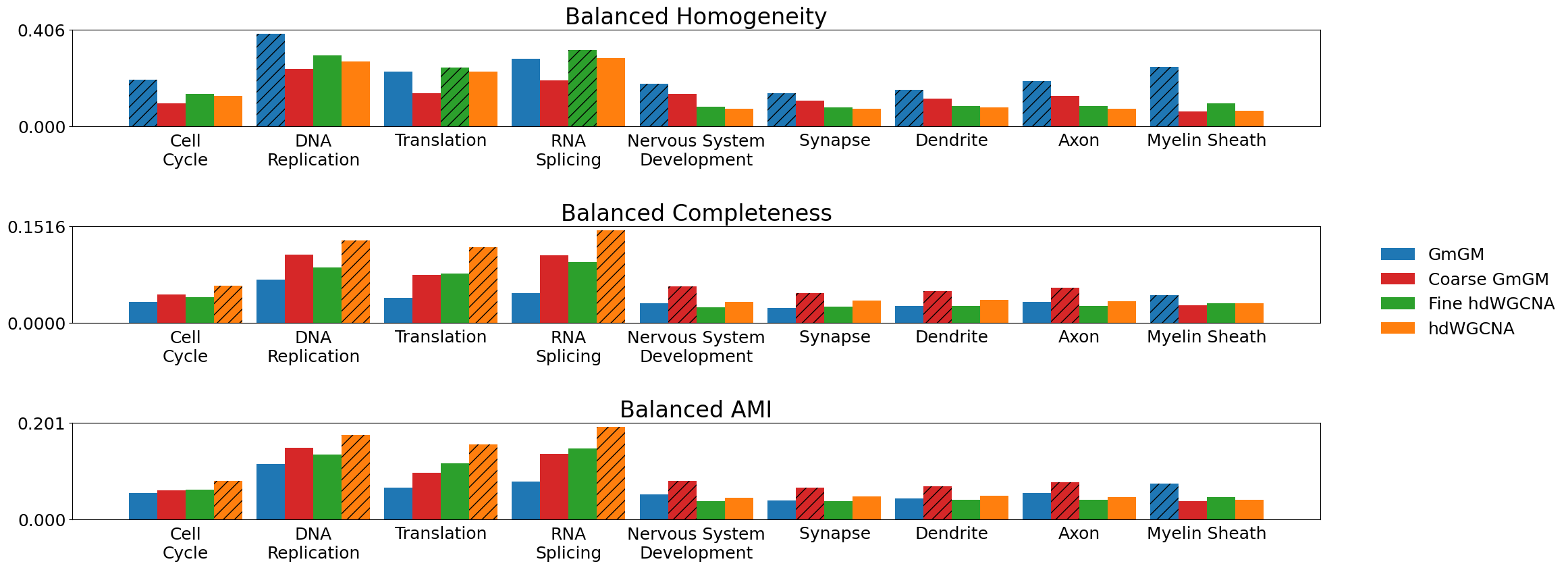}
    \caption{A comparison of the homogeneity and completeness of clusterings our method (GmGM) and hdWGCNA across various metrics over different sets of GO terms (see Appendix \ref{sec:gmgm-bio-go-term-cluster-metrics} for GO ids).  Homogeneity measures the tendency for modules to contain only genes of a given set of GO terms, whereas completeness measures the tendency for a set of GO terms to all be contained in a single module.  AMI stands for adjusted mutual information.  `Coarse GmGM' refers to a clustering on GmGM with 7 modules (the same as hdWGCNA), and `fine hdWGCNA' refers to a clustering on hdWGCNA with 119 modules (the same as GmGM), most of which were singletons.  These metrics were adjusted for chance following \citet{romano_adjusting_2016}.}
    \label{fig:gmgm-bio-go-term-cluster-metrics}
    \vspace{-10pt}
\end{figure*}

The previous section proved the scalability of our method; in this section, we will demonstrate that our method leads to real biological insights, benchmarking against the single-axis hdWGCNA \citep{morabito_hdwgcna_2023} methodology, which identifies gene co-expression networks.  {Method comparison for biological networks is difficult, as there is no ground truth; however, we will see that GmGM leads to networks with different structure than hdWGCNA, and that these differences result in new biological insights.  In particular, GmGM yields gene modules that are more specific to the relevant biology in our dataset.}  Both methods are applied to a scRNA-seq dataset of induced pluripotent stem cells (iPSCs), measured at 4 different time points as they develop into neural progenitor cells; the dataset consists of 6876 cells and 29,324 genes.  See Appendix \ref{sec:gmgm-bio-lncrna-data-details} for details on how the data was collected.

To generate the GmGM network, the top 100 principal components were kept and the nonparanormal skeptic was used. The top 5 strongest edges per gene/cell were retained.  To generate gene clusters (`modules') from the learned gene network, the Leiden \citep{traag_louvain_2019} algorithm was used with a `resolution' parameter of 3 (chosen to give at least 100 modules).  For the hdWGCNA network, `soft power' of 4 was determined to be optimal by the hdWGCNA package's TestSoftPower's method.  {We followed hdWGCNA's guidelines to ensure that the network and clustering obtained were the best possible, as defined by the authors of hdWGCNA}.  The hdWGCNA package's recommended clustering was used for the modules (only 8 modules); a finer clustering was attempted {with Leiden, although} this resulted in many singletons and very small clusters.  GmGM learns conditional dependency and hdWGCNA learns co-expression; these types of networks have fundamentally different topologies which is the likely cause of this difference.

The networks were evaluated by analyzing clusterings on them.  We identified 119 gene modules on the GmGM network, labeled m0, m1, m2, and so on. A short description of each is given in Appendix \ref{sec:gmgm-bio-lncrna-data-results} (Table \ref{tab:gmgm-bio-lnc-big-table}).  The modules range from having 37 genes (m118) to 600 genes (m2), with the exception of m1 with 1756 genes and m0 with 9047.  hdWGCNA identified 7 gene modules (ranging from 241 to 6399 genes each) - an eighth `unnassigned' module contained the 14,405 genes that were not assigned to any module by hdWGCNA.  A short description of each is given in Appendix \ref{sec:gmgm-bio-lncrna-data-results} (Table \ref{tab:gmgm-bio-hdwgcna-gene-modules}).  To limit the potentially confounding effects of cluster quantity on our analysis, an additional clustering was produced for each method.  These were an 8 module clustering on the GmGM network (`Coarse GmGM') and a 119 module clustering on the hdWGCNA network (`Fine hdWGCNA').

The networks and clusterings considered had very different properties.  The hdWGCNA network contained many singletons (the genes in the unassigned module) whereas GmGM contained none.  The GmGM network also had a higher transitivity (the tendency for connected genes to have neighbors in common) than hdWGCNA (0.055 vs 0.002).  The GmGM, coarse GmGM, fine hdWGCNA, and hdWGCNA clusterings achieved modularities (the tendency for a clustering to have more within-cluster than between-cluster edges) of 0.569, 0.600, 0.347, and 0.220, respectively.  

To compare methods in terms of their ability to extract biologically relevant information, we evaluated how well their clusterings aligned with a set of curated GO terms.  The GO terms were picked to either represent a fundamental biological process (such as the cell cycle or DNA replication) or the development of a key component of neurons (axons, dendrites, synapses, the myelin sheath, and more broadly nervous system development in general).  Appendix \ref{sec:gmgm-bio-go-term-cluster-metrics} (Table \ref{tab:gmgm-bio-go-term-choices}) contains the GO terms chosen as well as the rationale behind the choice.

The GO term comparison is given in Figure \ref{fig:gmgm-bio-go-term-cluster-metrics}.  It shows that our method's clusters are better aligned with neuron-specific GO terms, while hdWGCNA better captured non-specific processes, such as the cell cycle.  This is unsurprising; when we qualitatively inspect hdWGCNA modules (Appendix \ref{sec:gmgm-bio-lncrna-data-results}'s Table \ref{tab:gmgm-bio-hdwgcna-gene-modules}), only two of the eight modules are associated with GO terms specific to neurons (the `blue' and `red' clusters); GmGM, in contrast, had many (Appendix \ref{sec:gmgm-bio-lncrna-data-results}'s Table \ref{tab:gmgm-bio-lnc-big-table}).

In general, a qualitative analysis shows that every module of hdWGCNA (except the `unassigned' module) clearly relates to some biological process, although typically a broad process such as RNA binding.  In contrast, only 39 of GmGM's 119 modules show something of immediate biological interest. However, those that do tend to be more specific than hdWGCNA; GmGM can capture biological processes that are more relevant to the domain at hand (in this case, neural development).  For a module-by-module breakdown, see the tables of Appendix \ref{sec:gmgm-bio-lncrna-data-results}.

\begin{figure*}[]
    \centering
    \includegraphics[width=0.44\textwidth]{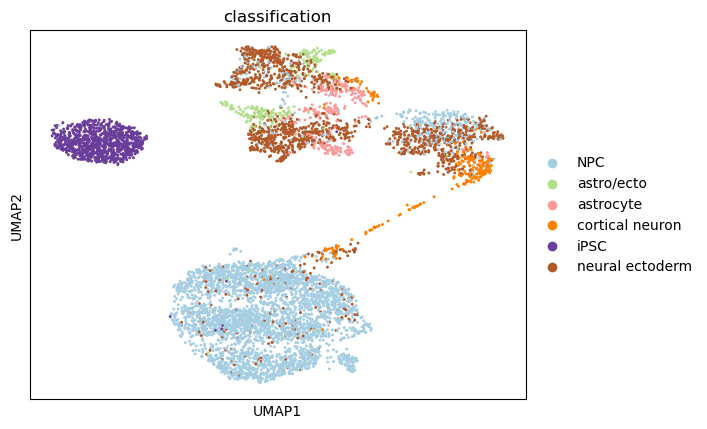}
    \includegraphics[width=0.40\textwidth]{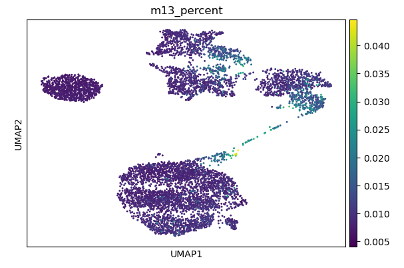}
    \caption{(left) The cell type classification produced by GmGM using the marker genes from Appendix \ref{sec:gmgm-bio-lncrna-markers}.  NPC stands for `neural progenitor cell', and `astro/ecto' denotes cells which could not be distinguished between astrocyte and neural ectoderm with the markers given. (right) The average gene expression of the genes in GmGM module m13; it tends to express itself only in and around cells that are identified as cortical neurons. (both) UMAP plots have been used for visualization; the overexpression of m13 in the suspected cortical neuron cells has also been validated statistically.}
    \vspace{-7pt}
    \label{fig:gmgm-bio-cell-types}
\end{figure*}

\begin{table*}[t!]
\centering
\begin{tabular}{l >{\raggedright\arraybackslash}p{10cm}}
\hline
Cell Type & Clusters (Overexpressed gene modules in parentheses) \\
\hline
iPSC & 
\mbox{0 (m1)}, \mbox{5 (m1)} \\ \arrayrulecolor{gray}\hline

Neural ectoderm & 
\mbox{1 (m90,m9)}; 
\mbox{3 (m6,m90,m118,m30)}; 
\mbox{6 (m14,m11,m90,m30)}; 
\mbox{15 (m13)} \\ \arrayrulecolor{gray}\hline

NPC & 
\mbox{2 (m25)}; 
\mbox{4 (m11,m19,m51,m25,m50,m62)}; 
\mbox{7 (m11,m50,m91)}; 
\mbox{8 (m11)}; 
\mbox{9 (m10)}; 
\mbox{10 (m14,m51)}; 
\mbox{12 (m11,m91,m50,m25,m51,m62)}; 
\mbox{14 (m110,m91,m11,m50,m20,m51,m10,m62,m25,m12)}; 
\mbox{16 (m51,m8)}; 
\mbox{17 (m16,m90,m62,m14)}; 
\mbox{19 (m11)}; 
\mbox{20 (m50)}; 
\mbox{21 (m11,m91,m62,m25,m51,m50)}; 
\mbox{23 (m14)} \\ \arrayrulecolor{gray}\hline

Astrocyte & 
\mbox{18 (m90,m3,m13)}; 
\mbox{22 (m9)} \\ \arrayrulecolor{gray}\hline

Cortical Neuron & 
\mbox{11 (m3,m13,m27)} \\ \arrayrulecolor{gray}\hline

Astrocyte or neural ectoderm & 
\mbox{13 (m90,m2,m6)} \\
\arrayrulecolor{black}\hline
\end{tabular}
\caption{The overexpressed modules for each GmGM cell cluster, using the marker genes from Appendix \ref{sec:gmgm-bio-lncrna-markers}'s Table \ref{tab:gmgm-bio-marker-genes} to identify cell types. Each cluster had many statistically significantly overexpressed modules; here, we are only showing the most extremely overexpressed modules.}
\label{tab:gmgm-bio-lncrna-cell-clusters}
\vspace{-7pt}
\end{table*}

Another benefit of GmGM is its ability to simultaneously learn a cell network; this is not available to hdWGCNA, which only learns the gene graph.  This cell network was clustered (Leiden, resolution=1) and each cluster was associated with a cell type based on its statistical association with a set of marker genes (Appendix \ref{sec:gmgm-bio-lncrna-markers}, Table \ref{tab:gmgm-bio-marker-genes}).  The resulting clustering is shown in Figure \ref{fig:gmgm-bio-cell-types}.

There was often an overlap between cell clusters and the expressions of genes in the gene modules; for example, the cortical neuron cells significantly express the modules m3, m13, and m27.  The relation of these clusters to cortical neurons is confirmed by their associated GO terms (including `forebrain development', `synapse organization', and `dendrite') and their dominant genes (the m27 sub-network is dominated by \textit{TBR1}, which is a key transcription factor for early cortical neurogenesis).  The overlap between GmGM's gene and cell clusters are given by Table \ref{tab:gmgm-bio-lncrna-cell-clusters}.

We conclude this section with a brief exploration of selected modules of the GmGM network.  The m90 module was especially interesting, as it had GO terms indicating a role in neurogenesis and transcription factors.  The most central genes for this module were \textit{NFIA} and \textit{NFIB} (known to be crucial for nervous system development) and the unnamed lncRNA \textit{ENSG00000243620}.  \textit{ENSG00000243620} and another unnamed lncRNA (\textit{ENSG00000286757}) were found in the GmGM network to be strongly connected
to the \textit{ZIC1}, \textit{ZIC2}, and \textit{ZIC4} transcription factors (also known to be important in neural development).  They lie on chromosomes 3 and 13, respectively; \textit{ENSG00000243620} is only 50 kilobases from \textit{ZIC1} and \textit{ZIC4}, whereas \textit{ENSG00000286757} is only 20 kilobases from \textit{ZIC2} (and \textit{ZIC5}, not in this module).  The closeness on the genome may indicate a regulatory relationship between \textit{ZIC} genes and the unnamed lncRNAs.

To see if this finding could be repeated, the two other neurogenesis-and-transcription-factor-associated modules were also investigated (m13 and m8).  m13 associated with the genes \textit{EBF2}, \textit{EBF3}, and \textit{NEUROD4}
(all of which are relevant to neural development).  The strongest connections in this module often involve the \textit{ELAVL1}, \textit{ELAVL3}, and \textit{ELAVL4} genes, one of which is between \textit{ELAVL2} and an unnamed lncRNA `\textit{ENSG00000283982}' 17 kilobases away. Finally, m8 is not particularly dominated by any one set of genes, but one of the strongest connections in the module is between \textit{SHISA9} and an unnamed lncRNA `\textit{ENSG00000262801}' 5000 base pairs away.

The hdWGCNA network was checked to see if it also identified relationships between the same lncRNAs.  The lncRNA results show that the m8 and m13 lncRNAs of interest are unimportant in the hdWGCNA network: each only connected to 5 other genes, on other chromosomes.  In contrast, the m90 lncRNAs of interest are both highly central to hdWGCNA's `brown' module, being in the top 4 most central (behind \textit{DMD} and \textit{MIR100HG}), and each having more than 100 connections to other genes, including the \textit{ZIC} genes.

\vspace{-10pt}
\section{Discussion}
\label{sec:gmgm-bio-discussion}

In this paper, we have developed a novel and general multi-axis methodology able to produce results for large-scale scRNA-seq datasets, and proven this by applying it to a million-cell dataset.  Furthermore, we show that our methodology is able to extract novel biological insights by finding promising connections between transcription factors and some poorly-studied lncRNA genes.

We compared our results to hdWGCNA, demonstrating that our method is able to create networks conducive to finer clustering, albeit with a higher proportion of clusters for which no immediate biological meaning is apparent.  The results from hdWGCNA lend credence to a potential ZIC-lncRNA relationship found by our method which warrants further study.

Finally, by simultaneously learning both the gene and cell graph, one gains multiple means of learning new biological insights and validating results.  This is a benefit our methodology has over more traditional single-axis methodologies like hdWGCNA, which only learn gene graphs.

\section{Acknowledgements}
\label{sec:gmgm-bio-acknowledgements}

Bailey Andrew is supported by the UKRI Engineering and Physical Sciences Research Council (EPSRC)
[EP/S024336/1].  David Westhead is supported in part by the National Institute for Health and Care Research (NIHR) Leeds Biomedical Research Centre (BRC) (NIHR203331). The views expressed are those of the author(s) and not necessarily those of the NHS, the NIHR or the Department of Health and Social Care.  James Poulter is supported by a UKRI Future Leaders Fellowship (MR/Y034325/1).

\section{Author Contributions}
\label{sec:gmgm-bio-author-contributions
}

Bailey Andrew (Conceptualization, Formal analysis, Methodology, Software, Writing - original draft), Erica Harris (Data curation, Resources, Investigation, Writing - original draft), James Poulter (Conceptualization, Supervision, Writing - review \& editing), David R. Westhead (Conceptualization, Supervision, Writing - review \& editing), and Luisa Cutillo (Conceptualization, Supervision, Writing - review \& editing).

\bibliography{references}

@article{romano_adjusting_2016,
	title = {Adjusting for {Chance} {Clustering} {Comparison} {Measures}},
	volume = {17},
	issn = {1533-7928},
	url = {http://jmlr.org/papers/v17/15-627.html},
	number = {134},
	urldate = {2026-02-17},
	journal = {Journal of Machine Learning Research},
	author = {Romano, Simone and Vinh, Nguyen Xuan and Bailey, James and Verspoor, Karin},
	year = {2016},
	pages = {1--32},
}

@article{langfelder_wgcna_2008,
	title = {{WGCNA}: an {R} package for weighted correlation network analysis},
	volume = {9},
	issn = {1471-2105},
	shorttitle = {{WGCNA}},
	url = {https://doi.org/10.1186/1471-2105-9-559},
	doi = {10.1186/1471-2105-9-559},
	language = {en},
	number = {1},
	urldate = {2026-01-23},
	journal = {BMC Bioinformatics},
	author = {Langfelder, Peter and Horvath, Steve},
	month = dec,
	year = {2008},
	pages = {559},
}

@article{morabito_hdwgcna_2023,
	title = {{hdWGCNA} identifies co-expression networks in high-dimensional transcriptomics data},
	volume = {3},
	issn = {2667-2375},
	doi = {10.1016/j.crmeth.2023.100498},
	language = {eng},
	number = {6},
	journal = {Cell Reports Methods},
	author = {Morabito, Samuel and Reese, Fairlie and Rahimzadeh, Negin and Miyoshi, Emily and Swarup, Vivek},
	month = jun,
	year = {2023},
	pmid = {37426759},
	pmcid = {PMC10326379},
	pages = {100498},
}

@inproceedings{lin_dnnlasso_2024,
	title = {{DNNLasso}: {Scalable} {Graph} {Learning} for {Matrix}-{Variate} {Data}},
	shorttitle = {{DNNLasso}},
	url = {https://proceedings.mlr.press/v238/lin24b.html},
	language = {en},
	urldate = {2026-01-15},
	booktitle = {Proceedings of {The} 27th {International} {Conference} on {Artificial} {Intelligence} and {Statistics}},
	publisher = {PMLR},
	author = {Lin, Meixia and Zhang, Yangjing},
	month = apr,
	year = {2024},
	note = {ISSN: 2640-3498},
	pages = {316--324},
}

@article{regev_human_2017,
	title = {The {Human} {Cell} {Atlas}},
	volume = {6},
	issn = {2050-084X},
	doi = {10.7554/eLife.27041},
	language = {eng},
	journal = {eLife},
	author = {Regev, Aviv and Teichmann, Sarah A. and Lander, Eric S. and Amit, Ido and Benoist, Christophe and Birney, Ewan and Bodenmiller, Bernd and Campbell, Peter and Carninci, Piero and Clatworthy, Menna and Clevers, Hans and Deplancke, Bart and Dunham, Ian and Eberwine, James and Eils, Roland and Enard, Wolfgang and Farmer, Andrew and Fugger, Lars and Göttgens, Berthold and Hacohen, Nir and Haniffa, Muzlifah and Hemberg, Martin and Kim, Seung and Klenerman, Paul and Kriegstein, Arnold and Lein, Ed and Linnarsson, Sten and Lundberg, Emma and Lundeberg, Joakim and Majumder, Partha and Marioni, John C. and Merad, Miriam and Mhlanga, Musa and Nawijn, Martijn and Netea, Mihai and Nolan, Garry and Pe'er, Dana and Phillipakis, Anthony and Ponting, Chris P. and Quake, Stephen and Reik, Wolf and Rozenblatt-Rosen, Orit and Sanes, Joshua and Satija, Rahul and Schumacher, Ton N. and Shalek, Alex and Shapiro, Ehud and Sharma, Padmanee and Shin, Jay W. and Stegle, Oliver and Stratton, Michael and Stubbington, Michael J. T. and Theis, Fabian J. and Uhlen, Matthias and van Oudenaarden, Alexander and Wagner, Allon and Watt, Fiona and Weissman, Jonathan and Wold, Barbara and Xavier, Ramnik and Yosef, Nir and {Human Cell Atlas Meeting Participants}},
	month = dec,
	year = {2017},
	pmid = {29206104},
	pmcid = {PMC5762154},
	pages = {e27041},
}

@article{meinshausen_high-dimensional_2006,
	title = {High-dimensional graphs and variable selection with the {Lasso}},
	volume = {34},
	issn = {0090-5364, 2168-8966},
	url = {https://projecteuclid.org/journals/annals-of-statistics/volume-34/issue-3/High-dimensional-graphs-and-variable-selection-with-the-Lasso/10.1214/009053606000000281.full},
	doi = {10.1214/009053606000000281},
	number = {3},
	urldate = {2025-07-25},
	journal = {The Annals of Statistics},
	author = {Meinshausen, Nicolai and Bühlmann, Peter},
	month = jun,
	year = {2006},
	note = {Publisher: Institute of Mathematical Statistics},
	pages = {1436--1462},
}

@article{friedman_sparse_2008,
	title = {Sparse inverse covariance estimation with the graphical lasso},
	volume = {9},
	issn = {1465-4644},
	url = {https://doi.org/10.1093/biostatistics/kxm045},
	doi = {10.1093/biostatistics/kxm045},
	number = {3},
	urldate = {2025-07-25},
	journal = {Biostatistics},
	author = {Friedman, Jerome and Hastie, Trevor and Tibshirani, Robert},
	month = jul,
	year = {2008},
	pages = {432--441},
}

@article{nene_columbia_1996,
	title = {Columbia {Object} {Image} {Library} ({COIL}-20)},
	language = {en},
	author = {Nene, Sameer A and Nayar, Shree K and Murase, Hiroshi},
	month = feb,
	year = {1996},
}

@article{greenewald_tensor_2019,
	title = {Tensor {Graphical} {Lasso} ({TeraLasso})},
	volume = {81},
	issn = {1369-7412},
	url = {https://doi.org/10.1111/rssb.12339},
	doi = {10.1111/rssb.12339},
	number = {5},
	urldate = {2025-07-21},
	journal = {Journal of the Royal Statistical Society Series B: Statistical Methodology},
	author = {Greenewald, Kristjan and Zhou, Shuheng and Hero, III, Alfred},
	month = nov,
	year = {2019},
	pages = {901--931},
}

@article{brand_fast_2006,
	series = {Special {Issue} on {Large} {Scale} {Linear} and {Nonlinear} {Eigenvalue} {Problems}},
	title = {Fast low-rank modifications of the thin singular value decomposition},
	volume = {415},
	issn = {0024-3795},
	url = {https://www.sciencedirect.com/science/article/pii/S0024379505003812},
	doi = {10.1016/j.laa.2005.07.021},
	number = {1},
	urldate = {2024-07-08},
	journal = {Linear Algebra and its Applications},
	author = {Brand, Matthew},
	month = may,
	year = {2006},
	pages = {20--30},
}

@article{yazar_single-cell_2022,
	title = {Single-cell {eQTL} mapping identifies cell type–specific genetic control of autoimmune disease},
	volume = {376},
	url = {https://www.science.org/doi/10.1126/science.abf3041},
	doi = {10.1126/science.abf3041},
	number = {6589},
	urldate = {2024-07-06},
	journal = {Science},
	author = {Yazar, Seyhan and Alquicira-Hernandez, Jose and Wing, Kristof and Senabouth, Anne and Gordon, M. Grace and Andersen, Stacey and Lu, Qinyi and Rowson, Antonia and Taylor, Thomas R. P. and Clarke, Linda and Maccora, Katia and Chen, Christine and Cook, Anthony L. and Ye, Chun Jimmie and Fairfax, Kirsten A. and Hewitt, Alex W. and Powell, Joseph E.},
	month = apr,
	year = {2022},
	note = {Publisher: American Association for the Advancement of Science},
	pages = {eabf3041},
}

@article{kolberg_gprofilerinteroperable_2023,
	title = {g:{Profiler}—interoperable web service for functional enrichment analysis and gene identifier mapping (2023 update)},
	volume = {51},
	issn = {0305-1048},
	shorttitle = {g},
	url = {https://doi.org/10.1093/nar/gkad347},
	doi = {10.1093/nar/gkad347},
	number = {W1},
	urldate = {2024-07-06},
	journal = {Nucleic Acids Research},
	author = {Kolberg, Liis and Raudvere, Uku and Kuzmin, Ivan and Adler, Priit and Vilo, Jaak and Peterson, Hedi},
	month = jul,
	year = {2023},
	pages = {W207--W212},
}

@article{srinivasan_what_2023,
	title = {What is the gradient of a scalar function of a symmetric matrix?},
	volume = {54},
	issn = {0975-7465},
	url = {https://doi.org/10.1007/s13226-022-00313-x},
	doi = {10.1007/s13226-022-00313-x},
	language = {en},
	number = {3},
	urldate = {2024-07-02},
	journal = {Indian Journal of Pure and Applied Mathematics},
	author = {Srinivasan, Shriram and Panda, Nishant},
	month = sep,
	year = {2023},
	pages = {907--919},
}

@article{holbrook_differentiating_2018,
	title = {Differentiating the pseudo determinant},
	volume = {548},
	issn = {0024-3795},
	url = {https://www.sciencedirect.com/science/article/pii/S0024379518301289},
	doi = {10.1016/j.laa.2018.03.018},
	urldate = {2024-06-26},
	journal = {Linear Algebra and its Applications},
	author = {Holbrook, Andrew},
	month = jul,
	year = {2018},
	pages = {293--304},
}

@inproceedings{andrew_gmgm_2024,
	title = {{GmGM}: a fast multi-axis {Gaussian} graphical model},
	shorttitle = {{GmGM}},
	url = {https://proceedings.mlr.press/v238/b-andrew24a.html},
	language = {en},
	urldate = {2024-06-26},
	booktitle = {Proceedings of {The} 27th {International} {Conference} on {Artificial} {Intelligence} and {Statistics}},
	publisher = {PMLR},
	author = {Andrew, Ethan B. and Westhead, David and Cutillo, Luisa},
	month = apr,
	year = {2024},
	note = {ISSN: 2640-3498},
	pages = {2053--2061},
}

@article{traag_louvain_2019,
	title = {From {Louvain} to {Leiden}: guaranteeing well-connected communities},
	volume = {9},
	copyright = {2019 The Author(s)},
	issn = {2045-2322},
	shorttitle = {From {Louvain} to {Leiden}},
	url = {https://www.nature.com/articles/s41598-019-41695-z},
	doi = {10.1038/s41598-019-41695-z},
	language = {en},
	number = {1},
	urldate = {2024-02-22},
	journal = {Scientific Reports},
	author = {Traag, V. A. and Waltman, L. and van Eck, N. J.},
	month = mar,
	year = {2019},
	note = {Number: 1
Publisher: Nature Publishing Group},
	pages = {5233},
}

@inproceedings{mccarter_sparse_2014,
	title = {On {Sparse} {Gaussian} {Chain} {Graph} {Models}},
	volume = {27},
	url = {https://papers.nips.cc/paper_files/paper/2014/hash/81c650caac28cdefce4de5ddc18befa0-Abstract.html},
	urldate = {2024-02-22},
	booktitle = {Advances in {Neural} {Information} {Processing} {Systems}},
	publisher = {Curran Associates, Inc.},
	author = {McCarter, Calvin and Kim, Seyoung},
	year = {2014},
}

@inproceedings{rocklin_dask_2015,
	title = {Dask: {Parallel} {Computation} with {Blocked} algorithms and {Task} {Scheduling}},
	shorttitle = {Dask},
	doi = {10.25080/Majora-7b98e3ed-013},
	author = {Rocklin, Matthew},
	month = jan,
	year = {2015},
	pages = {126--132},
}

@inproceedings{lam_numba_2015,
	address = {New York, NY, USA},
	series = {{LLVM} '15},
	title = {Numba: a {LLVM}-based {Python} {JIT} compiler},
	isbn = {978-1-4503-4005-2},
	shorttitle = {Numba},
	url = {https://dl.acm.org/doi/10.1145/2833157.2833162},
	doi = {10.1145/2833157.2833162},
	urldate = {2024-01-18},
	booktitle = {Proceedings of the {Second} {Workshop} on the {LLVM} {Compiler} {Infrastructure} in {HPC}},
	publisher = {Association for Computing Machinery},
	author = {Lam, Siu Kwan and Pitrou, Antoine and Seibert, Stanley},
	month = nov,
	year = {2015},
	pages = {1--6},
}

@article{virtanen_scipy_2020,
	title = {{SciPy} 1.0: fundamental algorithms for scientific computing in {Python}},
	volume = {17},
	copyright = {2020 The Author(s)},
	issn = {1548-7105},
	shorttitle = {{SciPy} 1.0},
	url = {https://www.nature.com/articles/s41592-019-0686-2},
	doi = {10.1038/s41592-019-0686-2},
	language = {en},
	number = {3},
	urldate = {2024-01-18},
	journal = {Nature Methods},
	author = {Virtanen, Pauli and Gommers, Ralf and Oliphant, Travis E. and Haberland, Matt and Reddy, Tyler and Cournapeau, David and Burovski, Evgeni and Peterson, Pearu and Weckesser, Warren and Bright, Jonathan and van der Walt, Stéfan J. and Brett, Matthew and Wilson, Joshua and Millman, K. Jarrod and Mayorov, Nikolay and Nelson, Andrew R. J. and Jones, Eric and Kern, Robert and Larson, Eric and Carey, C. J. and Polat, İlhan and Feng, Yu and Moore, Eric W. and VanderPlas, Jake and Laxalde, Denis and Perktold, Josef and Cimrman, Robert and Henriksen, Ian and Quintero, E. A. and Harris, Charles R. and Archibald, Anne M. and Ribeiro, Antônio H. and Pedregosa, Fabian and van Mulbregt, Paul},
	month = mar,
	year = {2020},
	note = {Number: 3
Publisher: Nature Publishing Group},
	pages = {261--272},
}

@article{han_high_2014,
	title = {High {Dimensional} {Semiparametric} {Scale}-{Invariant} {Principal} {Component} {Analysis}},
	volume = {36},
	issn = {0162-8828},
	url = {https://www.ncbi.nlm.nih.gov/pmc/articles/PMC5266498/},
	doi = {10.1109/TPAMI.2014.2307886},
	number = {10},
	urldate = {2024-01-18},
	journal = {IEEE transactions on pattern analysis and machine intelligence},
	author = {Han, Fang and Liu, Han},
	month = oct,
	year = {2014},
	pmid = {26352632},
	pmcid = {PMC5266498},
	pages = {2016--2032},
}

@inproceedings{liu_nonparanormal_2012,
	title = {The nonparanormal {SKEPTIC}},
	url = {https://pure.johnshopkins.edu/en/publications/the-nonparanormal-skeptic-4},
	language = {English (US)},
	urldate = {2023-10-09},
	booktitle = {Proceedings of the 29th {International} {Conference} on {Machine} {Learning}, {ICML} 2012},
	author = {Liu, Han and Han, Fang and Yuan, Ming and Lafferty, John and Wasserman, Larry},
	month = oct,
	year = {2012},
	pages = {1415--1422},
}

@article{cai_joint_2016,
	title = {Joint {Estimation} of {Multiple} {High}-dimensional {Precision} {Matrices}},
	volume = {26},
	issn = {1017-0405},
	url = {https://www.ncbi.nlm.nih.gov/pmc/articles/PMC5351783/},
	doi = {10.5705/ss.2014.256},
	number = {2},
	urldate = {2023-05-17},
	journal = {Statistica Sinica},
	author = {Cai, T. Tony and Li, Hongzhe and Liu, Weidong and Xie, Jichun},
	month = apr,
	year = {2016},
	pmid = {28316451},
	pmcid = {PMC5351783},
	pages = {445--464},
}

@article{danaher_joint_2014,
	title = {The joint graphical lasso for inverse covariance estimation across multiple classes},
	volume = {76},
	issn = {1369-7412},
	url = {https://www.ncbi.nlm.nih.gov/pmc/articles/PMC4012833/},
	doi = {10.1111/rssb.12033},
	number = {2},
	urldate = {2023-05-17},
	journal = {Journal of the Royal Statistical Society. Series B, Statistical methodology},
	author = {Danaher, Patrick and Wang, Pei and Witten, Daniela M.},
	month = mar,
	year = {2014},
	pmid = {24817823},
	pmcid = {PMC4012833},
	pages = {373--397},
}

@article{wolf_scanpy_2018,
	title = {{SCANPY}: large-scale single-cell gene expression data analysis},
	volume = {19},
	issn = {1474-760X},
	shorttitle = {{SCANPY}},
	url = {https://doi.org/10.1186/s13059-017-1382-0},
	doi = {10.1186/s13059-017-1382-0},
	number = {1},
	urldate = {2023-05-16},
	journal = {Genome Biology},
	author = {Wolf, F. Alexander and Angerer, Philipp and Theis, Fabian J.},
	month = feb,
	year = {2018},
	pages = {15},
}

@article{harris_array_2020,
	title = {Array programming with {NumPy}},
	volume = {585},
	copyright = {2020 The Author(s)},
	issn = {1476-4687},
	url = {https://www.nature.com/articles/s41586-020-2649-2},
	doi = {10.1038/s41586-020-2649-2},
	language = {en},
	number = {7825},
	urldate = {2023-05-16},
	journal = {Nature},
	author = {Harris, Charles R. and Millman, K. Jarrod and van der Walt, Stéfan J. and Gommers, Ralf and Virtanen, Pauli and Cournapeau, David and Wieser, Eric and Taylor, Julian and Berg, Sebastian and Smith, Nathaniel J. and Kern, Robert and Picus, Matti and Hoyer, Stephan and van Kerkwijk, Marten H. and Brett, Matthew and Haldane, Allan and del Río, Jaime Fernández and Wiebe, Mark and Peterson, Pearu and Gérard-Marchant, Pierre and Sheppard, Kevin and Reddy, Tyler and Weckesser, Warren and Abbasi, Hameer and Gohlke, Christoph and Oliphant, Travis E.},
	month = sep,
	year = {2020},
	note = {Number: 7825
Publisher: Nature Publishing Group},
	pages = {357--362},
}

@misc{li_scalable_2022,
	title = {Scalable {Bigraphical} {Lasso}: {Two}-way {Sparse} {Network} {Inference} for {Count} {Data}},
	shorttitle = {Scalable {Bigraphical} {Lasso}},
	url = {http://arxiv.org/abs/2203.07912},
	urldate = {2023-02-24},
	publisher = {arXiv},
	author = {Li, Sijia and López-García, Martín and Lawrence, Neil D. and Cutillo, Luisa},
	month = mar,
	year = {2022},
	note = {arXiv:2203.07912 [cs, stat]},
}

@misc{dahl_network_2013,
	title = {Network inference in matrix-variate {Gaussian} models with non-independent noise},
	url = {http://arxiv.org/abs/1312.1622},
	doi = {10.48550/arXiv.1312.1622},
	urldate = {2023-02-27},
	publisher = {arXiv},
	author = {Dahl, Andy and Hore, Victoria and Iotchkova, Valentina and Marchini, Jonathan},
	month = dec,
	year = {2013},
	note = {arXiv:1312.1622 [stat]},
}

@inproceedings{kalaitzis_bigraphical_2013,
	title = {The {Bigraphical} {Lasso}},
	url = {https://proceedings.mlr.press/v28/kalaitzis13.html},
	language = {en},
	urldate = {2023-02-24},
	booktitle = {Proceedings of the 30th {International} {Conference} on {Machine} {Learning}},
	publisher = {PMLR},
	author = {Kalaitzis, Alfredo and Lafferty, John and Lawrence, Neil D. and Zhou, Shuheng},
	month = may,
	year = {2013},
	note = {ISSN: 1938-7228},
	pages = {1229--1237},
}

\newpage

\onecolumn

\appendix

\section{Evidence for the low-rank approximability assumption}
\label{sec:gmgm-bio-low-rank-approx-evidence}

\begin{figure}[h!]
    \centering
    \includegraphics[width=0.5\linewidth]{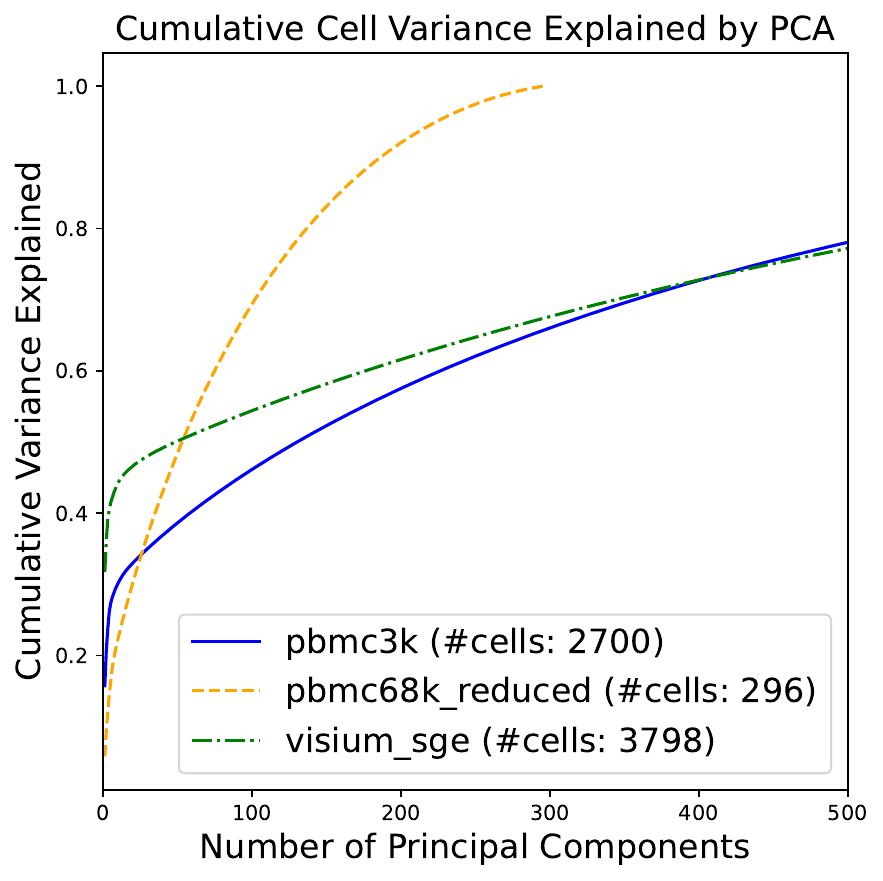}
    \caption{The cumulative variance explained by low-rank approximations (PCA) for three scRNA-seq datasets.}
    \label{fig:gmgm-bio-low-rank}
\end{figure}

The low-rank approximability assumption (Assumption \ref{ass:low-rank-assumption}) is frequently made in practice; whenever one performs PCA, this assumption is being made.  This section aims to demonstrate the reasonability of Assumption \ref{ass:low-rank-assumption}.  As PCA is likely to be most practitioner's primary encounter with low-rank approximations, we will frame our argument in terms of PCA.

The number of principal components to chose (and hence the rank of the low-rank approximation) is typically chosen by looking at the total explained variance of the components.  An explained variance of 100\% indicates perfect representation of the original matrix by the low-rank approximation.  One method of picking the amount of principal components is to simply pick an amount that gives a satisfactory recovery rate, such as 60\% or 80\%.  However, most components often represent random noise; we do not care about recovering them.  An alternative selection procedure is to look for `sharp corners' in the plot of cumulative variance, indicating the transition from informative to uninformative components.  The first approach is liable to pick too many components, whereas the second approach is liable to pick too few.

Figure \ref{fig:gmgm-bio-low-rank} displays the cumulative variance plot for three scRNA-seq datasets available in the ScanPy package and produced by 10X Genomics \citep{wolf_scanpy_2018}.  Figure \ref{fig:gmgm-bio-low-rank} shows that one does not need too many principal components to achieve good recovery; the explained variance approach requires a couple hundred components and the informative components approach requires less than a hundred for the larger datasets.  This validates our use of low-rank approximations.

\section{Derivation of the probability distribution}
\label{sec:gmgm-bio-prob-deriv}

The space of all full-rank $d\times d$ matrices has dimension $d^2$.  In contrast, the space of all rank-$r$ $d\times d$ matrices has dimension $2dr-r^2$.  Thus, any probability distribution over rank-$r$ matrices must be degenerate, or `singular', in the sense that the density is only positive on a subset of measure 0.  First, let's consider a multivariate Gaussian of independent variables:

\begin{align*}
    &&\mathbf{x} &\sim \mathcal{N}\left(\mathbf{0}, \mathrm{diag}\left[\boldsymbol{\sigma}\right]\right) \\
    \iff&& \mathrm{pdf}\left(\mathbf{x}\right) &= \frac{\sqrt{\prod_i \frac{1}{\sigma_i}}}{\left(2\pi\right)^{\frac{r}{2}}} e^{-\frac{1}{2}\mathbf{x}^T\mathrm{diag}\left[\boldsymbol{\sigma}\right]^{-1}\mathbf{x}} \tag{$r$ Independent Normal Variables}
\end{align*}

Suppose we then left-multiply $\mathbf{x}$ by an $r \times r$ orthonormal matrix $\mathbf{V}$ to introduce dependencies into our data while preserving Gaussianity.  Letting $\mathbf{V}\mathrm{diag}\left[\boldsymbol{\sigma}\right]^{-1}\mathbf{V} = \mathbf{\Psi}$, we get:

\begin{align*}
    \mathrm{pdf}\left(\mathbf{V}\mathbf{x}\right) &= \frac{\sqrt{\prod_i \frac{1}{\sigma_i}}}{\left(2\pi\right)^{\frac{r}{2}}} e^{-\frac{1}{2}\mathbf{x}^T\mathbf{V}\mathrm{diag}\left[\boldsymbol{\sigma}\right]^{-1}\mathbf{V}^T\mathbf{x}} \\
    &= \frac{\sqrt{\det \mathbf{\Psi}}}{\left(2\pi\right)^{\frac{r}{2}}}e^{-\frac{1}{2}\mathbf{x}^T\mathbf{\Psi}\mathbf{x}}
\end{align*}

Notice that $\mathbf{\Psi}$ is fully general; $\mathbf{V}$ are the eigenvectors and $\boldsymbol{\sigma}^{-1}$ are the eigenvalues.  To achieve a singular probability distribution, we multiply by a $d \times r$ orthonormal matrix $\mathbf{U}$ ($d > r$) instead.  The only noticeable change in the distribution's specification is the use of the pseudo-determinant (the product of only the non-zero eigenvalues of $\mathbf{\Psi}=\mathbf{U}\mathrm{diag}\left[\boldsymbol{\sigma}\right]\mathbf{U}^T$).

\begin{align*}
    \mathrm{pdf}\left(\mathbf{U}\mathbf{x}\right) &= \frac{\sqrt{\det^\dagger \mathbf{\Psi}}}{\left(2\pi\right)^{\frac{r}{2}}}e^{-\frac{1}{2}\mathbf{x}^T\mathbf{\Psi}\mathbf{x}}
\end{align*}

The above formulas were for vector-variate $\mathbf{x}$; for tensor-variate $\mathcal{X}$ we can use the identity $\mathrm{vec} \left[\mathcal{X} \times_1 \mathbf{V}_1 \times_2 \cdots \times_K \mathbf{V}_K\right] = \left(\bigoplus_{\ell=K}^1 \mathbf{V}_\ell\right)\mathrm{vec}\left[\mathcal{X}\right]$.  It is then straightforward to present a Kronecker-sum-structured version of this distribution, the `singular Kronecker-sum-structured normal distribution'.  We let $r_\ell$ be the rank chosen for each axis, and $r_\forall^\gamma$ be the product of ranks of all axes occurring in modality $\gamma$.

\begin{align*}
    \mathrm{pdf}_{\text{normal}}(\mathbf{\mathcal{X}}^\gamma) &= \frac{\sqrt{\mathrm{det}^\dagger\left(\bigoplus_{\ell \in \gamma}\mathbf{\Psi}_\ell\right)}}{\left(2\pi\right)^{\frac{r_\forall^\gamma}{2}}} e^{-\frac{1}{2}\mathrm{vec}\left[\mathcal{X}^\gamma\right]^T\left(\bigoplus_{\ell \in \gamma}\mathbf{\Psi}_\ell\right)\mathrm{vec}\left[\mathcal{X}^\gamma\right]} \\
    &= \frac{\sqrt{\mathrm{det}^\dagger\left(\bigoplus_{\ell \in \gamma}\mathbf{\Psi}_\ell\right)}}{\left(2\pi\right)^{\frac{r_\forall^\gamma}{2}}} e^{-\frac{1}{2}\sum_{\ell\in\gamma} \mathrm{tr}\left[\mathbf{S}_\ell^\gamma\mathbf{\Psi}_\ell\right]}
\end{align*}

When given a multi-modal dateset, we have:

\begin{align*}
    \mathrm{pdf}_{\text{normal}}\left(\left\{\mathbf{\mathcal{X}}^\gamma\right\}\right) &=  \prod_\gamma \frac{\sqrt{\mathrm{det}^\dagger\left(\bigoplus_{\ell \in \gamma}\mathbf{\Psi}_\ell\right)}}{\left(2\pi\right)^{\frac{r_\forall^\gamma}{2}}} e^{-\frac{1}{2}\sum_{\ell\in\gamma} \mathrm{tr}\left[\mathbf{S}_\ell^\gamma\mathbf{\Psi}_\ell\right]}
\end{align*}


\section{Identifiability}
\label{sec:gmgm-bio-identifiability}

The Kronecker sum decomposition is typically parameterized by a non-identifiable representation of $\bigoplus_{\ell\in\gamma} \mathbf{\Psi}_\ell$.  Let $\{c_\ell\}_{\ell\in\gamma}$ be a set of constants such that $\sum_{\ell\in\gamma} c_\ell = 0$, then: \[\bigoplus_{\ell\in\gamma} \left(\mathbf{\Psi}_\ell + c_\ell \mathbf{I}_{d_\ell}\right) = \bigoplus_{\ell\in\gamma}\mathbf{\Psi}_\ell + \left(\sum_{\ell\in\gamma}c_\ell\right)\mathbf{I}_{d_\forall^\gamma} = \bigoplus_{\ell\in\gamma}\mathbf{\Psi}_\ell\]

This is not unique to the Kronecker sum model; all Kronecker-separable models experience some form of non-identifiability.  However, this non-identifiability is rather benign.  Firstly, it only affects the diagonals of our estimates; the off-diagonals (which encode the graph structure) are unaffected.  Secondly, \cite{greenewald_tensor_2019} derived the following identifiable parameterization for the single-modality case:

\begin{align*}
    \mathbf{\Omega}^\gamma &= \left(\bigoplus_{\ell\in\gamma}\tilde{\mathbf{\Psi}}_\ell\right) + \tau^\gamma \mathbf{I}_{d_\forall^\gamma} \\
    where &\hspace{5pt}\mathrm{tr}\left[\tilde{\mathbf{\Psi}}_\ell\right] = 0
\end{align*}

We can map the standard parameterization to this one with ease.  As a shorthand, let $\tau_\ell = \mathrm{tr}\left[\mathbf{\Psi}_\ell\right]$:

\begin{align*}
    \tau^\gamma &= \sum_{\ell\in\gamma} d_{\backslash\ell}^\gamma \tau_\ell \\
    \tilde{\mathbf{\Psi}}_\ell &= \mathbf{\Psi}_\ell - \tau_\ell\mathbf{I}_{d_\ell}
\end{align*}

However, \citeauthor{greenewald_tensor_2019}'s parameterization does not apply to the multi-modality case.  Ideally, we would wish to parameterize our dataset by the parameters $\{\tau^\gamma\}_\gamma$ and $\{\tilde{\mathbf{\Psi}}_\ell\}_\ell$, but there are linear dependencies in our parameterization which may prevent this.

\begin{amsthmexample}
    \label{ex:first-lindep-param}
    Consider the bi-modal dataset where $\gamma_1 = (\ell_1, \ell_2)$ and $\gamma_2 = (\ell_1, \ell_2, \ell_3)$.  There is a linear dependency in our trace parameters:
    
    \begin{align*}
        \tau^{\gamma_1} &= d_2 \tau_1 + d_1\tau_2 \\
        \tau^{\gamma_2} &= d_2 d_3 \tau_1 + d_1d_3\tau_2 + d_2d_3\tau_3 \\
        &= d_3\tau^{\gamma_1} + d_2d_3\tau_3
    \end{align*}
\end{amsthmexample}

In Example \ref{ex:first-lindep-param}, the parameterization is still valid - we need both $\tau^{\gamma_1}$ and $\tau^{\gamma_2}$ to identify the parameterization.  However, we also get the trace of the third axis, $\tau_3$, as a byproduct.  Unlike in the single modality case, in this case $\mathbf{\Psi}_3$ is fully-specified.

\begin{amsthmexample}
    \label{ex:gmgm-bio-second-lindep-param}
    Consider the tri-modal dataset where $\gamma_1 = (\ell_1, \ell_2)$, $\gamma_2 = (\ell_3, \ell_4)$, and $\gamma_3 = (\ell_1, \ell_2, \ell_3, \ell_4)$.  Then $\tau^{\gamma_3} = d_3d_4\tau^{\gamma_1} + d_1d_2\tau^{\gamma_2}$.
\end{amsthmexample}

Example \ref{ex:gmgm-bio-second-lindep-param} shows that \citeauthor{greenewald_tensor_2019}'s parameterization can be over-determined in the multi-modal case.  GmGM can be used on arbitrarily complicated multi-tensor-variate data; how can we know when a parameterization is over-determined or not?

Note that $\{\tau^\gamma\}$ are always a linear combination of $\tau_\ell$.  Letting $d_{\backslash\ell}^\gamma = 0$ if $\ell \notin \gamma$, we can express this linear relationship as:

\begin{align*}
    \begin{bmatrix}
        \tau^{\gamma_1} \\ \vdots \\ \tau^{\gamma_\Gamma}
    \end{bmatrix} &= \mathbf{M}_{\boldsymbol{\gamma}}\begin{bmatrix}
        \tau_{\ell_1} \\ \vdots \\ \tau_{\ell_L}
    \end{bmatrix} \\
    \mathbf{M}_{\boldsymbol{\gamma}} &= \begin{bmatrix}
        d^{\gamma_1}_{\backslash\ell_1} & \cdots & d^{\gamma_1}_{\backslash\ell_L} \\
        \vdots & \ddots & \vdots \\
        d^{\gamma_\Gamma}_{\backslash\ell_1} & \cdots & d^{\gamma_\Gamma}_{\backslash\ell_L}
    \end{bmatrix}
\end{align*}

Depending on the modalities and the lengths of each axis, $\mathbf{M}_{\boldsymbol{\gamma}}$ could be practically any $\Gamma \times L$ matrix with nonnegative integer coefficients, so it is hard to say much about it in general.  However, we can make a few observations.  In particular, $\mathbf{M}_{\boldsymbol{\gamma}}$ describes how our identifiable quantities $\{\tau^{\gamma}\}_\gamma$ are linked through the latent factors $\{\tau_\ell\}_\ell$.  Zero-eigenvectors $\mathbf{x}$ of $\mathbf{M}_{\boldsymbol{\gamma}}$ correspond to the non-identifiabilities, as $\mathbf{M}_{\boldsymbol{\gamma}} \left(\mathbf{x}+\mathbf{y}\right) = \mathbf{M}_{\boldsymbol{\gamma}} \mathbf{y}$.  Our parameter vector $\begin{bmatrix}\tau^{\gamma_1} & \cdots & \tau^{\gamma_\Gamma}\end{bmatrix}^T$ must lie in the \textit{rowspace} of $\mathbf{M}_{\boldsymbol{\gamma}}$.

For most datasets, $\mathbf{M}_{\boldsymbol{\gamma}}$ has maximal row rank, in which case the approach of \citeauthor{greenewald_tensor_2019} is unchanged, and we can parameterize by $\{\tau^\gamma\}_\gamma$.  However, there are sometimes linear dependencies in $\{\tau^\gamma\}_\gamma$, causing us to choose a new, smaller parameterization that functions as a basis for this rowspace.  Table \ref{tab:gmgm-bio-rowspace-basis} gives example parameterizations of the most common scenarios along with some unlikely-but-demonstrative ones.

\begin{sidewaystable}[htbp!]
\centering
\scalebox{1.2}{
\begin{tabular}{ c c c c }
 \hline\centering
 Dataset & $\mathbf{M}_{\boldsymbol{\gamma}}$ & Basis Vectors & Parameterization \\\hline \hline
 $\begin{matrix}
     \\\gamma_1 = (\ell_1, \ell_2)
 \end{matrix}$ & $\begin{bmatrix}
     d_{\ell_2} & d_{\ell_1}
 \end{bmatrix}$ & $\begin{bmatrix}\mathbf{e}_1\end{bmatrix} = \begin{bmatrix}
     d_{\ell_2} & d_{\ell_1}
 \end{bmatrix}$
 &
 $\mathbf{\Omega}^{\gamma_1} = t_1 \mathbf{I}_{d^{\gamma_1}_\forall} + \bigotimes_{\ell\in\gamma_1} \Tilde{\mathbf{\Psi}}_\ell$
 \\  \hline 

$\begin{matrix}
     \\\gamma_1 = (\ell_1, ..., \ell_L)
 \end{matrix}$ & $\begin{bmatrix}
     d_{\backslash\ell_1} & ... & d_{\backslash\ell_L}
 \end{bmatrix}$ & $\begin{bmatrix}\mathbf{e}_1\end{bmatrix} = \begin{bmatrix}
     d_{\backslash\ell_1} & ... & d_{\backslash\ell_L}
 \end{bmatrix}$
 &
 $\mathbf{\Omega}^{\gamma_1} = t_1 \mathbf{I}_{d^{\gamma_1}_\forall} + \bigotimes_{\ell\in\gamma_1} \Tilde{\mathbf{\Psi}}_\ell$
 \\  \hline 
 
 $\left\{\begin{matrix}\\\gamma_i = (\ell_0, \ell_i)\\\hspace{5pt}\end{matrix}\right\}_{1\leq i \leq \Gamma}$ & $\begin{bmatrix}
     d_{\ell_1} & d_{\ell_0} & 0 & ...\\
     d_{\ell_2} & 0 & d_{\ell_0} & ... \\
     d_{\ell_3} & 0 & 0 & ... \\
     \vdots & \vdots & \vdots & \ddots
 \end{bmatrix}$ & $
    \begin{bmatrix}
        \mathbf{e}_1 \\
        \vdots
    \end{bmatrix} = \begin{bmatrix}
     d_{\ell_1} & d_{\ell_0} & 0 & ...\\
     d_{\ell_2} & 0 & d_{\ell_0} & ... \\
     d_{\ell_3} & 0 & 0 & ... \\
     \vdots & \vdots & \vdots & \ddots
 \end{bmatrix}
    $ & $\begin{bmatrix}
        \mathbf{\Omega}^{\gamma_1} \\
        \vdots
    \end{bmatrix} = \begin{bmatrix}t_1 \mathbf{I}_{d^{\gamma_1}_\forall} + \bigotimes_{\ell\in\gamma_1} \Tilde{\mathbf{\Psi}}_\ell \\ \vdots\end{bmatrix}$ \\  \hline 

$\begin{matrix}\\\gamma_1 = (\ell_1, \ell_2) \\ \gamma_2 = (\ell_2, \ell_3) \\ \gamma_3 = (\ell_3, \ell_1)\end{matrix}$ & $\begin{bmatrix}
     d_{\ell_2} & d_{\ell_1} & 0 \\
     0 & d_{\ell_3} & d_{\ell_2} \\
     d_{\ell_3} & 0 & d_{\ell_1}
 \end{bmatrix}$ & $
    \begin{bmatrix}
        \mathbf{e}_1 \\
        \mathbf{e}_2 \\
        \mathbf{e}_3
    \end{bmatrix} = \begin{bmatrix}
       d_{\ell_2} & d_{\ell_1} & 0 \\
       0 & d_{\ell_3} & d_{\ell_2} \\
       d_{\ell_3} & 0 & d_{\ell_1}
    \end{bmatrix}
    $ & $\begin{bmatrix}
        \mathbf{\Omega}^{\gamma_1} \\
        \mathbf{\Omega}^{\gamma_2} \\
        \mathbf{\Omega}^{\gamma_3}
    \end{bmatrix} = \begin{bmatrix}t_1 \mathbf{I}_{d^{\gamma_1}_\forall} + \bigotimes_{\ell\in\gamma_1} \Tilde{\mathbf{\Psi}}_\ell \\
    t_2 \mathbf{I}_{d^{\gamma_2}_\forall} + \bigotimes_{\ell\in\gamma_2} \Tilde{\mathbf{\Psi}}_\ell
    \\
    t_3 \mathbf{I}_{d^{\gamma_3}_\forall} + \bigotimes_{\ell\in\gamma_3} \Tilde{\mathbf{\Psi}}_\ell
    \end{bmatrix}$ \\  \hline

$\begin{matrix}\\\gamma_1 = (\ell_1, \ell_2) \\ \gamma_2 = (\ell_1, \ell_2, \ell_3) \\ \gamma_3 = (\ell_2, \ell_3, \ell_4) \\\hspace{5pt}\end{matrix}$ & $\begin{bmatrix}
     d_{\ell_2} & d_{\ell_1} & 0 & 0\\
     d_{\ell_2}d_{\ell_3} & d_{\ell_1}d_{\ell_3} & d_{\ell_1}d_{\ell_2} & 0 \\
     0 & d_{\ell_3}d_{\ell_4} & d_{\ell_2}d_{\ell_4} & d_{\ell_2}d_{\ell_3}
 \end{bmatrix}$ & $
    \begin{bmatrix}
        \mathbf{e}_1 \\
        \mathbf{e}_2 \\
        \mathbf{e}_3
    \end{bmatrix} = \begin{bmatrix}
     d_{\ell_2} & d_{\ell_1} & 0 & 0\\
     d_{\ell_2}d_{\ell_3} & d_{\ell_1}d_{\ell_3} & d_{\ell_1}d_{\ell_2} & 0 \\
     0 & d_{\ell_3}d_{\ell_4} & d_{\ell_2}d_{\ell_4} & d_{\ell_2}d_{\ell_3}
 \end{bmatrix}
    $ & $\begin{bmatrix}
        \mathbf{\Omega}^{\gamma_1} \\
        \mathbf{\Omega}^{\gamma_2} \\
        \mathbf{\Omega}^{\gamma_3}
    \end{bmatrix} = \begin{bmatrix}t_1 \mathbf{I}_{d^{\gamma_1}_\forall} + \bigotimes_{\ell\in\gamma_1} \Tilde{\mathbf{\Psi}}_\ell \\
    t_2 \mathbf{I}_{d^{\gamma_2}_\forall} + \bigotimes_{\ell\in\gamma_2} \Tilde{\mathbf{\Psi}}_\ell
    \\
    t_3 \mathbf{I}_{d^{\gamma_3}_\forall} + \bigotimes_{\ell\in\gamma_3} \Tilde{\mathbf{\Psi}}_\ell
    \end{bmatrix}$ \\  \hline

    $\begin{matrix}\\\gamma_1 = (\ell_1, \ell_2) \\ \gamma_2 = (\ell_2, \ell_1)\end{matrix}$ & $\begin{bmatrix}
     d_{\ell_2} & d_{\ell_1} \\
     d_{\ell_2} & d_{\ell_1}
 \end{bmatrix}$ & $
    \begin{bmatrix}
        \mathbf{e}_1
    \end{bmatrix} = \begin{bmatrix}
       d_{\ell_2} & d_{\ell_1} \\
    \end{bmatrix}
    $ & \cellcolor{gray!25}$\begin{bmatrix}
        \mathbf{\Omega}^{\gamma_1} \\
        \mathbf{\Omega}^{\gamma_2}
    \end{bmatrix} = \begin{bmatrix}t_1 \mathbf{I}_{d^{\gamma_1}_\forall} + \bigotimes_{\ell\in\gamma_1} \Tilde{\mathbf{\Psi}}_\ell \\
    t_1 \mathbf{I}_{d^{\gamma_2}_\forall} + \bigotimes_{\ell\in\gamma_2} \Tilde{\mathbf{\Psi}}_\ell\end{bmatrix}$ \\  \hline

    $\begin{matrix}\\\gamma_1 = (\ell_1, \ell_2) \\ \gamma_2 = (\ell_3, \ell_4) \\ \gamma_3 = (\ell_1, \ell_2, \ell_3, \ell_4)\\\hspace{5pt}\end{matrix}$ & $\begin{bmatrix}
     d_{\ell_2} & d_{\ell_1} & 0 & 0 \\
     0 & 0 & d_{\ell_4} & d_{\ell_3} \\
     d_{\backslash\ell_1} & d_{\backslash\ell_2} & d_{\backslash\ell_3} & d_{\backslash\ell_4} &
 \end{bmatrix}$ & $
    \begin{bmatrix}
        \mathbf{e}_1 \\
        \mathbf{e}_2 \\
    \end{bmatrix} = \begin{bmatrix}
     d_{\ell_2} & d_{\ell_1} & 0 & 0 \\
     0 & 0 & d_{\ell_4} & d_{\ell_3}
 \end{bmatrix}
    $ & $
    \begin{bmatrix}t_1 \mathbf{I}_{d^{\gamma_1}_\forall} + \bigotimes_{\ell\in\gamma_1} \Tilde{\mathbf{\Psi}}_\ell \\
    t_2 \mathbf{I}_{d^{\gamma_2}_\forall} + \bigotimes_{\ell\in\gamma_2} \Tilde{\mathbf{\Psi}}_\ell \\
    (d_{\ell_3}d_{\ell_4}t_1 + d_{\ell_1}d_{\ell_2}t_2) \mathbf{I}_{d^{\gamma_2}_\forall} + \bigotimes_{\ell\in\gamma_2} \Tilde{\mathbf{\Psi}}_\ell
    \end{bmatrix}$ \cellcolor{gray!25}\\
    \hline 
\end{tabular}
}
\caption{Identifiable parameterizations for several datasets.  Here, $t_i$ is the coefficient corresponding to basis vector $\mathbf{e}_i$, i.e, $t_i$ is one of the parameters, along with $\Tilde{\mathbf{\Psi}}_\ell$.  Often, $t_i = \tau^{\gamma_i}$.  Note that many different bases are possible; our selection here is somewhat arbitrary.  As can be seen, typically $\mathbf{M}_\gamma$ itself can be used for the basis vectors, but sometimes there can be linear dependencies in the data.  The first three rows in this matrix correspond to the cases most likely to arise in practice; a matrix-variate dataset, a tensor-variate dataset, and a multi-modality matrix-variate dataset with a consistently shared axis.  Cells with a grey background indicate a smaller parameterization than that of \citeauthor{greenewald_tensor_2019}.}
\label{tab:gmgm-bio-rowspace-basis}
\end{sidewaystable}

In practice, the choice of parameterization is not too important; it does not have an effect on how we go about finding solutions, nor does it have an effect on the conditional dependency graph.  However, it will be useful for hypothesis testing, which we will discuss in the next section.

\newpage
\section{The algorithm}
\label{sec:gmgm-bio-algorithm-statement}

\begin{algorithm}[th!]
\begin{algorithmic}[1]
\Require Dataset $\{\mathcal{D}^\gamma\}$, minimum eigenvalues $\{\epsilon_\ell\}$, sparsity $\{n_\ell\}$, principal components $\{r_\ell\}$
\Ensure $\{\mathbf{\Psi}_\ell\}$
\For $1 \leq \ell \leq K$
\State $\mathbf{D}_\ell \gets \begin{bmatrix}\mathrm{mat}_{\ell}\left[\mathcal{D}^{\gamma_1}\right] & \cdots & \mathrm{mat}_{\ell}\left[\mathcal{D}^{\gamma_\Gamma}\right] \end{bmatrix}$
\State $\mathbf{V}_\ell^{(r_{\ell})} \gets \text{top }r_\ell\text{ left singular vectors of }\mathbf{D}_\ell$ \Comment{Theorem \ref{thm:gmgm-bio-partial-eigendecomposition}}
\State $\mathbf{E}_\ell^{(r_{\ell})} \gets \left(\text{top }r_\ell\text{ singular values of }\mathbf{D}_\ell \right)^2$
\EndFor
\State $\mathbf{\Lambda}_\ell^{0, (r_{\ell})} \gets \begin{bmatrix}
    \frac{1}{E_{\ell_{11}}^{(r_{\ell})}} & ... & \frac{1}{E_{\ell_{r_\ell r_\ell}}^{(r_{\ell})}}
\end{bmatrix}^T$
\State $\mu \gets 1$
\While \text{not converged}
    \While $\exists\gamma \sum_{\ell\in\gamma}\mathrm{min}\mathbf{\Lambda}^{t+1,(r_{\ell})}_\ell < \epsilon_\ell$ \Comment{Ensure smallest eigenvalue is at least $\epsilon_\ell$}
    \State Decrease $\mu$ if needed
    \For $1 \leq \ell \leq K$ 
    \Comment{Gradient descent on eigenvalues}
        \State $\scriptstyle \mathbf{\Lambda}^{t+1,(r_{\ell})}_{\ell} \gets \mathbf{\Lambda}^{t,(r_{\ell})}_{\ell} - \frac{\mu}{2} \left(\mathbf{E}_\ell^{(r_\ell)} - \sum_{\gamma|\ell\in\gamma}\mathrm{tr}^{k^\gamma_{<\ell}}_{k^\gamma_{>\ell}}\left[\left(\bigoplus_{\ell'\in\gamma}\mathbf{\Lambda}_{\ell'}^{t,(r_{\ell'})}\right)^{-1}\right]\right)$
    \EndFor
    \EndWhile
\EndWhile
\For $1 \leq \ell \leq K$
    \State $\mathbf{\Psi}_\ell \gets \mathrm{thresh}_{n_\ell}\left[\mathbf{V}_\ell^{(r_{\ell})} \mathbf{\Lambda}_\ell^{t,(r_{\ell})} \mathbf{V}_\ell^{T, (r_\ell)}\right]$ \Comment{Threshold simultaneously with recomposition}
\EndFor
\end{algorithmic}
\caption{\textbf{The improved GmGM algorithm}}
\label{alg:GmGM}
\end{algorithm}

\section{Computational complexity}
\label{sec:gmgm-bio-computational-complexity}

Calculating the top $r_\ell$ singular vectors for each axis's $d_\ell \times d_{\forall\backslash\ell}$  matrix $\mathbf{D}_\ell$ takes $O(\sum_\ell r_\ell d_{\forall | \ell})$ time and $O(\sum_\ell r_\ell d_\ell)$ space.  Per iteration, the convex optimization step requires computing $\left(\bigoplus_{\ell\in\gamma}\mathbf{\Lambda}_\ell^{(r_\ell)}\right)^{-1}$, which takes $O(\sum_\gamma \prod_{\ell \in \gamma} r_\ell)$ time and space.  The eigen-recomposition step takes $O(\sum_\ell r_\ell d_\ell^2)$ time, with thresholding talking $O(\sum d_\ell^2)$ time - although our implementation when not thresholding by significance takes $O(\sum_\ell d_\ell^2 \log n_\ell)$ to minimize the number of times we pass through the data.
As we threshold and eigen-recompose simultaneously, it requires $O(n_\ell)$ space.

The Nonparanormal Skeptic requires ranking the data, taking $O(\sum_\ell d_{\forall|\ell} \log d_\ell)$ time and $O(sd_\forall)$ space, where $s$ is the sparsity of the input dataset.  This leads to an algorithm with the following complexity:

\begin{align*}
    O\left(\sum_\gamma \prod_{\ell \in \gamma} r_\ell + \sum_\ell \left(r_\ell d_{\forall | \ell} + r_\ell d_\ell^2 \textcolor{gray}{\log n_\ell} + \textcolor{blue}{d_{\forall | \ell} \log d_\ell}\right)\right) \tag{time}\\
    O\left(\sum_\gamma \prod_{\ell \in \gamma} r_\ell + \sum_\ell \left(r_\ell d_\ell + n_\ell\right) + \textcolor{blue}{sd_\forall}\right) \tag{space}
\end{align*}

Blue terms arise only if using the Nonparanormal Skeptic, and the grey term is a product of our implementation; it could be dropped, but would in practice make the algorithm slower.  As this formula is complicated, we will give an example of a special case.  In the case of matrix-variate data with $n_\ell \in O(d_\ell)$, and constant $r_\ell$ (the case for many real-world datasets), as well as for simplicity assuming $d = d_{\ell_1} \approx d_{\ell_2}$ and thresholding by statistical significance, these become:

\begin{align*}
    O\left(d^2 + \textcolor{blue}{d^2 \log d}\right) \tag{time}\\
    O\left(d + \textcolor{blue}{sd^2}\right) \tag{space}
\end{align*}

In such a scenario, the optimal runtime would be $\Omega\left(sd^2\right)$ as it must visit every element at least once.  Use of the Nonparanormal Skeptic increases the asymptotics, but as long as the dataset is sparse such an increase is modest.

\section{Kronecker product and stridewise-blockwise trace lemmas}
\label{sec:gmgm-bio-kp-lemmas}

This section contains the derivation of a few lemmas pertaining to Kronecker products.  The stridewise-blockwise trace often appears in derivatives involving Kronecker products, and thus is integral to our method as well.  The first two are known from prior work.

\begin{lemma*}[\citet{dahl_network_2013} Lemma 1; Extraction Property]
    \begin{align*}
        \mathrm{tr}\left[\left(\mathbf{I}_{a\times a} \otimes \mathbf{X} \otimes \mathbf{I}_{b\times b} \right)\mathbf{M}\right] &= \mathrm{tr}\left[\mathbf{X}\mathrm{tr}^a_b\left[\mathbf{M}\right]\right]
    \end{align*}
\end{lemma*}

\begin{lemma*}[Cyclic Property \citep{andrew_gmgm_2024}]
    \begin{align*}
        \mathrm{tr}^a_b\left[\left(\mathbf{A}_{a\times a} \otimes I \otimes \mathbf{B}_{b\times b}\right)\mathbf{M}\right] &= \mathrm{tr}^a_b\left[\mathbf{M}\left(\mathbf{A}_{a\times a} \otimes I \otimes \mathbf{B}_{b\times b}\right)\right]
    \end{align*}
\end{lemma*}

The next lemma is a generalized version of lemmas from prior work; an `extraction' property (a generalized Lemma 2 from \citet{dahl_network_2013}).

\begin{lemma}[Extraction Property]
    \label{lem:extraction}
    \begin{align*}
        \mathrm{tr}^a_b\left[\left(\mathbf{I}_a \otimes \mathbf{X} \otimes \mathbf{I}_b\right)\mathbf{M}\left(\mathbf{I}_a \otimes \mathbf{Y}^T \otimes \mathbf{I}_b\right)\right] &= \mathbf{X} \mathrm{tr}^a_b\left[\mathbf{M}\right]\mathbf{Y}^T
    \end{align*}
\end{lemma}
\begin{proof}
    This proof follows the original by \citet{dahl_network_2013} closely; the only differences are that it is proven in the tensor-variate case and that $\mathbf{X}$ and $\mathbf{Y}$ are not constrained to be the same.  The differences do not significantly affect the proof.

    \begin{align*}
        \mathrm{tr}^a_b\left[\left(\mathbf{I}_a \otimes \mathbf{X} \otimes \mathbf{I}_b\right)\mathbf{M}\left(\mathbf{I}_a \otimes \mathbf{Y}^T \otimes \mathbf{I}_b\right)\right] &= \mathrm{tr}\left[\left(\mathbf{I}_a \otimes \mathbf{X} \otimes \mathbf{I}_b\right)\mathbf{M}\left(\mathbf{I}_a \otimes \mathbf{Y}^T \otimes \mathbf{I}_b\right)\left(\mathbf{I}_a \otimes \mathbf{J}^{ij} \otimes \mathbf{I}_b\right)\right]_{ij} \\
        &= \mathrm{tr}\left[\left(\mathbf{I}_a \otimes \mathbf{j}^{j}\mathbf{X} \otimes \mathbf{I}_b\right)\mathbf{M}\left(\mathbf{I}_a \otimes \mathbf{Y}^T\mathbf{j}^{i, T} \otimes \mathbf{I}_b\right)\right]_{ij} \\
        &= \mathrm{tr}\left[\left(\mathbf{I}_a \otimes \mathbf{X}_j \otimes \mathbf{I}_b\right)\mathbf{M}\left(\mathbf{I}_a \otimes \mathbf{Y}_i^T \otimes \mathbf{I}_b\right)\right]_{ij} \\
        &= \mathrm{tr}\left[\left(\mathbf{I}_a \otimes \otimes \mathbf{Y}_i^T\mathbf{X}_j \otimes \mathbf{I}_b\right)\mathbf{M}\right]_{ij} \\
        &= \mathrm{tr}\left[\mathbf{Y}_i^T\mathbf{X}_j\mathrm{tr}^a_b\left[\mathbf{M}\right]\right]_{ij} \\
        &= \mathrm{tr}\left[\mathbf{X}_j\mathrm{tr}^a_b\left[\mathbf{M}\right]\mathbf{Y}_i^T\right]_{ij}
    \end{align*}

    Note now that the value inside the trace is a scalar value, and hence we can drop the outer trace.

    \begin{align*}
        \mathrm{tr}^a_b\left[\left(\mathbf{I}_a \otimes \mathbf{X} \otimes \mathbf{I}_b\right)\mathbf{M}\left(\mathbf{I}_a \otimes \mathbf{Y}^T \otimes \mathbf{I}_b\right)\right] 
        &= \left[\mathbf{X}_j\mathrm{tr}^a_b\left[\mathbf{M}\right]\mathbf{Y}_i^T\right]_{ij}
    \end{align*}

    Which leads us to the final result.

    \begin{align*}
        \mathrm{tr}^a_b\left[\left(\mathbf{I}_a \otimes \mathbf{X} \otimes \mathbf{I}_b\right)\mathbf{M}\left(\mathbf{I}_a \otimes \mathbf{Y}^T \otimes \mathbf{I}_b\right)\right] 
        &= \mathbf{X}\mathrm{tr}^a_b\left[\mathbf{M}\right]\mathbf{Y}^T
    \end{align*}
    
\end{proof}

Finally, we need a new property of the stridewise-blockwise trace, the `Downsampling Property'.  This is essential to be able to work with rectangular matrices in the stridewise-blockwise trace, such as those arising from partial eigendecomposition.

\begin{lemma}[Downsampling Property]
    \label{lem:downsampling}
    Suppose $\mathbf{U}_{a \times x}$ and $\mathbf{V}_{a \times x}$ are matrices such that $\mathbf{V}^T\mathbf{U} = \mathbf{I}_{x \times x}$, and $\mathbf{W}_{b \times y}$ and $\mathbf{X}_{b \times y}$ are matrices such that $\mathbf{X}^T\mathbf{W} = \mathbf{I}_{y \times y}$.  Then we have the following.

    \begin{align*}
        \mathrm{tr}^a_b\left[\left(\mathbf{U}\otimes\mathbf{I}\otimes\mathbf{W}\right)\mathbf{M}\left(\mathbf{V}\otimes\mathbf{I}\otimes\mathbf{X}\right)^T\right] &= \mathrm{tr}^x_y\left[\mathbf{M}\right]
    \end{align*}
\end{lemma}
\begin{proof}
    \begin{align*}
        \mathrm{tr}^a_b\left[\left(\mathbf{U}\otimes\mathbf{I}\otimes\mathbf{W}\right)\mathbf{M}\left(\mathbf{V}\otimes\mathbf{I}\otimes\mathbf{X}\right)^T\right] &= \mathrm{tr}\left[\left(\mathbf{U}\otimes\mathbf{I}\otimes\mathbf{W}\right)\mathbf{M}\left(\mathbf{V}\otimes\mathbf{I}\otimes\mathbf{X}\right)^T\left(\mathbf{I}\otimes\mathbf{J}^{ij}\otimes\mathbf{I}\right)\right]_{ij} \\
        &= \mathrm{tr}\left[\mathbf{M}\left(\mathbf{V}^T\mathbf{U}\otimes\mathbf{J}^{ij}\otimes\mathbf{X}^T\mathbf{W}\right)\right]_{ij} \\
        &= \mathrm{tr}\left[\mathbf{M}\left(\mathbf{I}_{x\times x}\otimes\mathbf{J}^{ij}\otimes\mathbf{I}_{y\times y}\right)\right]_{ij} \\
        &= \mathrm{tr}^x_y\left[\mathbf{M}\right]
    \end{align*}
\end{proof}

\section{Finding the maximum likelihood estimator}
\label{sec:gmgm-bio-max-likelihood-estimator}

In this section, we will derive gradients for gradient descent and prove Theorem \ref{thm:gmgm-bio-partial-eigendecomposition}.  For convenience, we repeat the definition of the $\mathrm{NLL}$.

\begin{align*}
    \mathrm{NLL} &= \frac{-1}{2}\sum_\gamma \left(\log\mathrm{det}\left[\mathbf{t}^T\mathbf{a}^\gamma\mathbf{I}_{d_\forall^\gamma} + \bigoplus_{\ell\in\gamma}\mathbf{\Psi}_\ell\right] - \sum_{\ell\in\gamma} \mathrm{tr}\left[\mathbf{S}_\ell^\gamma\mathbf{\Psi}_\ell\right] - \sum_{\ell\in\gamma} \frac{\mathbf{t}^T\mathbf{a}^\gamma}{L^\gamma}\mathrm{tr}\left[\mathbf{S}^\gamma_\ell\right]\right)
\end{align*}

To derive the MLE eigenvectors, we proceed in a manner similar to \citet{andrew_gmgm_2024}.  A key difference is that we must now differentiate a log-pseudodeterminant, although it behaves analagously to the derivative of a full determinant.  This is given by Theorem 2.15 of \citet{holbrook_differentiating_2018}:

\begin{align*}
    \frac{\partial}{\partial \mathbf{A}} \mathrm{logdet}^\dagger \left(\mathbf{A}\right) &= \mathrm{tr} \left[\mathbf{A}^\dagger\right]
\end{align*}

\begin{lemma}
\label{lem:S_ell}
At the maximum likelihood of the singular Kronecker-sum-structured normal distribution, we have that $\mathbf{S}_\ell = \sum_{\gamma | \ell \in \gamma} \mathrm{tr}^{d^\gamma_{<\ell}}_{d^\gamma_{>\ell}}\left[\left(\bigoplus_{\ell'\in\gamma}\mathbf{\Psi}_{\ell'}\right)^\dagger\right]$.
\end{lemma}
\begin{proof}
    Note that $\mathbf{\Psi}_\ell$ is symmetric, and hence $\frac{\partial f\left(\mathbf{\Psi}\right)}{\partial \mathbf{\Psi}} = \mathrm{sym}\left[\mathbf{X}\right] = \frac{\mathbf{X} + \mathbf{X}^T}{2}$, where $\mathbf{X}$ is the derivative found without taking into account symmetry \citep{srinivasan_what_2023}.  It does not affect the location of the MLE (although we include it in the derivations for completeness), but it does affect the value of the gradient, and thus will be important to account for when the gradient itself is of interest.

    \begin{align*}
        \frac{\partial}{\partial \mathbf{\Psi}_\ell} \mathrm{NLL} &=  \mathrm{sym}\left[\frac{\partial}{\partial \mathbf{\Psi}_\ell}\sum_{\ell'}\frac{1}{2}\mathrm{tr}\left[\mathbf{S}_{\ell'}\mathbf{\Psi}_{\ell'}\right] - \sum_\gamma \left( \frac{\partial}{\partial \mathbf{\Psi}_{\ell_{ij}}}\frac{1}{2}\mathrm{logdet}^\dagger\left[\bigoplus_{\ell'\in\gamma}\mathbf{\Psi}_{\ell'}\right]_{ij}\right)\right] \\
        &= \mathrm{sym}\left[\sum_\gamma \frac{1}{2} \mathbf{S}_\ell^\gamma - \frac{1}{2} \sum_{\gamma | \ell \in \gamma} \mathrm{tr} \left[\left(\bigoplus_{\ell'\in\gamma}\mathbf{\Psi}_{\ell'}\right)^\dagger \frac{\partial}{\partial\mathbf{\Psi}_{\ell_{ij}}} \bigoplus_{\ell' \in \gamma} \mathbf{\Psi}_{\ell'}\right]_{ij}\right]
    \end{align*}

    It is also not too hard to show that $\frac{\partial}{\partial\mathbf{\Psi}_{\ell_{ij}}} \bigoplus_{\ell' \in \gamma} \mathbf{\Psi}_{\ell'} = \mathbf{I}_{d^\gamma_{<\ell}} \otimes \mathbf{J}^{ij}_{d_\ell \times d_\ell} \otimes \mathbf{I}_{d^\gamma_{>\ell}}$ (not accounting for symmetry).

    \begin{align*}
        \frac{\partial}{\partial\mathbf{\Psi}_{\ell_{ij}}} \bigoplus_{\ell' \in \gamma} \mathbf{\Psi}_{\ell'} &= \frac{\partial}{\partial\mathbf{\Psi}_{\ell_{ij}}} \sum_{\ell' \in \gamma} \mathbf{I}_{d^\gamma_{<\ell'}} \otimes \mathbf{\Psi}_{\ell'} \otimes \mathbf{I}_{d^\gamma_{>\ell'}} \\
        &= \mathbf{I}_{d^\gamma_{<\ell}} \otimes \frac{\partial}{\partial\mathbf{\Psi}_{\ell_{ij}}} \mathbf{\Psi}_{\ell} \otimes \mathbf{I}_{d^\gamma_{>\ell}} \\
        &= \mathbf{I}_{d^\gamma_{<\ell}} \otimes \mathbf{J}^{ij}_{d_\ell \times d_\ell} \otimes \mathbf{I}_{d^\gamma_{>\ell}}
    \end{align*}

    By the definition of the stridewise-blockwise trace, we have that:

    \begin{align*}
        \frac{\partial}{\partial \mathbf{\Psi}_\ell} \mathrm{NLL} &= \mathrm{sym}\left[\sum_\gamma \frac{1}{2} \mathbf{S}_\ell^\gamma - \frac{1}{2} \sum_{\gamma | \ell \in \gamma} \mathrm{tr}^{d^\gamma_{<\ell}}_{d^\gamma_{>\ell}} \left[\left(\bigoplus_{\ell'\in\gamma}\mathbf{\Psi}_{\ell'}\right)^\dagger\right]\right]
    \end{align*}

    Now, note that the maximum likelihood corresponds to the point where $\frac{\partial}{\partial \mathbf{\Psi}_\ell} \mathrm{NLL}$ is $0$.

    \begin{align*}
        \mathbf{S}_\ell = \sum_{\gamma | \ell \in \gamma} \mathrm{tr}^{d^\gamma_{<\ell}}_{d^\gamma_{>\ell}}\left[\left(\bigoplus_{\ell'\in\gamma}\mathbf{\Psi}_{\ell'}\right)^\dagger\right]
    \end{align*}
    
\end{proof}

\begin{theorem*}[Theorem \ref{thm:gmgm-bio-partial-eigendecomposition}]
    Let the rank-$r_\ell$ partial eigendecomposition of $\mathbf{S}_\ell$ be $\mathbf{V}_\ell^{(k)} \mathbf{E}_\ell^{(r_\ell)} \mathbf{V}_\ell^{(r_\ell), T}$, with $r_\ell \geq \mathrm{max}\left(\mathrm{rank}\left[\mathbf{S}_\ell\right], \mathrm{rank}\left[\mathbf{\Psi}_\ell\right]\right)$.  Then $\mathbf{V}_\ell^{(r_\ell)}$ are also eigenvectors of $\mathbf{\Psi}_\ell$.
\end{theorem*}
\begin{proof}
    Suppose we had a rank-$r_\ell$ partial eigendecomposition yielding a $d_\ell \times r_\ell$ matrix of eigenvectors $\mathbf{V}^{(r_\ell)}_\ell$; in this case, $\mathbf{V}^{(r_\ell), T}_\ell\mathbf{V}^{(r_\ell)}_\ell$ is still the identity matrix (albeit of a smaller dimension); this will allow us to use Lemma \ref{lem:downsampling} to `downsample' the stridewise-blockwise trace.

    We will use $\mathbf{V}^\gamma_{<\ell}$ as a shorthand for $\bigotimes_{\ell'\in\gamma | \ell' <^\gamma \ell} \mathbf{V}^{(r_\ell')}_{\ell'}$ (with $\mathbf{V}^\gamma_{>\ell}$ defined analogously).

    With this in mind, note that by Lemma \ref{lem:S_ell} we have that:

    \begin{align*}
        \mathbf{S}_\ell &= \sum_{\gamma | \ell \in \gamma} \mathrm{tr}^{d^\gamma_{<\ell}}_{d^\gamma_{>\ell}}\left[\left(\bigoplus_{\ell'\in\gamma}\mathbf{\Psi}_{\ell'}\right)^\dagger\right] \\
        &= \sum_{\gamma | \ell \in \gamma} \mathrm{tr}^{d_{<\ell}}_{d^\gamma_{>\ell}}\left[\left(\mathbf{V}^\gamma_{<\ell} \otimes \mathbf{V}_\ell^{(r_\ell)} \otimes \mathbf{V}^\gamma_{>\ell}\right)\left(\bigoplus_{\ell'\in\gamma}\mathbf{\Lambda}_{\ell'}^{(r_{\ell'})}\right)^\dagger \left(\mathbf{V}^\gamma_{<\ell} \otimes \mathbf{V}_\ell^{(r_\ell)} \otimes \mathbf{V}^\gamma_{>\ell}\right)^T\right] \\
        &= \mathbf{V}_\ell^{(r_\ell)} \sum_{\gamma | \ell \in \gamma}\mathrm{tr}^{k^\gamma_{<\ell}}_{k^\gamma_{>\ell}}\left[\left(\bigoplus_{\ell'\in\gamma}\mathbf{\Lambda}_{\ell'}^{(r_{\ell'})}\right)^\dagger\right] \mathbf{V}_\ell^{(r_\ell), T} \tag{Extraction and Downsampling Properties.}
    \end{align*}

    The $\mathbf{V}_{<\ell}^\gamma$ and $\mathbf{V}^\gamma_{>\ell}$ matrices depend on $\gamma$ as the precise tensor they are in affects which $\ell$s are present and in what order.  Thankfully, they disappear due to the downsampling property of sb-trace, allowing us to express our equation as the product of orthogonal matrices $\mathbf{V}_\ell^{(r_\ell)}, \mathbf{V}^{(r_\ell), T}_\ell$ and a diagonal matrix $\sum_{\gamma|\ell\in\gamma}\mathrm{tr}^{k^\gamma_{<\ell}}_{k^\gamma_{>\ell}}\left[\left(\bigoplus_{\ell'\in\gamma}\mathbf{\Lambda}_{\ell'}^{(r_{\ell'})}\right)^\dagger \right]$.  Thus, this is clearly a partial eigendecomposition of $\mathbf{S}_\ell$, completing the proof.
\end{proof}

Theorem \ref{thm:gmgm-bio-partial-eigendecomposition} establishes that we can solve for the eigenvectors of each $\mathbf{\Psi}_\ell$ analytically.  No such expression for the eigenvalues exist, but it is not too hard to find them with gradient descent.  The proportion of time this estimation takes is negligible, due to cost of partial eigendecomposition and the comparatively small dimension of the eigenvalues ($O(\sum_\ell d_\ell)$ parameters rather than $O(\sum_\ell d_\ell^2)$).  We derive the gradient below:

\begin{prop*}
    \label{prop:evals}
    \sloppy The gradient of the $\mathrm{NLL}$ with respect to the eigenvalues is given by $\frac{1}{2} \left(\mathbf{E}_\ell^{(r_\ell)} - \sum_{\gamma|\ell\in\gamma}\mathrm{tr}^{k^\gamma_{<\ell}}_{k^\gamma_{>\ell}}\left[\left(\bigoplus_{\ell'\in\gamma}\mathbf{\Lambda}_{\ell'}^{(r_{\ell'})}\right)^\dagger\right]\right)$.
\end{prop*}
\begin{proof}
    As mentioned earlier, we will use the functional invariance of the MLE to change the problem from finding when $\frac{\partial}{\partial\mathbf{\Psi}_\ell}\mathrm{NLL} = 0$ to one of finding when:

    \begin{align*}
        \frac{\partial}{\partial\mathbf{V}_\ell^{(r_{\ell})}}\mathrm{NLL} &= 0 \\
        \frac{\partial}{\partial\mathbf{\Lambda}_\ell^{(r_{\ell})}}\mathrm{NLL} &= 0
    \end{align*}

    We already know the conditions on $\mathbf{V}_\ell^{(r_{\ell})}$ at the MLE, thanks to Theorem \ref{thm:gmgm-bio-partial-eigendecomposition}, so we will focus on $\mathbf{\Lambda}_\ell^{(r_{\ell})}$.  We follow a similar derivation to Lemma \ref{lem:S_ell}.

    \begin{align*}
        \frac{\partial}{\partial \mathbf{\Lambda}_\ell} \mathrm{NLL} &=  \frac{\partial}{\partial \mathbf{\Lambda}_\ell^{(r_{\ell})}}\sum_{\ell'}\frac{1}{2}\mathrm{tr}\left[\mathbf{S}_{\ell'}\mathbf{\Psi}_{\ell'}\right] - \sum_\gamma \left( \frac{\partial}{\partial \mathbf{\Lambda}_{\ell_{ii}}^{(r_{\ell})}}\frac{1}{2}\mathrm{logdet}^\dagger\left[\bigoplus_{\ell'\in\gamma}\mathbf{\Psi}_{\ell'}\right]_{ii}\right) \\
        &= \frac{\partial}{\partial \mathbf{\Lambda}_\ell^{(r_{\ell})}}\sum_{\ell'}\frac{1}{2}\mathrm{tr}\left[\mathbf{E}_{\ell'}^{(r_{\ell'})}\mathbf{\Lambda}_{\ell'}^{(r_{\ell'})}\right] - \sum_\gamma \left( \frac{\partial}{\partial \mathbf{\Lambda}_{\ell_{ii}}^{(r_{\ell})}}\frac{1}{2}\mathrm{logdet}^\dagger\left[\bigoplus_{\ell'\in\gamma}\mathbf{\Lambda}_{\ell'}^{(r_{\ell'})}\right]_{ii}\right) \\
        &= \sum_\gamma \frac{1}{2} \mathbf{E}_\ell^{(r_{\ell}), \gamma} - \frac{1}{2} \sum_{\gamma | \ell \in \gamma} \mathrm{tr} \left[\left(\bigoplus_{\ell'\in\gamma}\mathbf{\Lambda}_{\ell'}^{(r_{\ell'})}\right)^\dagger \frac{\partial}{\partial\mathbf{\Lambda}_{\ell_{ii}}^{(r_{\ell})}} \bigoplus_{\ell' \in \gamma} \mathbf{\Lambda}_{\ell'}^{(r_{\ell'})}\right]_{ii} \\
        &= \frac{1}{2} \left(\mathbf{E}_\ell^{(r_{\ell})} - \sum_{\gamma|\ell\in\gamma}\mathrm{tr}^{k^\gamma_{<\ell}}_{k^\gamma_{>\ell}}\left[\left(\bigoplus_{\ell'\in\gamma}\mathbf{\Lambda}_{\ell'}^{(r_{\ell'})}\right)^\dagger\right]\right)
    \end{align*}

    In the derivation, we made use of the fact that $\mathbf{S_\ell}$ and $\mathbf{\Psi}_\ell$ had the same eigenvectors at the MLE, the cyclic property of the trace, and the fact that pseudodeterminants only depend on the nonzero eigenvalues.
\end{proof}

\section{Existence and uniqueness}
\label{sec:gmgm-bio-existence-and-uniqueness}

It is easy to see that our objective is convex; unstructured Gaussian precision matrix estimation is strictly convex, and the space of Kronecker-sum-decomposable matrices is also convex.  In this section, we will further prove the existence and uniqueness of solutions (Theorem \ref{thm:gmgm-bio-well-foundedness}).

\begin{lemma}[Existence]
    \label{lem:existence}
    There exists a minimizer of the negative log likelihood in the domain $\boldsymbol{\Lambda}_\ell^{(r_{\ell})} \geq \epsilon_\ell$ (Assumption \ref{ass:minimum-eigenvalue}), where $\epsilon_\ell$ is some strictly positive constant (for each $\ell$),  as long as $\mathbf{E}_\ell^{(r_{\ell})} \in \mathbb{R}^{++}_{r_\ell}$ (i.e. $r_\ell \leq \mathrm{rank}\left[\mathbf{S}_\ell\right]$).
\end{lemma}
\begin{proof}
    First, note that $\bigtimes_{\ell}\bigtimes_{1\leq i\leq r_\ell} [\epsilon_\ell, \infty]$ is a compact subset of the extended reals, and thus a minimum to any continuous function must exist on this set.  As it is convex, it will be a global minimum.  However, the minimum could exist at $\infty$, i.e., only in the limit.  To rule out this case, we will establish an upper bound on the solution.

    Recall that any minimum must satisfy $\mathbf{E}_\ell^{(r_{\ell})} = \sum_{\gamma|\ell\in\gamma}\mathrm{tr}^{k^\gamma_{<\ell}}_{k^\gamma_{>\ell}}\left[\left(\oplus_{\ell'\in\gamma}\mathbf{\Lambda}_{\ell'}^{(r_{\ell'})}\right)^\dagger\right]$ (for each $\ell$).  It will be helpful to see what this formula equates to elementwise:

    \begin{align*}
        \mathbf{E}_{\ell_{ii}}^{(r_{\ell})} &= \sum_{\gamma|\ell\in\gamma}\mathrm{tr}^{d^\gamma_{<\ell}}_{d^\gamma_{>\ell}}\left[\left(\oplus_{\ell'\in\gamma}\mathbf{\Lambda}_{\ell'}^{(r_{\ell'})}\right)^\dagger\right]_{ii} \\
        &= \sum_{\gamma|\ell\in\gamma}\mathrm{tr}\left[\left(\oplus_{\ell'\in\gamma}\mathbf{\Lambda}_{\ell'}^{(r_{\ell'})}\right)^\dagger\left(\mathbf{I}_{r_{<\ell}^\gamma}\otimes\mathbf{J}^{ii}\otimes\mathbf{I}_{r_{>\ell}^\gamma}\right)\right] \\
        &= \sum_{\gamma|\ell\in\gamma}\sum_{\text{  }\mathbf{j} \text{ indexing all non-$\ell$ axes}} \frac{1}{\lambda_{\ell_i} + \sum_{\ell'} \lambda_{\ell'_{j_{\ell'}}}}
    \end{align*}

    Thus, every element of $\mathbf{E}_\ell^{(r_{\ell})}$ has a corresponding element of $\mathbf{\Lambda}_\ell^{(r_{\ell})}$ that appears in all summands composing it, and vice versa.  Suppose $\lambda_{\ell_i} = \frac{\sum_\gamma d_{\backslash\ell}^\gamma}{\mathbf{E}_{\ell_{ii}}}$ (for a given $\ell$, and every $i$).  Note that if for all other $\ell'$, $\lambda_{\ell'_{j_{\ell'}}}$ were 0, then the sum of these $\lambda_\ell$ terms would sum to exactly $\mathbf{E}_{\ell_{ii}}$.  However, as $\boldsymbol{\Lambda}_\ell^{(r_{\ell})} \geq \epsilon_\ell$, each reciprocal is smaller than the case where they are 0; thus, they cannot add up to $\mathbf{E}_{\ell_{ii}}$.  
    
    As $\lambda_{\ell_i}$ increases above this bound, the problem only becomes more severe.  This is easy to see by checking the gradient, which indicates that the negative log-likelihood increases as one increases $\lambda_{\ell_i}$.

    \begin{align*}
        \frac{\partial}{\partial\mathbf{\Lambda}_{\ell_{ii}}^{(r_\ell)}} \mathrm{NLL} &= \frac{1}{2}\left(\mathbf{E}_{\ell_{ii}} -  \sum_{\gamma|\ell\in\gamma}\sum_{\text{  }\mathbf{j} \text{ indexing all non-$\ell$ axes}} \frac{1}{\lambda_{\ell_i} + \sum_{\ell'} \lambda_{\ell'_{j_{\ell'}}}}\right) \\
        &> \frac{1}{2}\left(\mathbf{E}_{\ell_{ii}} -  \mathbf{E}_{\ell_{ii}}\right) \\
        &= 0
    \end{align*}

    This implies that the minimum must occur between $\frac{\sum_\gamma d_{\backslash\ell}^\gamma}{\mathbf{E}_{\ell_{ii}}} \geq \lambda_{\ell_{ii}} \geq \epsilon_\ell$.  Clearly, this rules out the unsavory case of the minimum occurring in the limit.  This completes the proof.
\end{proof}

We already saw in Section \ref{sec:gmgm-bio-identifiability} that the diagonals of $\mathbf{\Psi}_\ell$ are not identifiable.  However, this non-identifiability is only an effect of our parameterization of the full precision matrices $\{\mathbf{\Omega}^\gamma\}$; it is reasonable to ask if the solution is unique up to this non-identifiability; Lemma \ref{lem:uniqueness-up-to-identifiability} shows that this is true.

\begin{lemma}[Uniqueness up-to Identifiability]
    \label{lem:uniqueness-up-to-identifiability}
    The solution of our optimization problem is unique, except for the non-identifiability in the diagonals. 
\end{lemma}
\begin{proof}
    As mentioned earlier, we can frame our optimization problem as the strictly convex problem of unstructured precision matrix estimation (as in the graphical lasso problem), subject to a (convex) Kronecker-sum-decomposability constraint.  As the problem is strictly convex, there is a unique solution.  Thus, any non-uniqueness must come from our parameterization of Kronecker-sum-decomposable matrices.
\end{proof}

We now tie these two lemmas together to complete the statement of Theorem \ref{thm:gmgm-bio-well-foundedness}.

\begin{theorem*}[Theorem \ref{thm:gmgm-bio-well-foundedness}]
    There exists a unique MLE (up to the non-identifiability of the diagonals of the Kronecker sum) that estimates the precision matrix of the singular Kronecker-sum-structured normal distribution, as long as $r_\ell \leq \mathrm{rank}\left[\mathbf{S}_\ell\right]$, subject to Assumption \ref{ass:minimum-eigenvalue}.
\end{theorem*}
\begin{proof}
    This follows directly from Lemmas \ref{lem:existence} and \ref{lem:uniqueness-up-to-identifiability}.
\end{proof}

\section{Derivation of the Fisher information}
\label{sec:gmgm-bio-fisher-information}

This section contains the derivation of the Fisher Information, which is far too long and tedious to put in the main paper.  For convenience, we repeat the definition of the $\mathrm{NLL}$.

\begin{align*}
    \mathrm{NLL} &= \frac{-1}{2}\sum_\gamma \left(\log\mathrm{det}\left[\mathbf{t}^T\mathbf{a}^\gamma\mathbf{I}_{d_\forall^\gamma} + \bigoplus_{\ell\in\gamma}\mathbf{\Psi}_\ell\right] - \sum_{\ell\in\gamma} \mathrm{tr}\left[\mathbf{S}_\ell^\gamma\mathbf{\Psi}_\ell\right] - \sum_{\ell\in\gamma} \frac{\mathbf{t}^T\mathbf{a}^\gamma}{L^\gamma}\mathrm{tr}\left[\mathbf{S}^\gamma_\ell\right]\right)
\end{align*}

Lemmas \ref{lem:first-derivative-psi}, \ref{lem:t-derivatives}, and \ref{lem:double-psi-derivative} contain the calculations that derive the Hessian of the $\mathrm{NLL}$.

\begin{lemma}
    \label{lem:first-derivative-psi}
    $\frac{\partial}{\partial\psi_{\ell_{ij}}} \mathrm{NLL} = -\frac{1}{2}\mathrm{sym}\left[\sum_{\gamma|\ell\in\gamma} \mathrm{tr}^{d^\gamma_{<\ell}}_{d^\gamma_{>\ell}}\left[\left(\mathbf{t}^T\mathbf{a}^\gamma\mathbf{I_{d^\gamma_\forall}} + \bigoplus_{\ell\in\gamma}\mathbf{\Psi}_\ell\right)^{-1}\right]_{ij} - \mathbf{S}_{\ell_{ij}}\right]$
\end{lemma}
\begin{proof}
    This follows from Lemma \ref{lem:S_ell}.
\end{proof}

\begin{lemma}
    \label{lem:t-derivatives}
    \begin{align*}
        \frac{\partial}{\partial t_i} \mathrm{NLL} &= \sum_\gamma -\frac{a^\gamma_i}{2}\mathrm{tr}\left[\left(\mathbf{t}^T\mathbf{a}^\gamma\mathbf{I}_{d^\gamma_\forall} + \bigoplus_{\ell\in\gamma}\mathbf{\Psi}_\ell\right)^{-1}\right] + \frac{a^\gamma_i}{2}\mathrm{tr}\left[\mathbf{S}^\gamma\right] \\
        \frac{\partial}{\partial t_i\partial t_j} \mathrm{NLL} &= \sum_\gamma \frac{a^\gamma_ia^\gamma_j}{2}\mathrm{tr}\left[\left(\mathbf{t}^T\mathbf{a}^\gamma\mathbf{I}_{d^\gamma_\forall} + \bigoplus_{\ell\in\gamma}\mathbf{\Psi}_\ell\right)^{-2}\right]
        \\
        \frac{\partial}{\partial t_i\partial \psi_{\ell_{jk}}} \mathrm{NLL} &= \mathrm{sym}\left[\sum_{\gamma|\ell\in\gamma} \frac{a^\gamma_i}{2}\mathrm{tr}^{d^\gamma_{<\ell}}_{d^\gamma_{>\ell}}\left[\left(\mathbf{t}^T\mathbf{a}^\gamma\mathbf{I}_{d^\gamma_\forall} + \bigoplus_{\ell\in\gamma}\mathbf{\Psi}_\ell\right)^{-2}\right]\right]_{jk}
    \end{align*}
\end{lemma}
\begin{proof}
    First, observe the following:

    \begin{align*}
        \frac{\partial}{\partial t_i} \log\mathrm{det}\left[\mathbf{t}^T\mathbf{a}^\gamma \mathbf{I}_{d^\gamma_\forall} + \bigoplus_{\ell\in\gamma}\mathbf{\Psi}_\ell\right] &= \mathrm{tr}\left[\left(\mathbf{t}^T\mathbf{a}^\gamma \mathbf{I}_{d^\gamma_\forall} + \bigoplus_{\ell\in\gamma}\mathbf{\Psi}_\ell\right)^{-1}\frac{\partial}{\partial t_i} \left(\mathbf{t}^T\mathbf{a}^\gamma \mathbf{I}_{d^\gamma_\forall} + \bigoplus_{\ell\in\gamma}\mathbf{\Psi}_\ell\right)\right] \\
        &= \mathrm{tr}\left[\left(\mathbf{t}^T\mathbf{a}^\gamma \mathbf{I}_{d^\gamma_\forall} + \bigoplus_{\ell\in\gamma}\mathbf{\Psi}_\ell\right)^{-1}a_i^\gamma \mathbf{I}_{d^\gamma_\forall}\right] \\
        &= a_i^\gamma\mathrm{tr}\left[\left(\mathbf{t}^T\mathbf{a}^\gamma \mathbf{I}_{d^\gamma_\forall} + \bigoplus_{\ell\in\gamma}\mathbf{\Psi}_\ell\right)^{-1}\right]
    \end{align*}

    Furthermore, note that:

    \begin{align*}
        \frac{\partial}{\partial t_i}\sum_{\ell\in\gamma} \frac{\mathbf{t}^T\mathbf{a}^\gamma}{L^\gamma}\mathrm{tr}\left[\mathbf{\mathbf{S}^\gamma_\ell}\right] &= \sum_{\ell\in\gamma}\frac{a^\gamma_i}{L^\gamma}\mathrm{tr}\left[\mathbf{S}^\gamma_\ell\right] \\
        &= \frac{a^\gamma_i}{L^\gamma}\mathrm{tr}\left[\mathbf{S}^\gamma\right]
    \end{align*}

    Putting these together yields the first claim:

    \begin{align*}
        \frac{\partial}{\partial t_i} \mathrm{NLL} &= \sum_\gamma -\frac{a^\gamma_i}{2}\mathrm{tr}\left[\left(\mathbf{t}^T\mathbf{a}^\gamma\mathbf{I}_{d^\gamma_\forall} + \bigoplus_{\ell\in\gamma}\mathbf{\Psi}_\ell\right)^{-1}\right] + \frac{a^\gamma_i}{2}\mathrm{tr}\left[\mathbf{S}^\gamma\right]
    \end{align*}

    Differentiating a second time results in the following:

    \begin{align*}
        \frac{\partial}{\partial t_i \partial t_j} \mathrm{NLL} &= \sum_\gamma \frac{a^\gamma_i}{2}\mathrm{tr}\left[\left(\mathbf{t}^T\mathbf{a}^\gamma\mathbf{I}_{d^\gamma_\forall} + \bigoplus_{\ell\in\gamma}\mathbf{\Psi}_\ell\right)^{-1}\left(a^\gamma_i\mathbf{I}_{d^\gamma_\forall}\right)\left(\mathbf{t}^T\mathbf{a}^\gamma\mathbf{I}_{d^\gamma_\forall} + \bigoplus_{\ell\in\gamma}\mathbf{\Psi}_\ell\right)^{-1}\right] \\
        &= \sum_\gamma \frac{a^\gamma_ia^\gamma_j}{2}\mathrm{tr}\left[\left(\mathbf{t}^T\mathbf{a}^\gamma\mathbf{I}_{d^\gamma_\forall} + \bigoplus_{\ell\in\gamma}\mathbf{\Psi}_\ell\right)^{-2}\right]
    \end{align*}

    Likewise, differentiating w.r.t. $\psi_{\ell_{jk}}$ gives us the final part of our result.  We'll delay the symmetrization step until the end to prevent the equation running off the right margin of the page.

    \begin{align*}
        \frac{\partial^\mathrm{nosym}}{\partial t_i \partial \psi_{\ell_{jk}}} \mathrm{NLL} &= \sum_{\gamma|\ell\in\gamma} \frac{a^\gamma_i}{2}\mathrm{tr}\left[\left(\mathbf{t}^T\mathbf{a}^\gamma\mathbf{I}_{d^\gamma_\forall} + \bigoplus_{\ell\in\gamma}\mathbf{\Psi}_\ell\right)^{-1}\left(\mathbf{I}_{d^\gamma_{<\ell}} \otimes \mathbf{J}^{jk} \otimes \mathbf{I}_{d^\gamma_{>\ell}}\right)\left(\mathbf{t}^T\mathbf{a}^\gamma\mathbf{I}_{d^\gamma_\forall} + \bigoplus_{\ell\in\gamma}\mathbf{\Psi}_\ell\right)^{-1}\right] \\
        &= \sum_{\gamma|\ell\in\gamma} \frac{a^\gamma_i}{2}\mathrm{tr}^{d^\gamma_{<\ell}}_{d^\gamma_{>\ell}}\left[\left(\mathbf{t}^T\mathbf{a}^\gamma\mathbf{I}_{d^\gamma_\forall} + \bigoplus_{\ell\in\gamma}\mathbf{\Psi}_\ell\right)^{-2}\right]_{jk} \\
        \frac{\partial}{\partial t_i \partial \psi_{\ell_{jk}}} \mathrm{NLL} &= \mathrm{sym}\left[\sum_{\gamma|\ell\in\gamma} \frac{a^\gamma_i}{2}\mathrm{tr}^{d^\gamma_{<\ell}}_{d^\gamma_{>\ell}}\left[\left(\mathbf{t}^T\mathbf{a}^\gamma\mathbf{I}_{d^\gamma_\forall} + \bigoplus_{\ell\in\gamma}\mathbf{\Psi}_\ell\right)^{-2}\right]\right]_{jk}
    \end{align*}

    This completes the proof.
\end{proof}

\begin{lemma}
    \label{lem:double-psi-derivative}
    \begin{align*}
        \frac{\partial^\mathrm{nosym}}{\partial\psi_{\ell^1_{i^1j^1}}\partial\psi_{\ell^2_{i^2j^2}}} \mathrm{NLL} &= \frac{1}{2}\sum_{\gamma|\ell^1,\ell^2\in\gamma} \mathrm{tr}^{d^\gamma_{<\ell^1}}_{d^\gamma_{>\ell^1}}\left[\left(\mathbf{\Omega}^\gamma\right)^{-1}\left(\mathbf{I}_{d^\gamma_{<\ell^2}} \otimes \mathbf{J}^{i^2j^2} \otimes \mathbf{I}_{d^\gamma_{>\ell^2}} \right)\left(\mathbf{\Omega}^\gamma\right)^{-1}\right]_{i^1j^1}&
    \end{align*}
\end{lemma}
\begin{proof}

    This follows from Lemma \ref{lem:first-derivative-psi}.  First we will compute the unsymmetrized derivative.

    \begin{align*}
        \frac{\partial^\mathrm{nosym}}{\partial\psi_{\ell^1_{i^1j^1}}\partial\psi_{\ell^2_{i^2j^2}}} \mathrm{NLL} &= -\frac{1}{2}\frac{\partial}{\partial\psi_{\ell^2_{i^2j^2}}}\sum_{\gamma|\ell^1\in\gamma} \mathrm{tr}^{d^\gamma_{<\ell^1}}_{d^\gamma_{>\ell^1}}\left[\left(\mathbf{t}^T\mathbf{a}^\gamma\mathbf{I_{d^\gamma_\forall}} + \bigoplus_{\ell\in\gamma}\mathbf{\Psi}_{\ell^1}\right)^{-1}\right]_{i^1j^1} \\
        &= -\frac{1}{2}\frac{\partial}{\partial\psi_{\ell^2_{i^2j^2}}}\sum_{\gamma|\ell^1\in\gamma} \mathrm{tr}^{d^\gamma_{<\ell^1}}_{d^\gamma_{>\ell^1}}\left[\left(\mathbf{\Omega}^\gamma\right)^{-1}\right]_{i^1j^1} \\
        &= \frac{1}{2}\sum_{\gamma|\ell^1,\ell^2\in\gamma} \mathrm{tr}^{d^\gamma_{<\ell^1}}_{d^\gamma_{>\ell^1}}\left[\left(\mathbf{\Omega}^\gamma\right)^{-1}\left(\mathbf{I}_{d^\gamma_{<\ell^2}} \otimes \mathbf{J}^{i^2j^2} \otimes \mathbf{I}_{d^\gamma_{>\ell^2}} \right)\left(\mathbf{\Omega}^\gamma\right)^{-1}\right]_{i^1j^1} \\
    \end{align*}

    \sloppy If we let the unsymmetrized derivative be $\mathbf{X}_{\ell^1_{i^1j^1}\ell^2_{i^2j^2}}$, the symmetrized version is:
    
    \[\frac{\mathbf{X}_{\ell^1_{i^1j^1}\ell^2_{i^2j^2}} + \mathbf{X}_{\ell^1_{j^1i^1}\ell^2_{i^2j^2}} + \mathbf{X}_{\ell^1_{i^1j^1}\ell^2_{j^2i^2}} + \mathbf{X}_{\ell^1_{j^1i^1}\ell^2_{j^2i^2}}}{4}\]
\end{proof}

With these, we can derive $\mathbf{F}$.  As the derivation of inter-axis connections in $\mathbf{F}$ is complicated, and will be aided by the following lemma, which is just an application of the fact that the trace distributes over Kronecker products.

\begin{lemma}
    \label{lem:trace-extraction}
    Let $\left\{\mathbf{M}_\ell\right\}_\ell$ be a subset of matrices where some subset $X$ has the property that, for all $\ell\in X$, $\mathbf{M}_\ell = \mathbf{I}_{d_\ell \times d_\ell}$, where $d_\ell$ represents the size of the $\ell$th element, and, for all $\ell\notin X$, $\mathbf{M}_\ell = \mathbf{J}_{d_\ell \times d_\ell}^{i_\ell j_\ell}$ (the matrix of zeros except at with a 1 at position $(i_\ell, j_\ell)$).

    Then $\mathrm{tr}\left[\bigotimes_\ell \mathbf{M}_\ell\right] = \left(\prod_{\ell\in X} d_\ell\right)\left(\prod_{\ell\notin X}\delta_{i_\ell j_\ell}\right)$.
\end{lemma}
\begin{proof}
    \begin{align*}
        \mathrm{tr}\left[\bigotimes_\ell \mathbf{M}_\ell\right] &= \prod_\ell \mathrm{tr}\left[\mathbf{M}_\ell\right] \\
        &= \prod_{\ell\in X} \mathrm{tr}\left[\mathbf{M}_\ell\right]\prod_{\ell\notin X} \mathrm{tr}\left[\mathbf{M}_\ell\right] \\
        &= \left(\prod_{\ell\in X}d_\ell\right)\prod_{\ell\notin X} \mathrm{tr}\left[\mathbf{M}_\ell\right] \\
        &= \left(\prod_{\ell\in X}d_\ell\right)\left(\prod_{\ell \notin X} \delta_{i_\ell j_\ell}\right)
    \end{align*}
\end{proof}

We are now equipped with all the tools we need to prove Lemma \ref{lem:null-hypothesis-F}, the values of $\mathbf{F}$ under the null hypothesis.

\begin{lemma}
    \label{lem:null-hypothesis-F}
    Under the null hypothesis where $\mathbf{\Omega}^\gamma = \frac{1}{\left(\sigma^\gamma\right)^2}\mathbf{I}_{d^\gamma_\forall}$, we have the following:

    \begin{align*}
        \mathbf{F}_{t_it_j} &= \sum_\gamma \frac{d^\gamma_\forall\left(\sigma^\gamma\right)^4}{2}\left(\mathbf{a}^\gamma\mathbf{a}^{\gamma, T}\right)_{ij} \\
        \mathbf{F}_{t_i\psi_{\ell_{jk}}} &= \sum_{\gamma|\ell\in\gamma} \frac{a^\gamma_id^\gamma_\forall\left(\sigma^\gamma\right)^4}{2 d_\ell}\delta_{jk} \\
        \mathbf{F}_{\psi_{\ell^1_{i^1j^1}}\psi_{\ell^2_{i^2j^2}}} &= \frac{1}{2}\sum_{\gamma|\ell^1,\ell^2\in\gamma}\left(\sigma^\gamma\right)^4\frac{d_\forall^\gamma}{d_{\ell^1}d_{\ell^2}}\delta_{i^2j^2}\delta_{i^1j^1} \tag{$\ell^1 \neq \ell^2$} \\
        \mathbf{F}_{\psi_{\ell^1_{i^1j^1}}\psi_{\ell^2_{i^2j^2}}} &= \frac{1}{2}\sum_{\gamma|\ell^1,\ell^2\in\gamma}\left(\sigma^\gamma\right)^4\frac{d^\gamma_\forall}{d_\ell}\left(\delta_{i^1j^1i^2j^2} + (1 - \delta_{i^1j^1i^2j^2})\frac{\delta_{i^1i^2}\delta_{j^1j^2}}{2}\right)\tag{$\ell^1 = \ell^2$}
    \end{align*}
\end{lemma}
\begin{proof}
    \begin{align*}
        \frac{\partial}{\partial t_i\partial t_j} \mathrm{NLL} &= \sum_\gamma \frac{a^\gamma_ia^\gamma_j}{2}\mathrm{tr}\left[\left(\mathbf{\Omega}^\gamma\right)^{-2}\right] \\
        &= \sum_\gamma \frac{a^\gamma_ia^\gamma_j d^\gamma_\forall\left(\sigma^\gamma\right)^4}{2} \\
        &= \sum_\gamma \frac{d^\gamma_\forall\left(\sigma^\gamma\right)^4}{2}\left(\mathbf{a}^\gamma\mathbf{a}^{\gamma, T}\right)_{ij}
        \\
        \frac{\partial}{\partial t_i\partial \psi_{\ell_{jk}}} \mathrm{NLL} &= \mathrm{sym}\left[\sum_\gamma \frac{a^\gamma_i}{2}\mathrm{tr}^{d^\gamma_{<\ell}}_{d^\gamma_{>\ell}}\left[\left(\mathbf{\Omega}^\gamma\right)^{-2}\right]_{jk}\right] \\
        &= \mathrm{sym}\left[\sum_\gamma \frac{a^\gamma_id^\gamma_\forall\left(\sigma^\gamma\right)^4}{2d_\ell}\left[\mathbf{I}_{d^\gamma_{\backslash\ell}}\right]_{jk}\right] \\
        &= \frac{1}{2}\sum_\gamma \frac{a^\gamma_id^\gamma_\forall\left(\sigma^\gamma\right)^4}{2 d_\ell}\delta_{jk} + \frac{a^\gamma_id^\gamma_\forall\left(\sigma^\gamma\right)^4}{2 d_\ell}\delta_{kj} \\
        &= \sum_\gamma \frac{a^\gamma_id^\gamma_\forall\left(\sigma^\gamma\right)^4}{2 d_\ell}\delta_{jk}
    \end{align*}

    \begin{align*}
        \frac{\partial^\mathrm{nosym}}{\partial\psi_{\ell^1_{i^1j^1}}\partial\psi_{\ell^2_{i^2j^2}}} \mathrm{NLL}
        &= \frac{1}{2}\sum_{\gamma|\ell^1\in\gamma} \mathrm{tr}^{d^\gamma_{<\ell^1}}_{d^\gamma_{>\ell^1}}\left[\left(\mathbf{\Omega}^\gamma\right)^{-1}\left(\mathbf{I}_{d^\gamma_{<\ell^2}} \otimes \mathbf{J}^{i^2j^2} \otimes \mathbf{I}_{d^\gamma_{>\ell^2}} \right)\left(\mathbf{\Omega}^\gamma\right)^{-1}\right]_{i^1j^1} \\
        &= \frac{1}{2}\sum_{\gamma|\ell^1\in\gamma} \left(\sigma^\gamma\right)^4\mathrm{tr}^{d^\gamma_{<\ell^1}}_{d^\gamma_{>\ell^1}}\left[\left(\mathbf{I}_{d^\gamma_{<\ell^2}} \otimes \mathbf{J}^{i^2j^2} \otimes \mathbf{I}_{d^\gamma_{>\ell^2}} \right)\right]_{i^1j^1}
    \end{align*}

    Let us focus on what happens to the stridewise-blockwise trace term under algebraic manipulations.

    \begin{align*}
        \mathrm{tr}^{d^\gamma_{<\ell^1}}_{d^\gamma_{>\ell^1}}\left[\left(\mathbf{I}_{d^\gamma_{<\ell^2}} \otimes \mathbf{J}^{i^2j^2} + \mathbf{J}^{j^2i^2} \otimes \mathbf{I}_{d^\gamma_{>\ell^2}} \right)\right]_{i^1j^1} &= \mathrm{tr}\left[\left(\mathbf{I}_{d^\gamma_{<\ell^2}} \otimes \mathbf{J}^{i^2j^2} \otimes \mathbf{I}_{d^\gamma_{>\ell^2}} \right)\left(\mathbf{I}_{d^\gamma_{<\ell^1}} \otimes \mathbf{J}^{i^1j^1} \otimes \mathbf{I}_{d^\gamma_{>\ell^1}} \right)\right]
    \end{align*}

    We now split the $(\psi_{\ell^1_{i^1j^1}}, \psi_{\ell^2_{i^2j^2}})$ computation into two cases, one in which $\ell = \ell^1 = \ell^2$ and another where $\ell^1 \neq \ell^2$.  In both cases we use Lemma \ref{lem:trace-extraction}.

    \textbf{Case 1: } $\ell^1 \neq \ell^2$

    \begin{align*}
        \mathrm{tr}^{d^\gamma_{<\ell^1}}_{d^\gamma_{>\ell^1}}\left[\left(\mathbf{I}_{d^\gamma_{<\ell^2}} \otimes \mathbf{J}^{i^2j^2} \otimes \mathbf{I}_{d^\gamma_{>\ell^2}} \right)\right]_{i^1j^1} &= \frac{d_\forall^\gamma}{d_{\ell^1}d_{\ell^2}}\delta_{i^2j^2}\delta_{i^1j^1}
    \end{align*}

    \textbf{Case 2: } $\ell = \ell^1 = \ell^2$

    \begin{align*}
        \mathrm{tr}^{d^\gamma_{<\ell}}_{d^\gamma_{>\ell}}\left[\left(\mathbf{I}_{d^\gamma_{<\ell}} \otimes \mathbf{J}^{i^2j^2} \otimes \mathbf{I}_{d^\gamma_{>\ell}} \right)\right]_{i^1j^1} &= \mathrm{tr}\left[\left(\mathbf{I}_{d^\gamma_{<\ell}} \otimes \mathbf{J}^{i^2j^2} \otimes \mathbf{I}_{d^\gamma_{>\ell}} \right)\left(\mathbf{I}_{d^\gamma_{<\ell}} \otimes \mathbf{J}^{i^1j^1} \otimes \mathbf{I}_{d^\gamma_{>\ell}} \right)\right] \\
        &= \mathrm{tr}\left[\left(\mathbf{I}_{d^\gamma_{<\ell}} \otimes \mathbf{J}^{i^2j^2}\mathbf{J}^{i^1j^1} \otimes \mathbf{I}_{d^\gamma_{>\ell}} \right)\right] \\
        &= \frac{d^\gamma_\forall}{d_\ell}\left(\mathrm{tr}\left[\mathbf{j}^{i^2}\mathbf{j}^{j^2, T}\mathbf{j}^{i^1}\mathbf{j}^{j^1, T}\right]\right) \\
        &= \frac{d^\gamma_\forall}{d_\ell}\left(\mathrm{tr}\left[\mathbf{j}^{j^2, T}\mathbf{j}^{i^1}\mathbf{j}^{j^1, T}\mathbf{j}^{i^2}\right] \right) \\
        &= \frac{d^\gamma_\forall}{d_\ell}\left(\delta_{j^2i^1}\delta_{j^1i^2}\right)
    \end{align*}

    This is the nonsymmetric derivative.  Symmetricizing requires replacing $\delta_{j^2i^1}\delta_{j^1i^2}$ with the average over all symmetries, swapping $i^1$ and $j^1$ and swapping $i^2$ and $j^2$.  These symmetries do not affect the different axis case, but do affect the same-axis case.

    One of the symmetries of $\delta_{j^2i^1}\delta_{j^1i^2}$ is $\delta_{i^1i^2}\delta_{j^1j^2}$, which we will rewrite as  $\delta_{(i^1, j^1)(i^2, j^2)}$ to make it more clear what is being compared.

    Suppose $i^1 = j^1$ is diagonal.  Then this is only true for $i^1 = i^2 = j^1 = j^2$, which if true under one symmetry then it is true under all symmetries.

    Suppose now that neither are diagonal, and recall we required $i \leq j$.  Then, any symmetries that swap the ordering of $i^1$ and $j^1$ will never be satisfied, unless they also swap $i^2$ and $j^2$.  This lets us rewrite the same-axis case as:

    \begin{align*}
        \delta_{i^1j^1i^2j^2} + (1 - \delta_{i^1j^1i^2j^2})\frac{\delta_{i^1i^2}\delta_{j^1j^2}}{2}
    \end{align*}

    This completes the proof.
    
\end{proof}

Recall that the only independent parameters are $\mathbf{t}$ and $\psi_{\ell_{ij}}$ for $i \leq j$ and $(i, j) \neq (1, 1)$.  There are three types of parameters here; $\mathbf{t}$, diagonal $\psi_{\ell_{ii}}$, and off-diagonal $\psi_{\ell_{ij}}$.  Of particular interest to us are the interactions between off-diagonal elements $\psi_{\ell_{ij}}$ (where $\delta_{ij} = 0$).  It is clear they cannot interact with $\mathbf{t}$ terms due to the $\delta_{ij}$ term in its formula.  They also cannot interact with any $\psi_{\ell'_{i'j'}}$ in different axes, since the $\delta_{ij}\delta_{i'j'}$ term is only nonzero for the diagonal elements.  Finally, the $\delta$ terms governing interactions within the same axis are only nonzero when $(i, j) = (i', j')$ (self-interactions) or $(i, j) = (j', i')$, which for off-diagonal elements can never be true due to the $i \leq j$ requirement.

\begin{table}[hpt!]
    \centering
    \begin{threeparttable}
    \begin{tabular}{c|ccc}
         &  $\mathbf{t}$ & Diagonal $\psi_{ii}$ & Off-diagonal $\psi_{ij}$ \\\hline
         $\mathbf{t}$ & Yes\tnote{a} & Yes & No \\
         Diagonal $\psi_{ii}$ & Yes & Only in different axes & No \\
         Off-diagonal $\psi_{ij}$ & No & No & No
    \end{tabular}
    \caption{Allowed interactions between parameters in $\mathbf{F}$.}
    \label{tab:F-interactions}
    \begin{tablenotes}
        \item[a] This depends on $\left\{\mathbf{a}^\gamma\right\}$.  For most datasets, the most natural choice of identifiable parameterization results in the $\mathbf{t}$-$\mathbf{t}$ interactions forming a diagonal matrix (none of the $\mathbf{t}$ interact).
    \end{tablenotes}
    \end{threeparttable}
\end{table}

For all possible interactions, consult Table \ref{tab:F-interactions} - these can be derived by investigating the $\delta$ terms in the formula for $\mathbf{F}$.  The fact that off-diagonal elements do not interact with anything in $\mathbf{F}$ is quite significant, as seen by the following corollary to Lemma \ref{lem:null-hypothesis-F}.

\begin{corollary}
    \label{cor:F-structure}
    Suppose we order $\mathbf{F}$ such that the first rows/columns represent $\mathbf{t}$ terms, the next represent diagonal $\psi_{ii}$ terms, and the last represent the off-diagonal terms.  Then, we have the following:

    \begin{align*}
        \mathbf{F} &= \begin{bmatrix}
            \mathbf{A} & \mathbf{0} \\
            \mathbf{0} & \mathbf{B}
        \end{bmatrix}
    \end{align*}

    $\mathbf{B}$ represents the off-diagonal terms, and \textbf{is a diagonal matrix}.  This allows a convenient expression for $\mathbf{F}^{-1}$, by inverting blockwise.
\end{corollary}

We will now give the example of a dataset $\gamma_1 = (\ell_1, \ell_2), \gamma_2 = (\ell_1, \ell_3)$ where $d_{\ell_1} = 3$ and $d_{\ell_2} = d_{\ell_3} = 2$.  Note that $\mathbf{a}^{\gamma_i} = \begin{bmatrix}
    \delta_{1i} & \delta_{2i}
\end{bmatrix}$ in this case (and most cases).  This leads to $\mathbf{F}_{\mathbf{t}, \mathbf{t}}$ being proportional to the identity matrix.  If we let $A^{\gamma} = \frac{d^\gamma_\forall\left(\sigma^\gamma\right)^4}{2}$, then it becomes easy to express the matrix.  Exact derivations are left as an exercise to the reader.

\[\scriptscriptstyle
\sum_\gamma A^\gamma \hspace{5pt} \begin{blockarray}{cccccccccccc}
t_1 & t_2 & \begin{matrix}\ell_1 \\ (2, 2)\end{matrix} & \begin{matrix}\ell_1 \\ (3, 3)\end{matrix} & \begin{matrix}\ell_2 \\ (2, 2)\end{matrix} & \begin{matrix}\ell_3 \\ (2, 2)\end{matrix} & \begin{matrix}\ell_1 \\ (1, 2)\end{matrix} & \begin{matrix}\ell_1 \\ (1, 3)\end{matrix} & \begin{matrix}\ell_1 \\ (2, 3)\end{matrix} & \begin{matrix}\ell_2 \\ (1, 2)\end{matrix} & \begin{matrix}\ell_3 \\ (1, 2)\end{matrix}\\
\begin{block}{(ccccccccccc)c}
  1 & 0 & \frac{1}{d_{\ell_1}} & \frac{1}{d_{\ell_1}} & \frac{1}{d_{\ell_2}} & \frac{1}{d_{\ell_3}} & 0 & 0 & 0 & 0 & 0 & t_1 \\
  0 & 1 & \frac{1}{d_{\ell_1}} & \frac{1}{d_{\ell_1}} & \frac{1}{d_{\ell_2}} & \frac{1}{d_{\ell_3}} & 0 & 0 & 0 & 0 & 0 & t_2 \\
  \frac{1}{d_{\ell_1}} & \frac{1}{d_{\ell_1}} & \frac{1}{d_{\ell_1}} & 0 & \frac{1}{d_{\ell_1\ell_2}} & \frac{1}{d_{\ell_1\ell_3}} & 0 & 0 & 0 & 0 & 0 & \begin{matrix}\ell_1 \\ (2, 2)\end{matrix} \\
  \frac{1}{d_{\ell_1}} & \frac{1}{d_{\ell_1}} & 0 & \frac{1}{d_{\ell_1}} & \frac{1}{d_{\ell_1\ell_2}} & \frac{1}{d_{\ell_1\ell_3}} & 0 & 0 & 0 & 0 & 0 & \begin{matrix}\ell_1 \\ (3, 3)\end{matrix} \\
  \frac{1}{d_{\ell_2}} & \frac{1}{d_{\ell_2}} & \frac{1}{d_{\ell_1}d_{\ell_2}} & \frac{1}{d_{\ell_1}d_{\ell_2}} & \frac{1}{d_{\ell_2}} & \frac{1}{d_{\ell_2\ell_3}} & 0 & 0 & 0 & 0 & 0 & \begin{matrix}\ell_2 \\ (2, 2)\end{matrix} \\
  \frac{1}{d_{\ell_3}} & \frac{1}{d_{\ell_3}} & \frac{1}{d_{\ell_1}d_{\ell_3}} & \frac{1}{d_{\ell_1}d_{\ell_3}} & \frac{1}{d_{\ell_2}d_{\ell_3}} & \frac{1}{d_{\ell_3}} & 0 & 0 & 0 & 0 & 0 & \begin{matrix}\ell_3 \\ (2, 2)\end{matrix} \\
  0 & 0 & 0 & 0 & 0 & 0 & \frac{1}{2d_{\ell_1}} & 0 & 0 & 0 & 0 & \begin{matrix}\ell_1 \\ (1, 2)\end{matrix} \\
  0 & 0 & 0 & 0 & 0 & 0 & 0 & \frac{1}{2d_{\ell_1}} & 0 & 0 & 0  & \begin{matrix}\ell_1 \\ (1, 3)\end{matrix} \\
  0 & 0 & 0 & 0 & 0 & 0 & 0 & 0 & \frac{1}{2d_{\ell_1}} & 0 & 0 & \begin{matrix}\ell_1 \\ (2, 3)\end{matrix} \\
  0 & 0 & 0 & 0 & 0 & 0 & 0 & 0 & 0 & \frac{1}{2d_{\ell_2}} & 0 & \begin{matrix}\ell_2 \\ (1, 2)\end{matrix} \\
  0 & 0 & 0 & 0 & 0 & 0 & 0 & 0 & 0 & 0 & \frac{1}{2d_{\ell_3}} & \begin{matrix}\ell_3 \\ (1, 2)\end{matrix} \\
\end{block}
\end{blockarray}
 \]
 
With everything in place, we can now define the hypothesis test.

\begin{theorem*}[Theorem \ref{thm:gmgm-bio-hypothesis-testing}]
    Under the null hypothesis $\mathbf{\Omega}^\gamma = \frac{1}{\left(\sigma^\gamma\right)^2}\mathbf{I}_{d^\gamma_\forall}$, we have the following distribution for each $\psi_{\ell_{ij}}$ independently:

    \begin{align*}
        \sqrt{\frac{\sum_{\gamma|\ell\in\gamma} d^\gamma_\forall\left(\sigma^\gamma\right)^4}{d_\ell}}\frac{\psi_{\ell_{ij}}}{2} &\mathrel{\dot{\sim}} \mathcal{N}\left(0, 1\right)
    \end{align*}
\end{theorem*}
\begin{proof}
    As noted in Corollary \ref{cor:F-structure}, the edge-edge interactions in $\mathbf{F}$ form a diagonal matrix, and there are no cross-terms connecting edges to diagonal and $\mathbf{t}$ parameters.  Thus, inversion is simple.

    By Lemma \ref{lem:null-hypothesis-F}, each edge-edge precision term for axis $\ell$ in $\mathbf{F}$ can be expressed as $\sum_{\gamma|\ell\in\gamma} \frac{d^\gamma_\forall\left(\sigma^\gamma\right)^4}{4d_\ell}$.  In $\mathbf{F}^{-1}$, this becomes $\frac{4d_\ell}{\sum_{\gamma|\ell\in\gamma} d^\gamma_\forall\left(\sigma^\gamma\right)^4}$.

    Thus, under the null hypothesis, we have that:

    \begin{align*}
        \psi_{\ell_{ij}} &\mathrel{\dot{\sim}} \mathcal{N}\left(0, \frac{4d_\ell}{\sum_{\gamma|\ell\in\gamma} d^\gamma_\forall\left(\sigma^\gamma\right)^4}\right) \\
        \sqrt{\frac{\sum_{\gamma|\ell\in\gamma} d^\gamma_\forall\left(\sigma^\gamma\right)^4}{d_\ell}}\frac{\psi_{\ell_{ij}}}{2} &\mathrel{\dot{\sim}} \mathcal{N}\left(0, 1\right)
    \end{align*}
\end{proof}

This completes the derivations.  While deriving such a test was tedious, it gives us an interpretable way to set the threshold for our recovered graphs, by choosing to only keep statistically significant edges!

\section{Preserving input sparsity}
\label{sec:gmgm-bio-input-sparsity}

The final practical issue to surmount is that of preserving input sparsity.  $\mathrm{mat}_\ell\left[\{\mathcal{D}^\gamma\}_\gamma\right]$ has $d_\forall$ elements, but in many contexts, such as transcriptomics, the dataset will be highly sparse; the majority of the input is full of zeros.  Ideally, we would never `densify' the matrix.  Using COCA, however, requires computing the matrix of ranks, which will be fully dense.  If unaddressed, this would enforce an $\Omega(d_\forall)$ minimum memory requirement.  We will demonstrate the problem with an example, where $\Phi$ represents the cumulative density function of the normal distribution.

\begin{amsthmexample}[Demonstration that ranking densifies data]
    \begin{align*}
    \begin{bmatrix}
        3 & 0  & -2 & 0 & 4  \\
        0 & -1 & -3  & 0 & -2
    \end{bmatrix} &\rightarrow \begin{bmatrix}
        4 & 2  & 1 & 2 & 5  \\
        4 & 3 & 1  & 4 & 2
    \end{bmatrix} \tag{rank data} \\
    &\rightarrow \begin{bmatrix}
        \Phi^{-1}\left(\frac{4}{6}\right) & \Phi^{-1}\left(\frac{2}{6}\right)  & \Phi^{-1}\left(\frac{1}{6}\right) & \Phi^{-1}\left(\frac{2}{6}\right) & \Phi^{-1}\left(\frac{5}{6}\right)  \\
        \Phi^{-1}\left(\frac{4}{6}\right) & \Phi^{-1}\left(\frac{3}{6}\right) & \Phi^{-1}\left(\frac{1}{6}\right)  & \Phi^{-1}\left(\frac{4}{6}\right) & \Phi^{-1}\left(\frac{2}{6}\right)
    \end{bmatrix} \tag{map to normal} \\
    &\approx \begin{bmatrix}
        0.43 & -0.43  & -0.97 & -0.43 & 0.97  \\
        0.43 & 0 & -0.97 & 0.43 & -0.43
    \end{bmatrix}
\end{align*}
\end{amsthmexample}

The exact values depend on how you handle ties when ranking data, but the general idea remains the same regardless.  The trick to avoiding the densification is to realize that, in each row, all zero values get mapped to the same final value.  We can rewrite our example as:

\begin{align*}
    \begin{bmatrix}
        0.43+0.43 & 0  & -0.97+0.43 & 0 & 0.97+0.43  \\
        0 & -0.43 & -0.97-0.43 & 0 & -0.43-0.43
    \end{bmatrix} + \begin{bmatrix}
        -0.43 \\ 0.43
    \end{bmatrix}\mathbf{1}_{1\times 5}
\end{align*}

Let $z_i$ be the normally-distributed value that the zeros in row $i$ get mapped to.  We can express our transformation more generally:

\begin{align*}
    \begin{bmatrix}
        \mathbf{x}_1 \\
        ... \\
        \mathbf{x}_a
    \end{bmatrix} &= \begin{bmatrix}
        \mathbf{x}_1 - z_1 \\
        ... \\
        \mathbf{x}_a - z_a
    \end{bmatrix} + \mathbf{z}\mathbf{1}^T
\end{align*}

Crucially, we have expressed our transformed data as the sum of a matrix with the same sparsity structure as our original data, and a rank-one matrix.  If we compute the SVD on the first matrix, we can then use a rank-one update algorithm to calculate the SVD of the sum of the two matrices \citep{brand_fast_2006}.  It is not difficult to work out $\mathbf{z}$, and hence we can perform this procedure without ever densifying our input matrix.

\newpage
\section{Additional experiments}
\label{sec:gmgm-bio-additional-experiments}

\subsection{Synthetic data}

In this section and in the main paper, we generated ground truth graphs from the Barabasi-Albert distribution.  We chose this distribution because it is a scale-free distribution; many real-world networks, such as social networks and gene regulatory networks, are approximately scale-free.  We then generated synthetic data using these graphs and the Kronecker-sum-structured normal distribution.

To test performance, and area-under-PR curves (AUPR) for GmGM and our proposed modifications to GmGM; when an algorithm correctly identifies an edge, it is considered a true positive.  When comparing to prior work, we used their default convergence tolerance and max iterations.

\begin{figure}
        \includegraphics[width=0.48\textwidth]{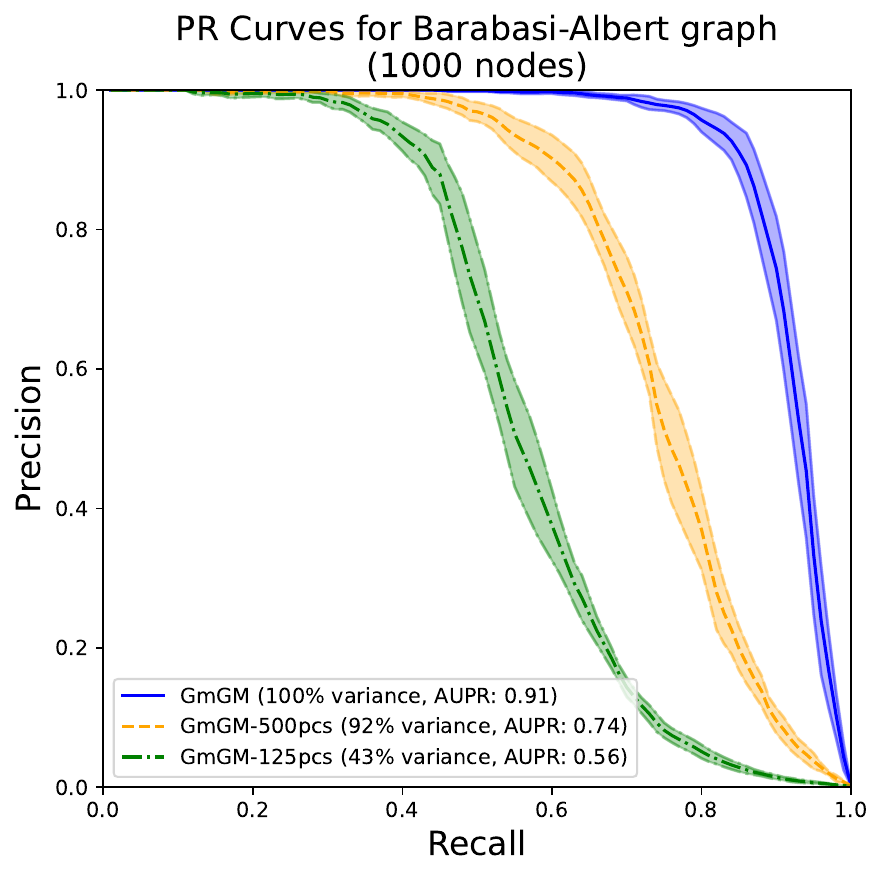}%
        \label{fig:pr-curves-1000}
    \hfill
        \includegraphics[width=0.48\textwidth]{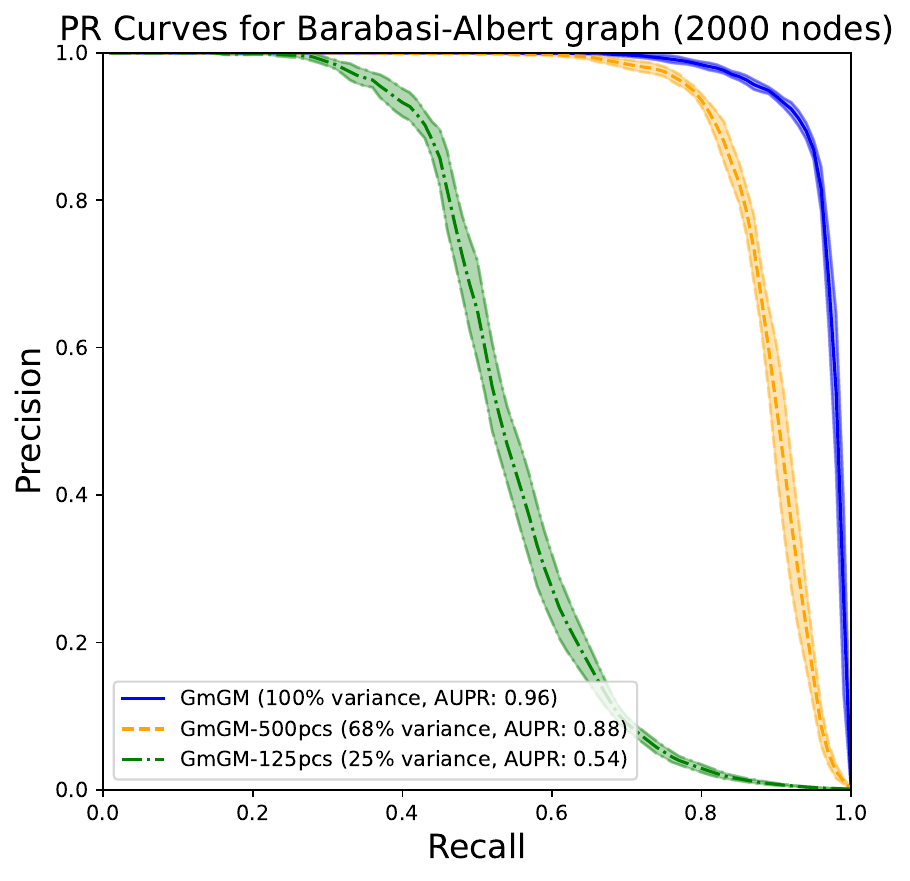}%
        \label{fig:pr-curves-2000}
    
    \medskip
    
        \includegraphics[width=0.48\textwidth]{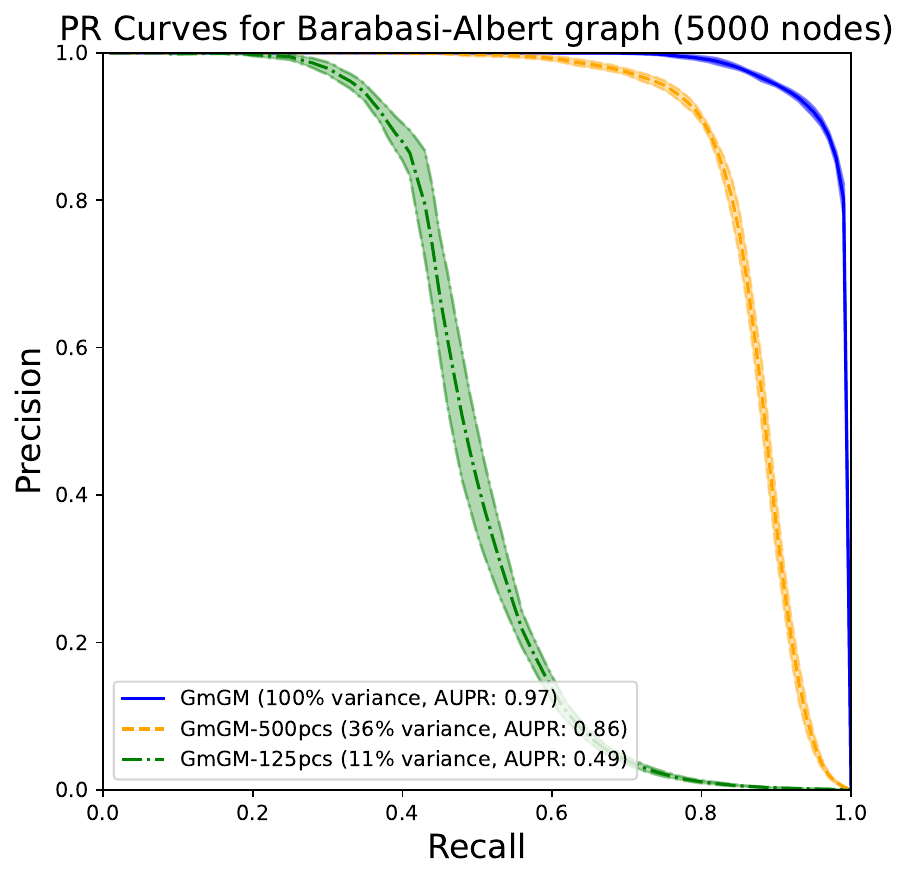}%
        \label{fig:pr-curves-5000}
    \hfill
        \includegraphics[width=0.48\textwidth]{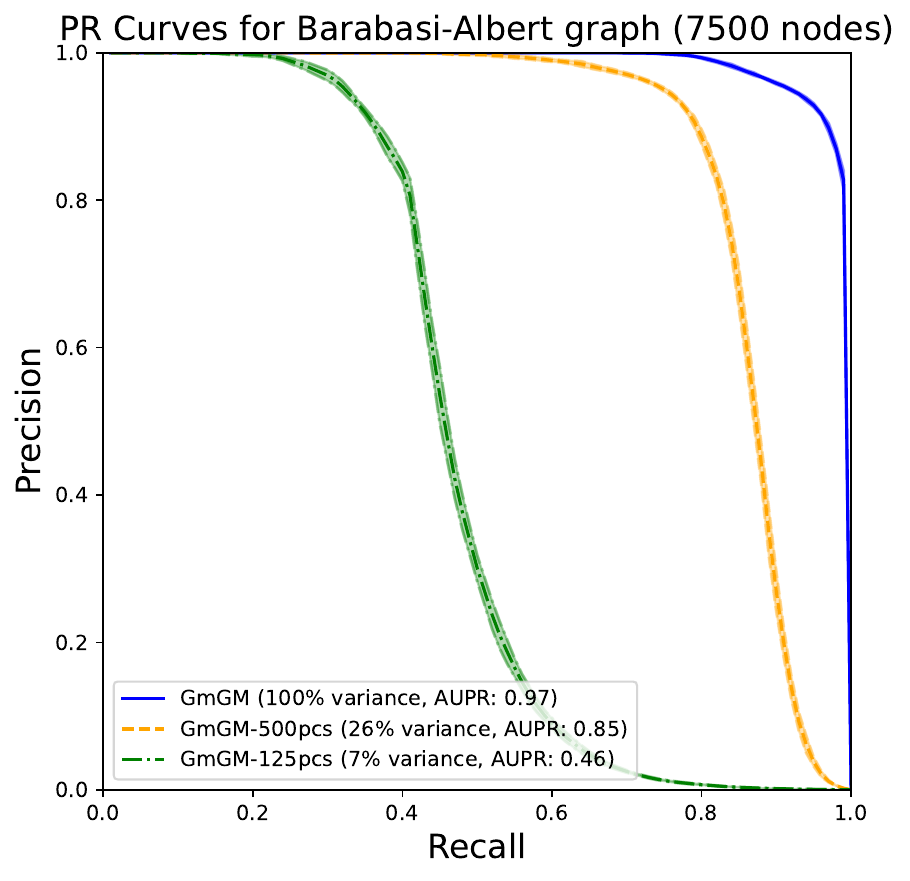}%
        \label{fig:pr-curves-7500}
    \caption{Precision-recall curves, with AUPR reported, for $n\in\{1000, 2000, 5000,  7500\}$ datasets of shape $n\times n$ whose graphs ome from a Barabasi-Albert distribution.  We compare the full-rank GmGM algorithm, along with our proposed algorithm with 500 principal components and a more extreme 125 principal components.  We see that 500 principal components still yield good results as the dataset size increases.}
    \label{fig:gmgm-bio-pr-curves}
\end{figure}

We report our PR curves and their AUPRs in Figure \ref{fig:gmgm-bio-pr-curves}.  It is clear that, for a fixed amount of components, the AUPR stays roughly the same as the axis size increases.  This is true even as the percentage of variance explained decreases.  There is a clear trade-off between the performance of the method and the number of principal components chosen.  Keeping 500 components still has high performance even as the percent of variance explained decreases; this shows the method is suitable on problem sizes that full-rank methods such as GmGM simply cannot handle.

\subsection{Sanity check of million-cell results}

At the beginning of this paper, we promised scalability to a million cells.  Here, we will use a 1,248,980 cell, 36,571 gene PBMC dataset by \cite{yazar_single-cell_2022}.  For preprocessing, we removed any cells and genes that had 0 counts everywhere.  Then, we kept only the top 2,200 highly variable genes.  No cells had uniformly 0 counts, so the post-processing dataset size contained 1,249,980 cells and 2200 genes.

Due to the size of the dataset, we kept only the top 50 principal components ($\sim$50\% of the explained variance).  We keep only the statistically significant edges and use the Nonparanormal Skeptic.  Overall, 5890 edges were statistically significant.  We do not compare to prior work as this dataset is far too large for them to handle.

\begin{table}[]
    \centering
    \scalebox{0.7}{
    \begin{tabular}{cccc}\hline
         Algorithm & Cluster & Biological Process & p-value \\\hline\hline
         & $\begin{matrix}
             \text{1} \\ \text{(67 genes)}
         \end{matrix}$ & $\begin{matrix}
             \text{cell-cell adhesion} \\
             \text{apoptotic process} \\
             \text{positive regulation of cellular process}
         \end{matrix}$ & $\begin{matrix}
             1.0\times 10^{-9} \\
             3.1\times 10^{-9} \\
             5.1\times 10^{-9}
         \end{matrix}$\\\cline{2-4}
         & $\begin{matrix}
             \text{2} \\ \text{(65 genes)}
         \end{matrix}$ & $\begin{matrix}
             \text{protein maturation} \\
             \text{response to endoplasmic reticulum stress} \\
             \text{organonitrogen compound metabolic process}
         \end{matrix}$ &$\begin{matrix}
             5.6\times 10^{-7} \\
             1.3\times 10^{-6} \\
             9.1\times 10^{-6}
         \end{matrix}$\\\cline{2-4}
         GmGM (p=0.05) & $\begin{matrix}
             \text{3} \\ \text{(60 genes)}
         \end{matrix}$ & $\begin{matrix}
             \text{\color{blue}immune response} \\
             \text{\color{blue}leukocyte mediated immunity} \\
             \text{\color{blue}immune system process}
         \end{matrix}$ & $\begin{matrix}
             9.5\times 10^{-15} \\
             3.8\times 10^{-14} \\
             8.7\times 10^{-14}
         \end{matrix}$\\\cline{2-4}
         & $\begin{matrix}
             \text{4} \\ \text{(43 genes)}
         \end{matrix}$ & $\begin{matrix}
             \text{GO:0048523} \\
             \text{GO:0051172} \\
             \text{GO:0031324}
         \end{matrix}$ & $\begin{matrix}
             4.7\times 10^{-5} \\
             6.8\times 10^{-5} \\
             8.9\times 10^{-5}
         \end{matrix}$\\\cline{2-4}
         & $\begin{matrix}
             \text{5} \\ \text{(27 genes)}
         \end{matrix}$ & $\begin{matrix}
             \text{\color{blue}positive regulation of immune system process} \\
             \text{\color{blue}regulation of immune response} \\
             \text{\color{blue}GO:0002504}
         \end{matrix}$ &$\begin{matrix}
             8.3\times 10^{-15} \\
             2.3\times 10^{-14} \\
             2.4\times 10^{-14}
         \end{matrix}$\\\cline{2-4}
         & $\begin{matrix}
             \text{6} \\ \text{(14 genes)}
         \end{matrix}$ & $\begin{matrix}
             \text{intracellular sequestering of iron ion} \\
             \text{sequestering of iron ion} \\
             \text{negative regulation of actin filament polymerization}
         \end{matrix}$ &$\begin{matrix}
             8.1\times 10^{-3} \\
             1.1\times 10^{-2} \\
             1.3\times 10^{-2}
         \end{matrix}$\\\hline
    \end{tabular}
    }
    \caption{GO terms with long names have been replaced with their GO ID; most were of the form ``negative regulation of cellular/nitrogen compound (metabolic) process''.  Immune-related GO terms have been colored blue.  These were identified using the Python API of the GProfiler \citep{kolberg_gprofilerinteroperable_2023} functional enrichment tool.}
    \label{tab:gmgm-bio-million-cell-global-structure}
\end{table}

We then investigated the GO terms of the clusters of our graphs.  If our graph were nonsense, we would expect there to be no significant GO terms associated with any clusters.  However, in Table \ref{tab:gmgm-bio-million-cell-global-structure} we can see that this is not the case; we find six clusters corresponding to meaningful groupings of genes, including immune-related groupings (which we would expect and hope to find in a PBMC dataset, which contain mostly immune cells).

This example ran in only a couple of minutes - faster than would be predicted from the runtimes in Figure \ref{fig:gmgm-bio-runtime}, likely due to the high sparsity of the dataset\footnote{Figure \ref{fig:gmgm-bio-runtime} gives performance on dense datasets, whereas our algorithm, unlike prior work, benefits substantially from sparsity in the dataset as well.  This is because there are many efficient partial eigendecomposition routines for sparse data.}.

\newpage
\subsection{COIL Dataset}

\begin{figure}[t!]
    \centering
    \includegraphics[width=0.9\textwidth]{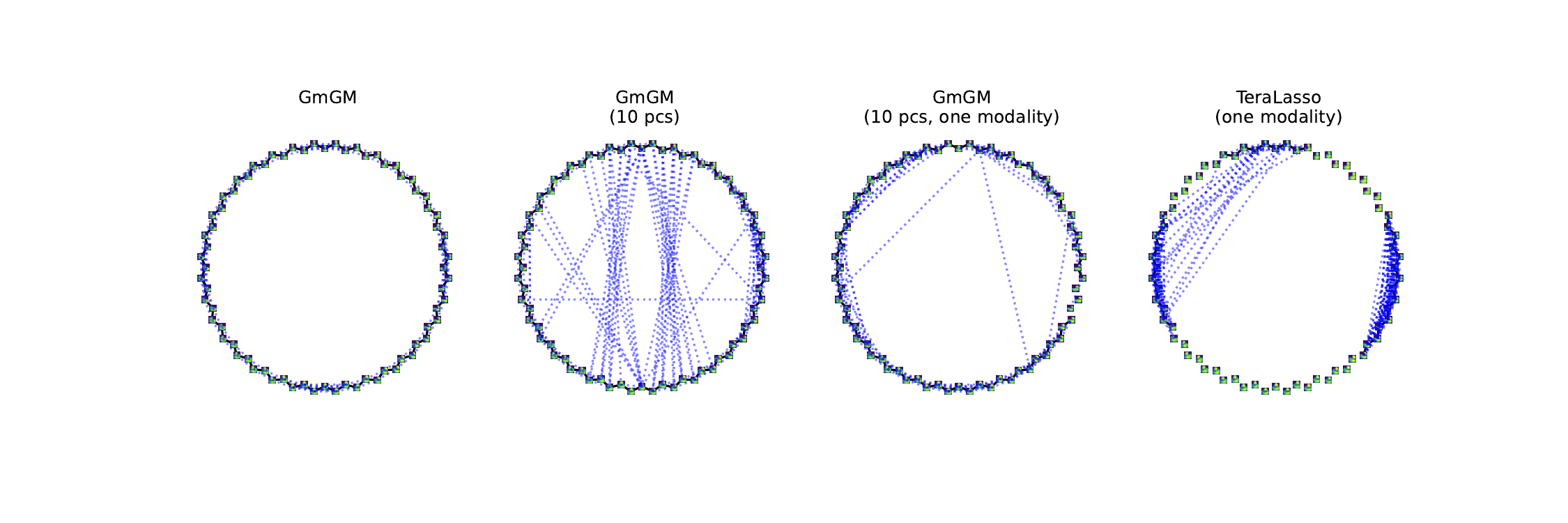}
    \caption{A comparison of several methods to find the frame graph.}
    \label{fig:gmgm-bio-coil-all}
\end{figure}

\begin{table}[]
    \centering
    \begin{threeparttable}
    \begin{tabular}{p{70pt}|cccccc|p{50pt}}
        Connections & Adjacent & One off & Two off & Three off & More & Total & Method \\ \hline
        GmGM & 72 & 72 & 25 & 4 & 0 & 173 & p=0.05 \\
        GmGM (10pcs) & 72 & 64 & 13 & 0 & 51 & 200 & p=0.05 \\
        GmGM (10pcs, one modality) & 66 & 58 & 44 & 13 & 19 & 200  & Thresholding \\
        TeraLasso (one modality) & 33 & 31 & 26 & 22 & 85 & 197\tnote{a} & Thresholding
    \end{tabular}
    \begin{tablenotes}
        \item[a] We were not able to get exactly 200 edges
    \end{tablenotes}
    \caption{The performance of several methods on reconstructing the frame graph of COIL-20 videos.  The adjacent column denotes how many edges were between adjacent frames (the `correct' edges), whereas the later columns denote how many edges exist between frames $i$ and $i\pm (n+1)$ for various $n$.  The More column summarizes all connections where $n>3$.}
    \label{tab:gmgm-bio-coil-performance}
    \end{threeparttable}
\end{table}

\begin{table}[t]
    \centering
    \begin{tabular}{p{70pt}|cccc}
        Connections & $\begin{matrix}\text{Precision/Recall}\\\text{(exact)}\end{matrix}$ & $\begin{matrix}\text{Precision/Recall}\\\text{(one off)}\end{matrix}$ & $\begin{matrix}\text{Precision/Recall}\\\text{(two off)}\end{matrix}$ & $\begin{matrix}\text{Precision/Recall}\\\text{(three off)}\end{matrix}$ \\ \hline
        GmGM & \textbf{41.6\%}/\textbf{100\%} & \textbf{83.2\%}/\textbf{100\%} & \textbf{97.6\%}/\textbf{78.2\%} & \textbf{100\%}/60\% \\
        GmGM (10pcs) & 36\%/\textbf{100\%} & 60\%/94.4\% & 74.5\%/69\% & 74.5\%/51.7\% \\
        GmGM (10pcs, one modality) & \color{blue}33\%/91.7\% & \color{blue}62\%/86\% & \color{blue}84\%/77.8\% & \color{blue}90.5\%/\textbf{62.8\%} \\
        TeraLasso (one modality) & 16.8\%/45.8\% & 32.5\%/44.4\% & 45.7\%/41.7\% & 56.7\%/38.9\%
    \end{tabular}
    \caption{The precisions and recalls for each algorithm, when we consider correct edges to be those between edges $i$ and $i\pm (n+1)$, for $n=0$ (`exact'), $n=1$ (`one off'), $n=2$ (`two off'), and $n=3$ (`three off').  In each column, the best performing model's results were given in bold for each metric.  \textbf{Bold numbers} represent the best performance overall, and {\color{blue}blue numbers} represent the best performance when restricted to one modality.}
    \label{tab:gmgm-bio-coil-pr}
\end{table}

It is difficult to test on real-world data, as ground truth graphs are often unknown.  This is perhaps why the COIL-20 video dataset \citep{nene_columbia_1996} has become a standard test for multi-axis models.

The dataset consists of 20 videos of 72 frames and 128x128 pixels.  Each video consists of a simple object rotating 360 degrees, and thus the frame graph can reasonably expected to correspond to this circular structure.  To test this, we considered the dataset to have the form $\left\{\gamma_i = (\ell_\mathrm{frame}, \ell_{\mathrm{row}(i)}, \ell_{\mathrm{col}(i)})\right\}_i$, that is, that there is a singular frame graph describing the rotational structure of the images, and then individual row and column graphs for each object (since the structure of each object may be different).  In other words, the dataset had 20 modalities, each of size 72x128x128.  TeraLasso can only handle a single modality, so to compare against it we ran our algorithm another time, this time on only one modality (the duck object, ``obj1'').  We compared the results using all principal components and when using only the top 10.

In the aforementioned experiments, we kept only the edges that were statistically significant to the 5\% level after applying the Bonferroni correction.  No edges were statistically significant, so instead we kept the same amount of edges as were found to be significant when considering the whole dataset (200).  We then did the same with TeraLasso.

The results can be compared visually in Figure \ref{fig:gmgm-bio-coil-all}.  These show that GmGM over all modalities does the best, which is to be expected - it uses the most information and does not make a low-rank assumption. TeraLasso cannot handle data of this form; however, even if we restrict our data to only a single modality and 10 principal components, our method performs best.  Numerical values are given in Tables \ref{tab:gmgm-bio-coil-performance} and \ref{tab:gmgm-bio-coil-pr}; we can see that, when considering statistically significant edges in the whole dataset, our method finds all the correct edges (even with only 10 principal components).  The false positives it finds are typically close-to-correct, which is reasonable.

\newpage
\section{Neuronal scRNA-seq dataset details}
\label{sec:gmgm-bio-lncrna-data-details}

\begin{figure}
    \centering
    \includegraphics[width=0.45\linewidth]{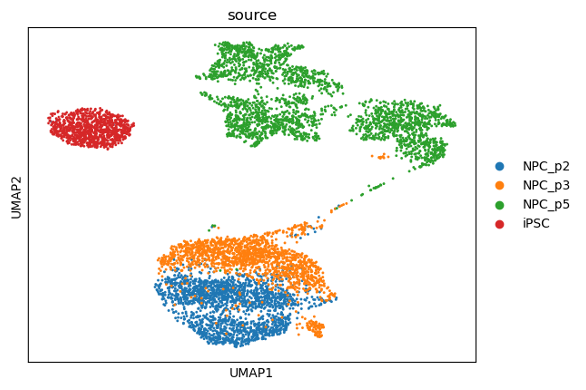}
    \includegraphics[width=0.48\linewidth]{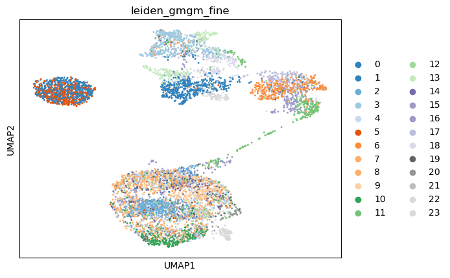}
    \caption{(Left) The four time points of the data, displayed as a UMAP plot. (Right) The cell clusters after applying Leiden to GmGM's output graph.}
    \label{fig:gmgm-bio-lncrna-passages}
\end{figure}

The dataset consists of four separate time points; the original iPSC cells, as well as cells taken from passages 2, 3, and 5 (`p2', `p3', and `p5'); these are shown graphically in Figure \ref{fig:gmgm-bio-lncrna-passages}.  In total, there are 6876 cells and 29,324 genes.  The dataset was generated in-house; the following subsections describe how this was done.

\subsection{Culture of induced pluripotent stem cells (iPSCs)}

STEMCELL Technologies’s SCTi003-A line (STEMCELL Technologies; $\#$200-0511) were cultured in complete mTeSR plus media (STEMCELL Technologies) on Matrigel® coated (Corning) 6-well plates (Helena Biosciences) and passaged at 1:40-1:60 split ratio with ReLeSR (STEMCELL Technologies) at approximately 70\% confluency. Cultures were regularly tested for pluripotency by immunocytochemistry and mycoplasma using MycoStrip lateral flow tests (InvivoGen).

\subsection{iPSC-derived neural progenitor cell (NPC) differentiation}

Neural induction of iPSCs was performed using the STEMdiff SMADi Neural Induction Kit (STEMCELL Technologies) according to the manufacturer's monolayer protocol.  iPSCs were dissociated with Accutase (Sigma) and resuspended as a single-cell suspension in DMEM/F12 (Thermo Fisher Scientific) for counting.  $2.5\times 10^6$ cells were pelleted at 300g for 5 minutes and resuspended in Neural Induction + SMADi Medium (NIM; STEMCELL Technologies) supplemented with 10 µM Y‑27632 onto hESC‑qualified Matrigel®‑coated 6‑well plates. Medium was replaced daily for 6–9 days, after which cells were passaged using Accutase and re‑plated at $2 \times 10^6$ cells per well. This induction–passage cycle was repeated until passage 3, at which point cells were re‑plated at $1.25 \times 10^6$ cells per well and subsequently maintained in STEMdiff Neural Progenitor Medium (STEMCELL Technologies) for all following passages.

\subsection{Cell fixation}

iPSCs were dissociated with Accutase at low passage and approximately 70\% confluency prior to fixation. NPCs were similarly lifted with Accutase on their designated passage day; for example, p2 NPCs were collected at passage 2 of differentiation. This procedure was performed for p2, p3, and p5 NPCs. Following cell counting, $1\times 10^6$ to $2\times 10^6$ cells were fixed using the Evercode Cell Fixation v2 kit (Parse Biosciences) according to the manufacturer’s user manual (v2.1.1). Final post‑fixation yields ranged from $7 \times 10^5$ to $1 \times 10^6$ cells, which were subsequently carried forward to barcoding.

\subsection{Barcoding, cDNA library preparation and sequencing}

Approximately $7 \times 10^4$ per sample were processed for barcoding and library preparation using the Evercode WT kit (v2.2.1) according to the manufacturers protocol (Parse Biosciences). Nine libraries were generated and assessed for fragment size and quality using an Agilent TapeStation prior to sequencing. The libraries were subsequently pooled and sequenced across two P4 flow cells on an Illumina NextSeq2000 using a 2x200bp paired end cycle. The resulting fastq files were processed using Trailmaker (Parse Biosciences) to generate gene read-counts per cell per time point.  This count matrix is what was passed into GmGM.

\section{Neuronal dataset cell marker genes}
\label{sec:gmgm-bio-lncrna-markers}

\begin{table}[]
    \centering
    \begin{tabular}{c|c}
        Cell type & Marker genes \\ \hline
        induced pluripotent stem cells & NANOG, POU5F1, LIN28A, DPPA4, TDGF1, ZSCAN10, UTF1  \\
        neural progenitors & PAX6, HES1, NES1, MSI1, VIM, FABP7 \\
        cortical neurons & TBR1, BCL11B, MAP2, RBFOX3, SLC17A7, SYT1, GRIN2B \\
        astrocytes & NFIA, SOX9, GFAP, AQP4, S100B, ALDH1L1, GLUL, SLC1A3, SLC1A2 \\
        oligodendrocytes & SOX10, MYRF, MBP, MOG, PLP1, MAG, CNP \\
        neural ectoderm & SOX1, ZIC1, OTX2, LHX2, EMX2, DACH1, SIX3
    \end{tabular}
    \caption{Marker genes for various neuronal cell types.}
    \label{tab:gmgm-bio-marker-genes}
\end{table}

We used the marker genes given in Table \ref{tab:gmgm-bio-marker-genes} to score each cluster.  A cell cluster was considered statistically significantly associated with a gene if the gene expression was significantly different in that cluster compared to others; cell clusters were then assigned the cell type corresponding to the one whose marker genes were the most commonly significantly differentially expressed.

\section{GO terms for network metrics}
\label{sec:gmgm-bio-go-term-cluster-metrics}

\begin{table}[t!]
    \centering
    \begin{tabular}{cp{2cm}|p{1.7cm}|p{7.5cm}}
        Reason & Category & Fundamental GO terms & Regulators/Occurs-In \\\hline
        Fundamental Process & DNA Replication & GO:0006260 & GO:0008156, GO:0045740, GO:0006275 \\\hline
        Fundamental Process & Cell Cycle & GO:0007049 & GO:0045786, GO:0045787, GO:0051726 \\ \hline
        Fundamental Process & Translation & GO:0006412 & GO:0032062, GO:0017148, GO:0140244, GO:0140245, GO:0140243, GO:0099578, GO:0099577, GO:0032055, GO:0043557, GO:0043555, GO:0045727, GO:1902010, GO:0006417, GO:0032061, GO:0032056, GO:0099547, GO:0036490, GO:0036493, GO:0043556, GO:0032939 \\ \hline
        Fundamental Process & Splicing & GO:0008380 & GO:0033119, GO:0033120, GO:0042484 \\ \hline
        Neurogenesis & Nervous System Development & GO:0007399 & GO:0051961, GO:0051960, GO:0051962 \\ \hline
        Neurogenesis & Synapse Organization & GO:0050808, GO:0045202 & GO:0140241, GO:0099574, GO:0140243, GO:1905809, GO:0050807 \\ \hline
        Neurogenesis & Dendrite Development & GO:0016358, GO:0030425 & GO:0050773, GO:2000171, GO:1900006 \\ \hline
        Neurogenesis & Axon Development & GO:0061564, GO:0030424 & N/A \\ \hline
        Neurogenesis & Myelination & GO:0042552, GO:0043209 & GO:0031642, GO:0031641, GO:0031643, GO:0019911 \\ \hline
    \end{tabular}
    \caption{GO term choices for metric comparison in Figure \ref{fig:gmgm-bio-go-term-cluster-metrics}.  Does not show \textit{is\_a} or \textit{part\_of} descendants, which were also included.}
    \label{tab:gmgm-bio-go-term-choices}
\end{table}

In Figure \ref{fig:gmgm-bio-go-term-cluster-metrics}, we picked which GO terms in two groups - those which are important processes dominant in all cells, and those which are part of the development of a key component found in neurons.  For each GO term category (such as 'DNA replication' or 'axon development'), we selected the GO term most relevant to the process, all GO terms that are annotated as \textit{regulates}, \textit{positively\_regulates}, \textit{negatively\_regulates}, or \textit{occurs\_in} the original term, and then any children and descendents of those GO terms via \textit{part\_of} or \textit{is\_a} relations.  The precise GO terms are given in Table \ref{tab:gmgm-bio-go-term-choices}.

\section{Neuronal scRNA-seq dataset full results}
\label{sec:gmgm-bio-lncrna-data-results}

\begin{table}[]
    \centering
    \begin{tabular}{c|c>{\raggedright\arraybackslash}p{3.5cm}c>{\raggedright\arraybackslash}p{3cm}c}
        Module & Top 3 genes by centrality & Select GO terms & Size & Annotation & Moran's I \\ \hline
        Unassigned & N/A & N/A & 14405 & A catch-all module. & 0.56 \\ \hline
        Grey & TSPAN6, CIBAR1, C3orf62 & DNA-binding transcription factor activity, regulation of biological process & 6399 & Transcription factor module, only expressed in non-iPSCs. & 0.99 \\\hline
        Turquoise & LINC00678, LNCPRESS2, ESRG & gene expression, biological regulation, RNA binding & 6146 & Only expressed in iPSCs. & 1.04 \\\hline
        Blue & SOX6, ZBTB20, ENSG00000258631 & nervous system development, neuron projection & 903 & Only expressed in non-iPSCs. & 1.01 \\\hline
        Brown & DMD, MIR100HG, ENSG00000286757 & anatomical structure development, cell differentiation, cilium organization & 559 & Expressed only in a subset of p5s. & 0.96 \\\hline
        Yellow & HMGCS1, SCD, FADS1 & cell differentiation, system development & 409 & Only expressed in p2/p3s. & 0.94 \\\hline
        Green & ECT2, SMC4, KIF20B & cell cycle, chromosome segregation, cell division & 262 & Cell cycle module. & 0.83 \\\hline
        Red & CFTR, SORCS3, PDGFD & regulation of biological process, nervous system development, cell differentiation & 241 & Only expressed in a subset of p5s. & 0.91 
    \end{tabular}
    \caption{Every hdWGCNA module, in order of size.  Moran's I is a statistic that expresses the tendency of `nearby' cells to have a similar average expression of a given module.  In this case, `nearby' is defined by the Euclidean distance of gene expression vectors of each cell.}
    \label{tab:gmgm-bio-hdwgcna-gene-modules}
\end{table}

In this section, we give a breakdown of our gene modules and cell clusters.  The GmGM cell clusters (and their links to the GmGM gene modules) are given in Table \ref{tab:gmgm-bio-lncrna-cell-clusters}, the hdWGCNA gene modules are given in Table \ref{tab:gmgm-bio-hdwgcna-gene-modules}, and the GmGM gene modules are given in Table \ref{tab:gmgm-bio-lnc-big-table}; there are no hdWGCNA cell clusters as hdWGCNA does not create a cell network.

For the gene modules, we defined `interesting' modules which satisfied one of the following criteria: high Moran's I ($>$0.3), high lncRNA percentage ($>$0.3), being related to transcription factors, being significantly expressed in a cell cluster, or being significantly expressed only in a specific passage.

Every module of hdWGCNA was `interesting'; in contrast, only 39 GmGM modules were.  Interesting modules are likely indicative of some true biological process; for example, GmGM modules m5, m7, and m117 are all interesting, and based on the GO terms represent the cell cycle, the ribosome, and mitochondrial genes, respectively.

\newpage

\begin{longtable}{c|c>{\raggedright\arraybackslash}p{3.5cm}c>{\raggedright\arraybackslash}p{3cm}c}
        \label{tab:gmgm-bio-lnc-big-table}
        Module & Top 3 genes by centrality & Select GO terms & Size & Annotation & Moran's I \\ \hline
    \endfirsthead
        Module & Top 3 genes by centrality & Select GO terms & Size & Annotation & Moran's I \\ \hline
    \endhead
        \multicolumn{6}{p{\textwidth}}{\textbf{Table \thetable} (\textit{multi-page})\textbf{.} Every GmGM module, in order of size.  Bold genes either are, are antisense to, or differential transcripts of transcription factors.  Select GO terms being N/A does not necessarily indicate a lack of GO terms, merely a lack of ones we felt worth highlighting.  Bold modules were considered `interesting'.}
    \endfoot
        \multicolumn{6}{p{\textwidth}}{\textbf{Table \thetable} (\textit{multi-page})\textbf{.} Every GmGM module, in order of size.  Bold genes either are, are antisense to, or differential transcripts of transcription factors (such as \textbf{LHX-DT}).  Select GO terms being N/A does not necessarily indicate a lack of GO terms, merely a lack of ones we felt worth highlighting.  Bold modules were considered `interesting'.}
    \endlastfoot
        m0 & DYNAP, DENRP3, ENSG00000250855 & N/A & 9047 & A catch-all module. & 0.15 \\ \hline
        \textbf{m1} & \textbf{FOXH1}, LAD1, FLT1 & N/A & 1756 & Only expresses in iPSCs. & 1.07 \\ \hline
        \textbf{m2} & CROCC2, ADGB, LMNTD1 & cilium movement, cilium organization, axoneme assembly & 600 & Only expressed in a subset of p5s. & 0.82 \\ \hline
        \textbf{m3} & SLC17A6, STMN2, \textbf{LHX9} & nervous system development, synaptic signalling, neuron projection morphogenesis & 433 & Only expresses in potential cortical neurons. & 0.83 \\ \hline
        \textbf{m4} & ENSG00000288903, CARMN, CPED1 & extracellular matrix organization, cell surface receptor signaling pathway & 319 & Only expressed in a subset of p5s. & 0.71 \\ \hline
        \textbf{m5} & SGO2, KIF14, CCNB1 & many cell-cycle related GO terms & 315 & Only expressed in a subset of p2/p3/p5s. & 0.82 \\ \hline
        \textbf{m6} & PDZRN4, SGPP2, CT75 & N/A & 313 & Only expresses in a subset of p4s, high amount of lncRNA, includes and is highly correlated with expression of GRM8. & 0.94 \\ \hline
        \textbf{m7} & RPS6, RPL7A, RPS8 & Ribosome-related GO terms & 303 & Ribosomal gene module. & 0.77 \\ \hline
        \textbf{m8} & {ANKRD1}, {SERPINE1}, COL8A1 & angiogenesis, vasculature development, DNA-binding transcription factor activity & 300 & Next three most central genes are all unnamed lncRNA; mostly expressed in p2s/p3s. One of three transcription-factor-related clusters.  & 0.58 \\ \hline
        \textbf{m9} & SPARCL1, SORCS3, CA12 & branching morphogenesis of an epithelial tube, branching involved in ureteric bud morphogenesis, mesonephric tubule morphogenesis & 276 & Only expressed in a subset of p5 cells.  Contains and correlates strongly with SORCS3. & 0.89 \\ \hline
        \textbf{m10} & CDKN1A, INPP5D, ENSG00000251095 & cell death/apoptosis GO terms & 262 & Not expressed in p5s. & 0.57 \\ \hline
        \textbf{m11} & MALAT1, MT-RNR1, MT-RNR2 & neurogenesis & 251 & Not expressed in iPSCs. & 0.90 \\ \hline
        \textbf{m12} & BRIP1, DTL, ATAD2 & DNA replication, DNA metabolic process, various cell-cycle terms & 249 & Only expressed in a subset of cells. & 0.66 \\ \hline
        \textbf{m13} & \textbf{EBF2}, \textbf{NEUROD4}, \textbf{EBF3} & nervous system development, neurogenesis, DNA-binding transcription factor activity & 239 & One of three transcription-factor-related clusters, only expresses in potential cortical neurons. & 0.70 \\ \hline
        \textbf{m14} & MIR9-2HG, \textbf{MEF2C-AS2}, WNT7A & nervous system development, cell differentiation, neurogenesis & 212 & Only expresses in a subset of p2/p3/p5 cells.  Contains and correlates with he expression of MIR9-2HG. High lncRNA percentage.  & 0.76 \\ \hline
        m15 & SEC24B, USB8, ZNF106 & localization, nucleic acid biosynthetic process & 204 & N/A & 0.04 \\\hline
        \textbf{m16} & LINC02306, ENSG00000253693, LINC03000 & nervous system development, neurogenesis & 203 & Only expressed in a subset of p5s. & 0.74 \\ \hline
        \textbf{m17} & CCT2, CCDC25, WDR1 & N/A & 201 & N/A & 0.39 \\ \hline
        m18 & CSNK2A1, DDR1, GDI1 & N/A & 199 & N/A & 0.08 \\ \hline
        \textbf{m19} & TXNL1, TMEM178A, FXYD6 & N/A & 198 & Significantly overexpressed in cell cluster 4. & 0.40 \\ \hline
        \textbf{m20} & PRANCR, C1QTNF7-AS1, ENSG00000250195 & N/A & 198 & Not expressed in iPSCs.  High lncRNA percentage. & 0.42 \\ \hline
        m21 & MED27, MBTD1, \textbf{MGA} & N/A & 186 & N/A & 0.06 \\ \hline
        m22 & CLCN3, PHF20L1, \textbf{ZNF532} & N/A & 184 & N/A & 0.11 \\ \hline
        \textbf{m23} & GDF7, \textbf{PAX3}, SV2C & anatomical structure morphogenesis, glutamatergic synapse & 181 & Only expressed in a subset of non-iPSC cells. & 0.64 \\ \hline
        \textbf{m24} & CDC20B, CCNO, \textbf{FOXN4} & centriole assembly, centriole replication, de novo centriole assembly involved in multi-ciliated epithelial cell differentiation & 180 & Only expressed in a subset of p5 cells. & 0.53 \\ \hline
        \textbf{m25} & NOS2, AGBL1, \textbf{BARHL1} & N/A & 178 & Mostly expressed in p2/p3 cells, some in p5s. & 0.85 \\ \hline
        m26 & LENG8, 7SK, ATN1 & N/A & 177 & N/A & 0.18 \\ \hline
        \textbf{m27} & \textbf{TBR1}, ISLR2, \textbf{LHX1-DT} & neuron projection morphogenesis, telencephalon development, synapse organization, dendrite, neurogenesis & 177 & Only expresses in potential cortical neurons. & 0.76 \\ \hline
        m28 & MTOR, PIN4, INO80D & nucleic acid metabolic process, nucleobase-containing compount metabolic process & 175 & N/A & 0.23 \\ \hline
        m29 & \textbf{UBP1}, INPP4A, PRKCSH & N/A & 174 & N/A & 0.12 \\ \hline
        \textbf{m30} & RSPO2, \textbf{LMX1A}, DYNLRB2-AS1 & regulation of cell differentiation, regulation of signal transduction & 174 & Expressed in all non-iPSC cells, more expression in a subset of p5 cells. & 0.86 \\ \hline
        m31 & CAPN2, EXOC6, BDNF-AS & gene expression & 172 & N/A & 0.07 \\ \hline
        m32 & BMP2K, FARP2, TAF2 & gene expression, RNA biosynthetic process & 171 & N/A & 0.08 \\ \hline
        m33 & \textbf{E2F5}, CCDC136, PREPL & gene xpression & 171 & N/A & 0.08 \\ \hline
        m34 & CPSF7, FGGY, SLC20A2 & transcription by RNA polymerase II, regulation of transcription by RNA polymerase II & 170 & N/A & 0.07 \\ \hline
        m35 & ARHGAP35, MAPKAP1, ORC2 & N/A & 169 & N/A & 0.05 \\ \hline
        m36 & ZNF84, KDM5C, EXOC3 & positive regulation of post-translational protein modification & 169 & N/A & 0.05 \\\hline
        m37 & RPAP3, SAFB, PEX14 & N/A & 168 & N/A & 0.10 \\\hline
        m38 & RAD50, SFSWAP, TBCA & MLL3/4 complex & 167 & N/A & 0.06 \\\hline
        m39 & DDX21, ENSG00000258131, LARP4 & N/A & 167 & N/A & 0.06 \\\hline
        m40 & RHEB, LINC00632, PRDX3 & N/A & 166 & N/A & 0.07 \\\hline
        m41 & MTREX, IP6K1, MED4 & gene expression, localization, translation & 164 & N/A & 0.16 \\\hline
        \textbf{m42} & CLEC2A, ENSG00000289578, NECTIN3-AS1 & N/A & 164 & Expressed only in a subset of p2/p3 cells, high lncRNA percentage. & 0.48 \\\hline
        m43 & PSMB7, KATNIP, SOS2 & inositol phosphate metabolic process, ``phosphotransferase activity, phosphate group as acceptor'' & 163 & N/A & 0.13 \\\hline
        m44 & MAPK8IP3, MED1, ANAPC1 & Glycosylphosphatidylinositol (GPI)-anchor biosynthesis & 163 & N/A & 0.09 \\\hline
        m45 & DZIP1, USP7, \textbf{AHDC1} & N/A & 163 & N/A & 0.10 \\\hline
        m46 & GBF1, CERK, \textbf{MRTFA} & mitochondrion & 161 & N/A & 0.08 \\\hline
        m47 & TARBP1, GCFC2, ESYT2 & N/A & 158 & N/A & 0.05 \\\hline
        m48 & SRP14, CEBPZ, CTTNBP2NL & RNA biosynthetic process, RNA metabolic process & 158 & N/A & 0.14 \\\hline
        m49 & SLC38A6, MIER1, ICE2 & proteolysis & 157 & N/A & 0.11 \\\hline
        \textbf{m50} & HMGCS1, MSMO1, HMGCR & sterol biosynthetic process, lipid biosynthetic process & 157 & Primarily expressed in p2/p3s, some expression in iPSCs and a subset of p5s. & 0.61 \\\hline
        \textbf{m51} & ENSG00000258312, MAMDC2, CCN1 & hippo signaling, cellular response to stimulus & 156 & Only expressed in p2/p3s. & 0.76 \\\hline
        m52 & SPG11, SSR1, SNX4 & N/A & 155 & N/A & 0.05 \\\hline
        m53 & UFL1-AS1, RGS12, MYH9 & N/A & 155 & N/A & 0.07 \\\hline
        m54 & TBC1D16, PRR14L, CMSS1 & N/A & 154 & N/A & 0.08 \\\hline
        m55 & WASL, DDAH2, PCGF3-AS1 & gene expression & N/A & 152 & 0.04 \\\hline
        m56 & XIST, LINC01515, HERC2P3 & N/A & 151 & N/A & 0.16 \\\hline
        m57 & RPS6KC1, PPFIA1, MCPH1 & N/A & 150 & N/A & 0.12 \\\hline
        m58 & CUL4B, QARS1, \textbf{PRDM4} & methylation & 148 & N/A & 0.04 \\\hline
        m59 & LDAH, FUBP1, KLHDC10 & N/A & 147 & N/A & 0.05 \\\hline
        m60 & IFT81, GANAB, TNFAIP3 & immune system response, immune system process & 147 & N/A & 0.18 \\\hline
        m61 & ZNF782, ENSG00000263551, SLCO6A1 & N/A & 146 & N/A & 0.05 \\\hline
        \textbf{m62} & CTNNA2, SMOC1, \textbf{SIM2} & neuron projection morphogenesis & 145 & Expressed in all non-iPSC cells. & 0.80 \\\hline
        \textbf{m63} & ATG5, ERCC5, PCBP1-AS1 & RNA processing, RNA metabolic process, gene expression & 144 & Expressed in all cells, decreasing through passages. & 0.34 \\\hline
        m64 & COPA, DHX57, MAPK1 & N/A & 144 & N/A & 0.06 \\\hline
        m65 & NDFIP1, SMURF1, TMEM245 & gene expression, RNA metabolic processing, lncRNA processing & 143 & N/A & 0.09 \\\hline
        m66 & RRAS2, FAM228B, IREB2 & N/A & 142 & N/A & 0.07 \\\hline
        m67 & CDK12, IARS1, ATP6V1D & tRNA metabolic process, RNA processing & 141 & N/A & 0.10 \\\hline
        m68 & SENP5, AMOTL1, ATE1 & N/A & 140 & N/A & 0.16 \\\hline
        m69 & BMERB1, VPS4A, RIMKLB & N/A & 140 & N/A & 0.12 \\\hline
        m70 & RICTOR, ZCCHC7, IGSF1 & N/A & 139 & N/A & 0.16 \\\hline
        \textbf{m71} & ENSG00000197585, ENSG00000237461, KYAT3 & N/A & 139 & High lncRNA percentage. & 0.26 \\\hline
        m72 & SLC25A46, LUC7L2, COX16 & N/A & 139 & N/A & 0.08 \\\hline
        m73 & GARRE1, HSPA8, AKAP8L & N/A & 138 & N/A & 0.17 \\\hline
        m74 & PPP2R2A, PGM3, MFSD8 & N/A & 138 & N/A & 0.05 \\\hline
        m75 & ZBED5, MRE11, TZNRD1 & mitochondrial membrane organization & 137 & N/A & 0.09 \\\hline
        m76 & RBM5, LAMP2, SCFD2 & N/A & 136 & N/A & 0.03 \\\hline
        m77 & NUMA1, NCK1, CFAP418-AS1 & succinate dehydrogenase (quinone) activity & 132 & N/A & 0.04 \\\hline
        m78 & UBE2D2, \textbf{ATF6}, SECISBP2L & N/A & 132 & N/A & 0.09 \\\hline
        m79 & NEK9, HYCC1, \textbf{ZNF577} & N/A & 132 & N/A & 0.09 \\\hline
        m80 & WDR70, CTBP2, \textbf{FOXO3} & serotonergic synapse, gultamatergic synapse & 131 & N/A & 0.10 \\\hline
        m81 & RHBDD1, BDP1, CDC14A & transcription factor TFIIIB complex & 129 & N/A & 0.07 \\\hline
        m82 & ABL1, RAPH1, PIGU & N/A & 127 & N/A & 0.07 \\\hline
        m83 & ENSG00000285713, RNF214, AMBRA1 & N/A & 127 & N/A & 0.04 \\\hline
        m84 & HADHA, TMSB10, ACTB & N/A & 126 & N/A & 0.33 \\\hline
        m85 & DNAJB14, HFM1, INTU & positive regulation of protein-containing complex disassembly & 125 & N/A & 0.06 \\\hline
        m86 & ZSWIM5, UBL3, PPM1A & N/A & 124 & N/A & 0.08 \\\hline
        m87 & ARMCX4, BSDC1, DCTN5 & DNA-templated transcription, regulation of DNA-templated transcription & 121 & N/A & 0.09 \\\hline
        m88 & GALNT13, TP53BP2, ANAPC10 & N/A & 121 & N/A & 0.06 \\\hline
        m89 & KDELR2, MEMO1, AIG1 & N/A & 120 & N/A & 0.06 \\\hline
        \textbf{m90} & \textbf{NFIA}, ENSG00000243620, \textbf{NFIB} & nervous system development, neuron differentiation, neurogenesis, DNA-binding transcription factor activity & 120 & One of the three transcription factor modules, highly expressed in p5s. & 0.93 \\\hline
        \textbf{m91} & ADAMTS18, \textbf{GLIS3}, LGR5 & cell adhesion, cell junction assembly, long-term synapse potentiation & 119 & Not expressed in iPSCs, \textbf{NR2F1} and related genes are very central to the module as well. & 0.85 \\\hline
        m92 & ATG2B, ALG11, MICU3 & N/A & 116 & N/A & 0.05 \\\hline
        m93 & CDC42SE2, RAB2A, FAM222B & N/A & 115 & N/A & 0.16 \\\hline
        m94 & AHCTF1, CHIC1, DIS3L2 & N/A & 115 & N/A & 0.11 \\\hline
        m95 & AGO4, PLEKHA8, PHF6 & N/A & 115 & N/A & 0.18 \\\hline
        m96 & NUMB, DLG3, UBE2B & regulation of primary metabolic process, regulation of macromolecule metabolic process & 113 & N/A & 0.07 \\\hline
        m97 & ARL13B, IFRD1, DPH6 & N/A & 108 & N/A & 0.07 \\\hline
        m98 & EML4, ELMO2, ANKRD6 & N/A & 107 & N/A & 0.10 \\\hline
        m99 & PRKCI, PTTG1IP, CCDC57 & ``hydrolase activity, acting on esther bonds'' & 104 & N/A & 0.04 \\\hline
        \textbf{m100} & \textbf{LMX1B}, C8orf34, FZD10 & N/A & 103 & Only expressed in a subset of p5 cells. & 0.68 \\\hline
        m101 & CALM3, PRKRA, ENSG00000285920 & N/A & 102 & N/A & 0.09 \\\hline
        m102 & DDX18, ENSG00000291130, \textbf{ZBTB25} & rRNA metabolic process, rRNA processing & 101 & N/A & 0.22 \\\hline
        m103 & TMEM168, NUP50, CCDC93 & N/A & 100 & N/A & 0.11 \\\hline
        \textbf{m104} & \textbf{ETV5}, LEMD1, TPRKB & negative regulation of lens fiber cell differentiation, regulation of lens fiber cell differentiation, negative regulation of ERK1 and ERK2 cascade & 96 & Expressed in all cells, but expression much higher in subset of p5 cells. & 0.25 \\\hline
        \textbf{m105} & H1-5, H1-2, H2AC8 & nucleosome assembly, chromatin remodeling, protein-DNA complex organization & 95 & Present in all passages, but only in a subset of each. & 0.55 \\\hline
        \textbf{m106} & SUPT20H, ATP6V1C1, ENSG00000289349 & N/A & 95 & High lncRNA percent. & 0.06 \\\hline
        m107 & STX16, CTNND1, OBSCN & N/A & 95 & N/A & 0.05 \\\hline
        m108 & SEC22A, TIAL1, SNHG29 & nucleic acid biosynthetic process & 94 & N/A & 0.08 \\\hline
        m109 & NUP107, DNAJB6, NBR1 & N/A & 91 & N/A & 0.08 \\\hline
        \textbf{m110} & LINC02646, DIRC3, PGK1 & pyruvate metabolic process, monosaccharide metabolic process, HIF1-signaling pathway, Glycolosis / Gluconeogenesis & 91 & Not expressed in p5s. & 0.38 \\\hline
        m111 & TMTC1, ATXN1, KRAS & N/A & 85 & N/A & 0.08 \\\hline
        m112 & KLHL7, CASTOR3P, AP2A2 & N/A & 85 & N/A & 0.09 \\\hline
        \textbf{m113} & CDH18, NAV3, SIPA1L3 & N/A & 84 & Expressed in all cells, but more so in certain subsets. & 0.47 \\\hline
        m114 & MCPH1-AS1, BCL9, CRYZL1 & N/A & 78 & N/A & 0.09 \\\hline
        m115 & SMARCB1, CLEC2D, PRKD1 & N/A & 74 & N/A & 0.08 \\\hline
        m116 & SCP2, GSPT1, ZMAT1 & N/A & 63 & N/A & 0.04 \\\hline
        \textbf{m117} & MT-CYB, MT-ND2, MT-ATP6 & oxidative phosphorylation, ATP synthesis coupled electron transport, NADH dehydrogenase activity & 48 & Consists of mitochondrial genes.  Expressed in all cells, less in p5s. & 0.35 \\\hline
        \textbf{m118} & ENSG00000227681, SAMD5, OSBPL1A & N/A & 37 & Expressed only in a subset of non-iPSC cells (all passages). & 0.69
\end{longtable}

\end{document}